\documentclass{article}

\PassOptionsToPackage{numbers, compress}{natbib}

\usepackage[accepted]{icml2020}

\usepackage[utf8]{inputenc} 
\usepackage[T1]{fontenc}    
\usepackage{hyperref}       
\usepackage{url}            
\usepackage{booktabs}       
\usepackage{amsfonts}       
\usepackage{nicefrac}       
\usepackage{microtype}      
\usepackage{enumerate}
\usepackage{amsmath,color}
\usepackage{algorithmic,algorithm}
\usepackage[algo2e]{algorithm2e}

\usepackage{graphicx}
\usepackage{float}
\usepackage{amsthm}
\usepackage{subcaption}
\usepackage{cleveref} 
\usepackage{enumitem}

\newcommand{\calO}{\mathcal{O}}

\makeatletter
\newcommand*\dotp{\mathpalette\dotp@{.5}}
\newcommand*\dotp@[2]{\mathbin{\vcenter{\hbox{\scalebox{#2}{$\m@th#1\bullet$}}}}}
\makeatother

\newtheorem{definition}{Definition}[section]
\newtheorem*{definition*}{Definition}

\newtheorem*{fact*}{Fact}

\newtheorem{lemma}[definition]{Lemma}

\newtheorem{assum}{Assumption}

\newtheorem*{theorem*}{Theorem}
\newtheorem*{lemma*}{Lemma}
\newtheorem*{proposition*}{Proposition}
\newtheorem{coro}{Corollary}   
\newtheorem{assumption*}{Assumption}
\newtheorem{condition*}{Condition}
\newtheorem{exercise*}{Exercise}
\newtheorem*{example*}{Example}

\newcommand{\red}[1]{\color{black} #1 \color{black}}
 
\usepackage{comment}
\usepackage{thm-restate}

\usepackage{authblk}    
\usepackage[T1]{fontenc}    
\usepackage[utf8]{inputenc}    

\usepackage[toc,page,header]{appendix}
\usepackage{minitoc}

\title{\textbf{Multi-Agent Off-Policy TD Learning: Finite-Time Analysis with Near-Optimal Sample Complexity and Communication Complexity}}
\author[1]{\textit{Ziyi Chen}}
\author[1]{\textit{Yi Zhou}}
\author[1]{\textit{Rongrong Chen}}

\affil[1]{Department of Electrical and Computer Engineering, University of Utah, USA}

\affil[ ]{\small {\{u1276972,yi.zhou\}@utah.edu, rchen@eng.utah.edu}}
\date{}

\begin{document}
\doparttoc 
\faketableofcontents 

\twocolumn[ 
\maketitle
\thispagestyle{empty}



\icmlsetsymbol{equal}{*}

\begin{icmlauthorlist}
\end{icmlauthorlist}


\icmlcorrespondingauthor{Cieua Vvvvv}{c.vvvvv@googol.com}
\icmlcorrespondingauthor{Eee Pppp}{ep@eden.co.uk}

\vskip 0.3in
]

\printAffiliationsAndNotice{\icmlEqualContribution}   

\begin{abstract}
The finite-time convergence of off-policy TD learning has been comprehensively studied recently. However, such a type of convergence has not been well established for off-policy TD learning in the multi-agent setting, which covers broader applications and is fundamentally more challenging.
This work develops two decentralized TD with correction (TDC) algorithms for multi-agent off-policy TD learning under Markovian sampling. In particular, our algorithms preserve full privacy of the actions, policies and rewards of the agents, and adopt mini-batch sampling to reduce the sampling variance and communication frequency. Under Markovian sampling and linear function approximation, we proved that the finite-time sample complexity of both algorithms for achieving an $\epsilon$-accurate solution is in the order of $\mathcal{O}(\epsilon^{-1}\ln \epsilon^{-1})$, matching the near-optimal sample complexity of centralized TD(0) and TDC. Importantly, the communication complexity of  our algorithms is in the order of $\mathcal{O}(\ln \epsilon^{-1})$, which is significantly lower than the communication complexity $\mathcal{O}(\epsilon^{-1}\ln \epsilon^{-1})$ of the existing decentralized TD(0). Experiments corroborate our theoretical findings. 
\end{abstract}

\section{Introduction}
Multi-agent reinforcement learning (MARL) has become an emerging technique that has broad applications in control \cite{Yanmaz2017,chalaki2020hysteretic}, wireless sensor networks \cite{Krishnamurthy2008,yuan2020towards}, robotics \cite{Yan2013}, etc. In MARL, agents interact with an environment and follow their own policies to collect private rewards. 
In particular, policy evaluation is a fundamental problem in MARL that aims to learn a multi-agent value function associated with the policies of the agents. This motivates the development of convergent and communication-efficient multi-agent TD learning algorithms.

For single-agent on-policy evaluation (i.e., samples are collected by target policy), the conventional TD(0) algorithm  \cite{sutton1988learning,sutton2018reinforcement} and Q-learning algorithm \cite{dayan1992convergence} have been developed  with asymptotic convergence guarantee. Recently, their finite-time (i.e., non-asymptotic) convergence has been established under Markovian sampling and linear approximation \cite{bhandari2018finite,zou2019finite}. However, these algorithms may diverge in the more popular off-policy setting \cite{baird1995residual}, where samples are collected by a different behavior policy. To address this important issue, a family of gradient-based TD (GTD) algorithms were developed for off-policy evaluation with asymptotic convergence guarantee  \cite{sutton2008convergent,sutton2009fast,maei2011gradient}. In particular, the TD with gradient correction (TDC) algorithm has been shown to have superior performance and its finite-time convergence has been established recently under Markovian sampling \cite{xu2019two,Gupta2019finite,kaledin2020finite}.

For multi-agent on-policy evaluation, various decentralized TD learning algorithms have been developed, e.g., the finite-time convergence of decentralized TD(0) was established with i.i.d samples \cite{Wai2018,Doan19a} and Markovian samples \cite{sun2020finite}, respectively, under linear function approximation, and an improved result is further obtained in \cite{wang2020decentralized} by leveraging gradient tracking.
However, these algorithms do not apply to the off-policy setting. In the existing literature, 
decentralized off-policy TD learning has been studied only in simplified settings, e.g., agents obtain independent MDP trajectories \cite{macua2014distributed,stankovic2016multi,cassano2020multi} or share their behavior and target policies with each other \cite{cassano2020multi}, and the data samples are either i.i.d.\ or have a finite sample size.
These MARL settings either are impractical or sacrifice the privacy of the agents. Therefore, we want to ask the following question:
\begin{itemize}[leftmargin=*,topsep=0pt]
	\item {\em Q1: Can we develop a decentralized off-policy TD algorithm for MARL with interdependent agents that collect private Markovian data following private policies?}
\end{itemize}
In fact, developing such a desired decentralized off-policy TD learning algorithm requires overcoming two major challenges. First, to perform decentralized off-policy TD learning, all the agents need to obtain a global importance sampling ratio (see \Cref{sec: TDC}). In \cite{cassano2020multi}, they obtained this ratio by sharing all local policies among the agents, which raises privacy concerns. Therefore, we need to develop private schemes to synchronize the global importance sampling ratio among the agents. Second, although the existing decentralized TD-type algorithms achieve the near-optimal sample complexity $\mathcal{O}(\epsilon^{-1}\ln \epsilon^{-1})$ \cite{sun2020finite,wang2020decentralized}, their communication complexities (number of communication rounds) are of the same order, which induces much communication overhead when the target accuracy $\epsilon$ is small. Hence, we want to ask the following fundamental question:
\begin{itemize}[leftmargin=*,topsep=0pt]
	\item {\em Q2: Can we develop a decentralized off-policy TD learning algorithm that achieves the near-optimal finite-time sample complexity while requires a significantly smaller communication complexity?}
\end{itemize}
In this work, we provide affirmative answers to these questions by developing two decentralized TDC algorithms. The algorithms preserve the privacy of all the agents and achieve the near-optimal sample complexity as well as a significantly reduced communication complexity. We summarize our contributions as follows.

\subsection{Our Contributions}
We consider a fully decentralized network where the agents share a common state space and take individual actions following their own behavior policies to collect local rewards. All of these information are kept private. The goal of the agents is to cooperatively learn the multi-agent value function associated with their target policies. 

To perform multi-agent off-policy evaluation, we develop two decentralized TDC algorithms with linear function approximation. In every iteration, each agent performs two timescale TDC updates locally and exchanges the linear model parameters with its neighborhood. In particular, our algorithms adopt the following designs to enable private off-policy TD learning and reduce communication complexity. 
\begin{itemize}[leftmargin=*,topsep=0pt]
	\item One critical issue is that the agents must use the global importance sampling ratio (product of local importance sampling ratios) to adjust their local updates. In our \Cref{alg: 1}, we propose to let the agents broadcast their local importance sampling ratios over the network until all agents obtain the exact global importance sampling ratio. In our \Cref{alg: 2}, we let the agents perform local averaging on the local importance sampling ratios to obtain approximated inexact global importance sampling ratios. We show that both schemes induce small communication overhead for synchronizing importance sampling ratios while ensuring fast convergence of the algorithms.
	\item We propose to let the agents use a mini-batch of samples to update their model parameters in each iteration. In this way, the mini-batch sampling reduces the sampling variance as well as the communication frequency, leading to an order-wise lower communication complexity than that of the decentralized TD0(0). 
	\vspace{-2mm}
	\item After the main decentralized TDC iterations, our algorithms perform local averaging of the model parameters to achieve a global consensus. Our proof shows that such local averaging steps are critical for establishing the near-optimal sample complexity and achieving an improved communication complexity.
\end{itemize}

Theoretically, we analyze the finite-time convergence of the two decentralized TDC algorithms with Markovian samples under exact and inexact global importance sampling, respectively. We show that both algorithms converge to a small neighborhood of the minimizer at a linear convergence rate, and the neighborhood size can be driven arbitrarily close to zero. The overall sample complexity of both algorithms to achieve an $\epsilon$-accurate solution is in the order of $\mathcal{O}(\epsilon^{-1} \ln \epsilon^{-1})$, which matches the state-of-the-art complexity of both centralized and decentralized TD learning and is near-optimal. More importantly, the total communication complexity of our algorithms for synchronizing model parameters is in the order of $\mathcal{O}(\ln \epsilon^{-1})$, which is significantly lower than the communication complexity $\mathcal{O}(\epsilon^{-1}\ln \epsilon^{-1})$ of the decentralized TD(0) \cite{sun2020finite} and matches the communication complexity lower bound for decentralized strongly-convex optimization \cite{scaman2017optimal}.

\subsection{Related Works}

\textbf{Centralized policy evaluation.} TD(0) with linear function approximation \cite{sutton1988learning} is popular for on-policy evaluation. The asymptotic and non-asymptotic convergence results of TD(0) have been established in \cite{sutton1988learning, dayan1992convergence, Jaakkola1993, gordon1995stable, baird1995residual, tsitsiklis1997analysis, tadic2001convergence, hu2019characterizing} and \cite{korda2015td,liu2015finite,bhandari2018finite,dalal2018finite, lakshminarayanan2018linear,wang2019multistep,srikant2019finite,xu2020reanalysis} respectively. 
\cite{sutton2009fast} proposed TDC for off-policy evaluation. The finite-sample convergence of TDC has been established in \cite{2018Finite,Dalal_2020} with i.i.d. samples and in \cite{xu2019two,Gupta2019finite,kaledin2020finite} with Markovian samples. 

\textbf{Decentralized policy evaluation.}
\cite{mathkar2016distributed} proposed the decentralized TD(0) algorithm. The asymptotic and non-asymptotic convergence rate of decentralized TD have been obtained in \cite{borkar2009stochastic} and \cite{sun2020finite,wang2020decentralized} respectively. 
Exisitng decentralized off-policy evaluation studies considered simplified settings. \cite{macua2014distributed,stankovic2016multi} obtained asymptotic result for decentralized off-policy evaluation where the agents obtained independent MDPs. \cite{cassano2020multi} obtained linear convergence rate also with indepedent MDPs by applying variance reduction and extended to the case where the individual behavior policies and the joint target policy are shared among the agents. 


\section{Policy Evaluation in Multi-Agent RL}
In this section, we introduce multi-agent reinforcement learning (MARL) and define the policy evaluation problem. 

Consider a fully decentralized multi-agent network that consists of $M$ agents. The network topology is specified by an undirected graph $\mathcal{G}=(\mathcal{M},\mathcal{E})$, where $\mathcal{M}=\{1,2,\cdots,M\}$ denotes the set of agents and $\mathcal{E}$ denotes the set of communication links. 
In MARL, the agents interact with a dynamic environment through a multi-agent Markov decision process (MMDP) specified as $\{\mathcal{S},\{\mathcal{A}^{(m)}\}_{m=1}^M,P,\{R^{(m)}\}_{m=1}^M,\gamma\}$. To elaborate, $\mathcal{S}$ denotes a global state space that is shared by all the agents, $\mathcal{A}^{(m)}$ corresponds to the action space of agent $m$, $P$ is the state transition kernel and $R^{(m)}$ denotes the reward function of agent $m$. All the state and action spaces have finite cardinality. $\gamma\in(0,1]$ is a discount factor. 

At any time $t$, assume that all the agents are in the global state $s_t\in \mathcal{S}$. Then, each  agent $m$ takes a certain action $a_{t}^{(m)}\in \mathcal{A}^{(m)}$ following its own stationary policy $\pi^{(m)}$, i.e., $a_{t}^{(m)}\sim \pi^{(m)} (\cdot|s_t)$. After all the actions are taken, the global state transfers to a new state $s_{t+1}$ according to the transition kernel $P$, i.e., $s_{t+1}\sim P(\cdot|s_t,a_t)$ where $a_t:=\{a_{t}^{(m)}\}_{m=1}^M$. At the same time, each agent $m$ receives a local reward $R_{t}^{(m)}:=R^{(m)}(s_t,a_t,s_{t+1})$ from the environment for this action-state transition. Throughout the MMDP, each agent $m$ has access to only the global state $\{s_t\}_t$ and its own actions $\{a_{t}^{(m)}\}_t$ and rewards $\{R_{t}^{(m)}\}_t$.  The goal of policy evaluation in MARL is to evaluate the following value function associated with all the local policies $\pi:=\{\pi^{(m)}\}_{m=1}^M$ for any global state $s$. 
\begin{align}
V^{\pi}(s)=\mathbb{E}\Big[\sum_{t=0}^{+\infty} \gamma^t \Big(\frac{1}{M}\sum_{m=1}^{M} R_{t}^{(m)}\Big)\Big|s_0=s,\pi\Big]. \label{Vfunc}
\end{align}
In particular, it is known that the above value function is a fixed point of the following Bellman operator $T^{\pi}$. 
\begin{align}
T^{\pi} [V(s)]:=& \mathbb{E}_{a\sim\pi(a|s), s'\sim P(s'|s,a)} \nonumber\\
& \Big[ \frac{1}{M}\sum_{m=1}^{M} R^{(m)}(s,a^{(m)},s') +\gamma V(s') \Big]. \label{BellmanOp}
\end{align}

{\bf Decentralized TD(0) with linear approximation.}  A popular algorithm for evaluating the value function in MARL is the decentralized TD(0) \cite{sun2020finite}, which is a decentralized variant of the centralized TD(0) algorithm.  Specifically, consider a popular linear function approximation of the value function  ${V}_{\theta}(s):=\theta^{\top}\phi(s)$, where $\theta\in \mathbb{R}^d$ contains the model parameters and $\phi(s)$ is a feature vector that corresponds to the state $s$. The linear function approximation has been widely considered in the existing literature \cite{xu2019two,sun2020finite,xu2020sample}, as it helps to avoid the curse of dimensionality ($d\ll |\mathcal{S}|$).

In decentralized TD(0), each agent $m$ collects a Markovian sample $\{s_t,a_{t}^{(m)},s_{t+1},R_{t}^{(m)}\}$ at time $t$ and updates its own model parameters $\theta_{t}^{(m)}$ with learning rate $\alpha>0$ as follows. 
\begin{align}
\theta_{t+1}^{(m)}&=\sum_{m'\in \mathcal{N}_m}V_{m,m'} \theta_{t}^{(m')} + \alpha \big(A_t \theta_{t}^{(m)}+b_{t}^{(m)}\big),
\end{align}
where $\mathcal{N}_m$ denotes the neighborhood of agent $m$, $V$ corresponds to a doubly stochastic communication matrix and $A_t=\phi(s_t) (\gamma \phi(s_{t+1})-\phi(s_t))^{\top}$, $b_{t}^{(m)}=R_{t}^{(m)} \phi(s_t)$. The above update rule applies the local TD error to update the parameters and synchronize the parameters among neighboring agents through the network. However, decentralized TD(0) encounters the following challenges: 1) decentralized TD(0) cannot be applied to the off-policy setting, where the agents have the flexibility to perform TD learning with samples that are collected by a different behavior policy; 2) decentralized TD(0) requires $O(\epsilon^{-1}\log \epsilon^{-1})$ number of communication rounds to achieve an $\epsilon$-accurate solution, and hence is not communication-efficient. Our goal is to develop a more communication-efficient decentralized TD learning algorithm that applies to the off-policy setting.

\section{Two Timescale Decentralized TDC for Off-Policy Evaluation}

\subsection{Centralized TDC} 
In this subsection, we review the centralized TD with gradient correction (TDC) algorithm proposed in \cite{sutton2009fast}. 
In RL, the agent may not have enough samples that are collected following the target policy $\pi$. Instead, it may have some historical data samples that are collected under a different behavior policy $\pi_b$. Therefore, in this {\em off-policy} setting, the agent would like to utilize the historical data obtained by following the behavior policy $\pi_b$ to help evaluate the value function  $V^{\pi}$ associated with the target policy $\pi$.

In \cite{sutton2009fast}, a family of gradient-based TD (GTD) learning algorithms have been proposed for off-policy evaluation with convergence guarantee.  In particular, the TDC algorithm has been shown to have superior performance. To explain, consider the linear function approximation ${V}_\theta(s) = \theta^\top \phi(s)$ and suppose the state space includes states $s_1,...,s_n$, we can define a total value function as ${V}_\theta:=[{V}_\theta(s_1), ..., {V}_\theta(s_n)]^\top$.
In TDC learning, the goal is to minimize the following mean square projected Bellman error (MSPBE).
\begin{align*}
\text{MSPBE}(\theta) := \mathbb{E}_{\mu_b}\| {V}_\theta - \Pi T^{\pi} {V}_\theta\|^2,
\end{align*} 
where $\mu_b$ is the stationary distribution induced by $\pi_{b}$, $T^{\pi}$ is the Bellman operator and $\Pi$ is a projection operator onto the space of linear models.
Given the $t$-th sample $(s_t,a_t,s_{t+1},R_t)$ obtained by the behavior policy, we define the following terms 
\begin{align*}
\rho_t &:= \frac{\pi(a_t|s_t)}{\pi_b(a_t|s_t)}, \quad b_t :=  \rho_t R_t \phi(s_{t}),\\
A_t &:=  \rho_t \phi(s_{t})(\gamma\phi(s_{t+1}) - \phi(s_{t}))^\top, \\
B_t &:=  -\gamma  \rho_t \phi(s_{t+1})\phi(s_{t})^\top, ~~C_t :=-\phi(s_{t})\phi(s_{t})^\top,
\end{align*}
where $\rho_t$ is referred to as the {\em importance sampling ratio}. 
Then, with learning rates $\alpha, \beta>0$ and initialization parameters $\theta_0, w_0$, the two timescale off-policy TDC algorithm takes the following recursive updates for $t=0,1,2,...$
\begin{equation}
\text{(TDC):}~~
\left\{
\begin{aligned}
\theta_{t+1} &= \theta_t + \alpha (A_t \theta_t + b_t + B_t w_t),  \\
w_{t+1} &= w_t + \beta (A_t \theta_t + b_t + C_t w_t). \label{eq: tdc}
\end{aligned}
\right.
\end{equation}

\subsection{Decentralized Mini-batch TDC}\label{sec: TDC}
In this subsection, we propose two decentralized TDC algorithms for off-policy evaluation in MARL.

In the multi-agent setting, without loss of generality, we assume that each agent $m$ has a target policy $\pi^{(m)}$ and its samples are collected by a different behavior policy $\pi_{b}^{(m)}$. In particular, if agent $m$ is on-policy, then we have $\pi_{b}^{(m)} = \pi^{(m)}$. In this {\em multi-agent off-policy} setting, the agents aim to utilize the data collected by the behavior policies $\pi_{b}=\{\pi_{b}^{(m)} \}_{m=1}^M$ to help evaluate the value function $V^{\pi}$ associated with the target policies $\pi=\{\pi^{(m)} \}_{m=1}^M$.

However, directly generalizing the centralized TDC algorithm to the decentralized setting will encounter several challenges. First, the centralized TDC in \cref{eq: tdc} consumes one sample per-iteration and achieves the sample complexity $O(\epsilon^{-1}\log\epsilon^{-1})$ \cite{xu2019two}. Therefore, the corresponding decentralized TDC would perform one local communication per-iteration and is expected to have a communication complexity in the order of $O(\epsilon^{-1}\log\epsilon^{-1})$, which induces {\em large communication overhead}. Second, in the multi-agent off-policy setting, every agent $m$ has a local importance sampling ratio $\rho_{t}^{(m)} := \pi^{(m)}(a_{t}^{(m)}|s_t) / {\pi_{b}^{(m)}(a_{t}^{(m)}|s_t)}$. However, to correctly perform off-policy updates, every agent needs to know all the other agents' local importance sampling ratios in order to obtain the {\bf global importance sampling ratio} 
$\rho_{t}:=\prod_{m=1}^{M}  \rho_{t}^{(m)}.$


To address these challenges in multi-agent off-policy TD learning, we next propose two  {\em decentralized TDC} algorithms that take mini-batch stochastic updates. 

{\bf Decentralized TDC with exact $\rho_{t}$.} We first consider an idealized case where all the agents obtain the exact global importance sampling ratio $\rho_t$ in every iteration $t$. This requires all the agents to broadcast their local importance sampling ratios over the decentralized network $\mathcal{G}$ using at most $M-1$ communication rounds. Although this setting may not be desired for large networks, it serves as a basis for understanding and analyzing decentralized TDC-type algorithms. 

\begin{algorithm}[h]
	\caption{Decentralized mini-batch TDC with exact global importance sampling.}
	\label{alg: 1}
	{\bf Input:} Batch size $N$, iterations $T,T'$, learning rates $\alpha, \beta$.
	
	
	{\bf Initialize:} $\theta_{0}^{(m)},w_{0}^{(m)}$ for all agents $m\in\mathcal{M}$.
	
	\For{iteration $t=0, 1, \ldots, T-1$}
	{	
		Each agent collects $N$ Markovian samples, computes its corresponding local importance sampling ratio and broadcasts over the network. \\

		\For{agent $m\in\mathcal{M}$ in parallel}{
			Agent $m$ obtains exact global importance sampling ratios for the $N$ samples and performs the updates in \cref{TDC1-theta,TDC1-w}.
		}
	}	
	
	\For{iteration $t=T, T+1, \ldots, T+T'-1$}
	{	
		\For{agent $m\in\mathcal{M}$ in parallel}{
			$\theta_{t+1}^{(m)}=\sum_{m'\in \mathcal{N}_m} V_{m,m'} \theta_{t}^{(m')}.$
		}
	}
	
	\textbf{Output:} $\{\theta_{T+T'}^{(m)}\}_{m=1}^M$.
\end{algorithm}

We formally present our first algorithm in \Cref{alg: 1}. 
To elaborate, at iteration $t$, every agent $m$ cooperatively collects a mini batch of $N$ Markovian samples $\big\{s_i,a_{i}^{(m)},s_{i+1},R_{i}^{(m)} \big\}_{i=tN}^{(t+1)N-1}$ and perform the following two timescale TDC-type updates in parallel.
\begin{align}
\theta_{t+1}^{(m)}&=\sum_{m'\in \mathcal{N}_m}V_{m,m'} \theta_{t}^{(m')} \nonumber\\
&\quad+ \frac{\alpha}{N} \sum_{i=tN}^{(t+1)N-1}\!\!\!\! \big(A_i \theta_{t}^{(m)}+b_{i}^{(m)}+B_i w_{t}^{(m)}\big),\label{TDC1-theta}\\
w_{t+1}^{(m)}&=\sum_{m'\in \mathcal{N}_m}V_{m,m'}w_{t}^{(m')} \nonumber\\
&\quad+ \frac{\beta}{N} \sum_{i=tN}^{(t+1)N-1}\!\!\!\! \big(A_i \theta_{t}^{(m)}+b_{i}^{(m)}+C_i w_{t}^{(m)}\big),\label{TDC1-w}
\end{align} 
where $\alpha, \beta$ are the learning rates, $\mathcal{N}_{m}$ denotes the neighborhood of agent $m$, and $V$ is a doubly-stochastic matrix. The matrices $A_i, B_i, C_i$ are defined in the same way as those in centralized TDC, and $b_{i}^{(m)} = \rho_t R_{t}^{(m)} \phi(s_t)$.
The above \cref{TDC1-theta,TDC1-w} apply mini-batch TDC updates to update the value function parameter $\theta_t^{(m)}$ and the auxiliary parameter $w_t^{(m)}$ of agent $m$, and these parameters are further synchronized among the neighboring agents $\mathcal{N}_{m}$.
As we show in Section \ref{sec_analysis}, the use of mini-batch updates helps significantly reduce the overall communication complexity.  Moreover, note that all the agents only exchange their local model parameters $\theta_t^{(m)}, w_t^{(m)}$ and local importance sampling ratios $\rho_t^{(m)}$. Hence, their actions, behavior and target policies, and rewards are kept private with regard to each other.

After performing the decentralized TDC updates for a sufficient number of $T$ iterations, we will halt the TDC updates and let all agents synchronize their value function parameters $\theta_t^{(m)}$ via $T'$ iterations of local averaging. In this way, every agent will converge exponentially fast to the model average $\overline{\theta}_T=\frac{1}{M}\sum_{m=1}^{M} \theta_{T}^{(m)}$, which we show to converge fast to the desired parameter $\theta^*$.
To summarize, every TDC iteration of \Cref{alg: 1} consumes $N$ Markovian samples, and requires two vector communication rounds for synchronizing the parameter vectors $\theta_t^{(m)},w_t^{(m)}$ and at most $(M-1)$  communication rounds for broadcasting the local importance sampling ratios, while the last $T’$ iterations only involve the 
communication of the parameter vector $\theta_t^{(m)}$.

{\bf Decentralized TDC with inexact $\rho_t$.} 
We also propose \Cref{alg: 2} as a variant of \Cref{alg: 1} that does not require full synchronization of the global importance sampling ratio $\rho_t$. To elaborate, note that $\rho_t$ can be rewritten as 
$$\rho_{t}=\exp\Big(M \cdot\frac{1}{M}\sum_{m=1}^{M}\ln\rho_{t}^{(m)}\Big).$$ Therefore, all the agents just need to obtain the sample average $\frac{1}{M}\sum_{m=1}^{M}\ln\rho_{t}^{(m)}$, 
which can be efficiently approximated via local communication of the quantities $\{\ln \rho_{t}^{(m)}\}_{m=1}^M$. Based on this idea, we propose to let the agents perform local averaging for $L$ rounds to obtain approximated global importance sampling ratios $\{\widehat{\rho}_{t}^{(m)}\}_{m=1}^M$. Specifically, every agent $m$ initializes $\widetilde{\rho}_{t,0}^{(m)}=\ln \rho_{t}^{(m)}$ and for iterations $\ell = 0,..., L-1$ do
\begin{align}
\widetilde{\rho}_{t,\ell+1}^{(m)}&=\sum_{m'\in \mathcal{N}_m}V_{m,m'} \widetilde{\rho}_{t,\ell}^{(m')}, \label{rho_iter2}\\
\text{Output}:~ \widehat{\rho}_{t}^{(m)}&=\exp(M\cdot\widetilde{\rho}_{t,L}^{(m)}) \label{rho_end2}.
\end{align}
Such a local averaging scheme is much less restrictive than the exact global synchronization in \Cref{alg: 1}, especially when the network size $M$ is large. 
In fact, in \Cref{coro_rho}, we prove that all of these local estimates $\{\widehat{\rho}_{t}^{(m)}\}_{m=1}^M$ converge exponentially fast to the desired consensus quantity $\rho_{t}$ as $L$ increases. Hence, we can control the approximation error to be arbitrarily small by choosing a proper $L$. Then, every agent $m$ performs the following two timescale updates 
\begin{align}
&\theta_{t+1}^{(m)}=\sum_{m'\in \mathcal{N}_m}V_{m,m'} \theta_{t}^{(m')} \nonumber\\
&\quad+ \frac{\alpha}{N} \sum_{i=tN}^{(t+1)N-1}\!\!\!\!\big(A_{i}^{(m)} \theta_{t}^{(m)}+\widetilde{b}_{i}^{(m)}+{B}_{i}^{(m)} w_{t}^{(m)}\big),\label{theta_iter2}\\
&w_{t+1}^{(m)} = \sum_{m'\in \mathcal{N}_m}V_{m,m'}w_{t}^{(m')} \nonumber\\
&\quad+ \frac{\beta}{N} \sum_{i=tN}^{(t+1)N-1}\!\!\!\!\big(A_{i}^{(m)} \theta_{t}^{(m)}+\widetilde{b}_{i}^{(m)}+C_{i} w_{t}^{(m)}\big),\label{w_iter2}
\end{align}
where $A_{i}^{(m)}, {B}_{i}^{(m)}, \widetilde{b}_{i}^{(m)}$ are defined by replacing the exact global $\rho_i$ involved in $A_i, B_i, b_{i}^{(m)}$ (see \cref{TDC1-theta,TDC1-w}) with the approximated $\widehat{\rho}_i^{(m)}$. To summarize, every TDC iteration of \Cref{alg: 2} consumes $N$ Markovian samples, and requires two vector communication rounds for synchronizing the parameter vectors $\theta_t^{(m)},w_t^{(m)}$ and $L$ scalar communication rounds for estimating the global importance sampling ratio.   

\begin{algorithm}[h]
	\caption{Decentralized mini-batch TDC with inexact global importance sampling.}
	\label{alg: 2}
	{\bf Input:} Batch size $N$, iterations $T,T'$, learning rates $\alpha, \beta$.
	
	
	{\bf Initialize:} $\theta_{0}^{(m)},w_{0}^{(m)}$ for all agents $m\in\mathcal{M}$.
	
	\For{iteration $t=0, 1, \ldots, T-1$}
	{	
		Each agent collects $N$ Markovian samples and computes their local importance sampling ratios $\rho_t^{(m)}$.\\
		
		\For{agent $m\in\mathcal{M}$ in parallel}{
			Agent $m$ estimates global importance sampling ratios for the $N$ samples via \cref{rho_iter2,rho_end2} and performs the updates in \cref{theta_iter2,w_iter2}.
		}
	}
\For{iteration $t=T, T+1, \ldots, T+T'-1$}
{	
	\For{agent $m\in\mathcal{M}$ in parallel}{
		$\theta_{t+1}^{(m)}=\sum_{m'\in \mathcal{N}_m}V_{m,m'} \theta_{t}^{(m')}.$
	}
}

\textbf{Output:} $\{\theta_{T+T'}^{(m)}\}_{m=1}^M$.	
\end{algorithm}

%



\section{Finite-Time Analysis of Decentralized TDC}\label{sec_analysis}
In this section, we analyze the finite-time convergence and complexity of both \Cref{alg: 1} and \Cref{alg: 2}. In all the theorem statements, we introduce some notations to denote the universal constants. Please refer to \Cref{supp: notation} for a summary of all notations and constants.


Denote $\mu_{\pi_b}$ as the stationary distribution of the Markov chain $\{s_t\}_{t}$ induced by the collection of agents' behavioral policies $\pi_b$. Throughout the analysis, we define the following expected quantities.
\begin{align*}
	A &:= \mathbb{E}_{\pi_b}[A_t], ~ B :=\mathbb{E}_{\pi_b}[B_t], ~C :=\mathbb{E}_{\pi_b}[C_t],\\
	 b^{(m)} &:=\mathbb{E}_{\pi_b}\big[b_{t}^{(m)}\big], ~ \overline{b}_{t}:=\frac{1}{M}\sum_{m=1}^{M} b_{t}^{(m)}, ~ \overline{b} :{=}\mathbb{E}_{\pi_b}\big[\overline{b}_{t}\big],
\end{align*}
 where $\mathbb{E}_{\pi_b}$ denotes the expectation when $s_t\sim\mu_{\pi_{b}}$, $a_{t}^{(m)}\sim \pi_{b}^{(m)}(s_t)$ and $s_{t+1}\sim P(\cdot|s_t,a_t)$. It is well-known that 
the optimal model parameter is $\theta^*=-A^{-1}\overline{b}$. \cite{xu2020reanalysis,xu2020sample}

We first make the following standard assumption on the mixing time of the Markov chain.
\begin{assum}\label{assum_TV_exp}
	There exist constants $\nu>0$ and $\delta\in(0,1)$ such that for all $t \geq 0$,
	\begin{align}
	\sup_{s \in \mathcal{S}} d_{TV}\left(\mathbb{P}_{\pi_{b}}\left(s_{t} \mid s_{0}=s\right), \mu_{\pi_{b}}\right) \leq \nu \delta^{t},
	\end{align}
	where $d_{TV}(P, Q)$ denotes the total-variation distance between probability measures $P$ and $Q$.
\end{assum}
Assumption \ref{assum_TV_exp} has been widely adopted in the existing literature \cite{bhandari2018finite,xu2019two,xu2020sample,ma20a,ma2021}. 
It assumes that the state distribution converges exponentially to its stationary distribution $\mu_{\pi_b}$. Such an assumption holds for all homogeneous Markov chains with finite state-space and all uniformly ergodic Markov chains.

\begin{assum}\label{assum_ACinv}
	The matrices $A$ and $C$ are invertible. 
\end{assum}

\begin{assum}\label{assum_phi_le1}
	The feature vectors are bounded, i.e., $\|\phi(s)\| \le 1$ for all $s\in\mathcal{S}$.
\end{assum}

\begin{assum}\label{assum_rhobound}
The rewards and importance sampling ratios are bounded, i.e.,  
there exist $R_{\max}, \rho_{\max} >0$ such that for all $m$:
	$\max_{s,a,s'} R^{(m)}(s,a,s')<R_{\max}$ and
	$\max_{s,a^{(m)}} \rho^{(m)}(s,a^{(m)}),\max_{s,a} \rho(s,a)<\rho_{\max}$.
\end{assum}

\begin{assum}\label{assum_doubly}
	The communication matrix $V\in\mathbb{R}^{M\times M}$ is doubly stochastic, i.e., all the entries of $V$ are nonnegative (i.e., $V_{ij}\ge 0$) and $V\mathbf{1}=\mathbf{1}$, $\mathbf{1}^{\top}V=\mathbf{1}^{\top}$. Also, $V_{ij}>0$ iff $i,j\in\mathcal{E}$. Moreover, the second largest singular value of $V$ satisfies $\sigma_2\in[0,1)$. 
\end{assum}

Assumptions \ref{assum_ACinv} -- \ref{assum_rhobound} are standard and widely adopted in the analysis of TD learning algorithms \cite{xu2019two,xu2020sample}. As a result of Assumption \ref{assum_ACinv}, the matrix $C$ is negative definite and thus we have $\lambda_{1} :=-\lambda_{\max}(A^{\top}C^{-1}A)>0$, $\lambda_{2} := -\lambda_{\max}(C)>0$. In particular, when $\mathcal{S}$ is finite, Assumption \ref{assum_ACinv} is equivalent to that the feature matrix $\Phi\in\mathbb{R}^{d\times |\mathcal{S}|}$ (each column is a feature vector) has full row rank.
Assumption \ref{assum_phi_le1} can always hold by normalizing the feature vectors $\phi(s)$. Assumption \ref{assum_rhobound} implies that $R_{t}^{(m)}\le R_{\max}$ and $\rho_{t}^{(m)}, \rho_t\le \rho_{\max}$ for all $m,t$. Assumption \ref{assum_doubly} is standard and has been widely adopted in decentralized optimization \cite{singh2020squarm,saha2020decentralized} and decentralized TD learning \cite{sun2020finite,wang2020decentralized}. It ensures that all agents can reach a global consensus on the parameters via local communication.

\subsection{Finite-Time Analysis of \Cref{alg: 1}}

We obtain the following finite-time error bound for \Cref{alg: 1} with exact global importance sampling ratio and Markovian samples. Please refer to \Cref{supp: notation} for the definitions of the universal constants $c_4, c_{12}$, etc.

\begin{restatable}{thm}{thmsync}\label{thm_sync}
		Let Assumptions \ref{assum_TV_exp}--\ref{assum_doubly} hold. Run \Cref{alg: 1} for $T$ iterations with learning rates $\alpha\le \min \{\mathcal{O} (1), \mathcal{O} (\beta)\}$, $\beta\le\mathcal{O}(1)$ and batch size $N\ge \max \{\mathcal{O}(1), \mathcal{O}(\beta/\alpha)\}$ (see eqs. \eqref{alpha_cond1},\eqref{beta_cond1}\&\eqref{N_cond1} for details). Then, the model average $\overline{\theta}_T=\frac{1}{M}\sum_{m=1}^{M} \theta_{T}^{(m)}$ satisfies
		\begin{align}
		\mathbb{E} \big[ \|\overline{\theta}_{T}-\theta^*\|^2\big] &\!\le\!  \Big(1\!-\!\frac{\alpha\lambda_{1}}{4}\Big)^{T} \!\big(\|\overline{\theta}_{0}\!-\!\theta^*\|^2\!+\!\|\overline{w}_{0}\!-\!w_{0}^*\|^2\big) \nonumber\\
		&\quad+ \frac{12c_4\beta }{\lambda_1 N\alpha},\label{sum_err}
		\end{align}
		where $\theta^*=A^{-1}b$, $w_0^*=-C^{-1}(A\overline{\theta}_0+b)$. Furthermore, after $T'$ iterations of local averaging, the local models of all agents satisfy that: for all $m=1,...,M$,
		\begin{align}
		\mathbb{E}\big[\|\theta_{T+T'}^{(m)}-\overline{\theta}_{T}\|^2\big] &\le \sigma_2^{2T'}c_{12}^2. \label{consensus_err1} 
		\end{align}
\end{restatable}
\Cref{thm_sync} shows that the average model $\overline{\theta}_T$ converges to a small neighborhood of the optimal solution $\theta^*$ at a linear convergence rate, which matches the convergence rate of the centralized TDC \cite{xu2019two,xu2020sample}. In particular, the convergence error is in the order of $\mathcal{O}(\frac{\beta}{N\alpha})$, which can be driven arbitrarily close to zero by choosing a sufficiently large mini-batch size $N$ and constant-level learning rates $\alpha,\beta$. Moreover, the $T'$ steps of local parameter averaging further help all the agents achieve a small consensus error at a linear convergence rate. \Cref{sum_err,consensus_err1} together ensure the fast convergence of all the local parameters. We want to emphasize that the $T'$ local averaging steps are critical for establishing fast convergence of local parameters. Specifically, without the $T'$ local averaging steps, the consensus error $\mathbb{E}\big[\|\theta_{T}^{(m)}-\overline{\theta}_{T}\|^2\big]$ would be in the order of $\mathcal{O}(\alpha+\beta)=\mathcal{O}(1)$, which is constant-level and hence cannot guarantee the local parameters converge arbitrarily close to the true solution. On the other hand, choosing a sufficiently small $\alpha,\beta$ would solve this problem, but at the cost of slowing down the convergence rate in \cref{sum_err}. We show next that the extra $T'$ local averaging steps help achieve the near-optimal sample complexity while inducing negligible communication overhead. 



Based on \Cref{thm_sync}, we obtain the following complexity results by using the relation $\mathbb{E}\big[\|\theta_{T+T'}^{(m)}-\theta^*\|^2\big]\le 2\mathbb{E}\big[\|\theta_{T+T'}^{(m)}-\overline{\theta}_T\|^2\big] + 2\mathbb{E} \big[ \|\overline{\theta}_{T}-\theta^*\|^2\big]$.

\begin{restatable}{proposition}{propsync}\label{prop_sync}
	Under the same conditions as those of \Cref{thm_sync} and choosing $N=\mathcal{O}(\epsilon^{-1})$, $T,T'=\mathcal{O}(\ln\epsilon^{-1})$, 
	we have that $\mathbb{E}[\|\theta_{T+T'}^{(m)}-\theta^*\|^2]\le \epsilon$ for all $m$. Moreover, the overall communication complexity for model parameters is $T+T'=\mathcal{O}(\ln\epsilon^{-1})$, the overall communication complexity for importance sampling ratio is $MT=\mathcal{O}(M\ln \epsilon^{-1})$, and the total sample complexity is $NT=\mathcal{O}(\epsilon^{-1}\ln\epsilon^{-1})$.
\end{restatable}

Therefore, with an exact global importance sampling ratio, \Cref{alg: 1} achieves the sample complexity $\mathcal{O}(\epsilon^{-1}\ln\epsilon^{-1})$, which matches that of centralized TDC for Markovian samples \cite{xu2019two} and nearly matches the theoretical lower bound $\calO(\epsilon^{-1})$ given in \cite{kaledin2020finite}. 
Importantly, the overall communication complexity for synchronizing model parameters is in the order of $\mathcal{O}(\ln\epsilon^{-1})$, which is significantly smaller than the communication complexity $\mathcal{O}(\epsilon^{-1}\ln\epsilon^{-1})$ required by the decentralized TD(0) \footnote{The two papers do not report sample complexity and communication complexity, we calculated them based on their finite-time error bounds.} \cite{sun2020finite,wang2020decentralized}. Intuitively, this is because \Cref{alg: 1} adopts mini-batch stochastic updates, which significantly reduce both the stochastic variance and the required number of communication rounds. 


\subsection{Finite-Time Analysis of \Cref{alg: 2}}
{For the finite-time analysis of \Cref{alg: 2}, we make the following additional assumption. 
\begin{assum}\label{assum_rhobound2}
	The importance sampling ratios are lower bounded, i.e., there exists $\rho_{\min}>0$ such that $\min_{s, a^{(m)}} \rho^{(m)}(s,a^{(m)}),\min_{s,a} \rho(s,a)>\rho_{\min}$ for all $m$.
\end{assum}
This is equivalent to say that all policies take all possible actions with non-zero probability. We obtain the following finite-time error bound for \Cref{alg: 2} with inexact global importance sampling ratio and Markovian samples. Please refer to \Cref{supp: notation} for the definitions of $c_9, c_{10}, c_{11}$, etc.

\begin{restatable}{thm}{thmlocal}\label{thm_local}
	Let Assumptions \ref{assum_TV_exp}--\ref{assum_rhobound2} hold. Run \Cref{alg: 2} for $T$ iterations with learning rates $\alpha\le \min \{\mathcal{O} (1), \mathcal{O} (\beta)\}$, $\beta\le\mathcal{O}(1)$ and batch size $N\ge \max \{\mathcal{O}(1), \mathcal{O}(\beta/\alpha)\}$
	(see eqs. \eqref{alpha_cond2},\eqref{beta_cond2}\&\eqref{N_cond2} for details). Then, we have
	\begin{align}
	\mathbb{E}\big[\big\|\overline{\theta}_{T}-&\theta^*\big\|^2\big] \!\le\! \Big(1 \!-\! \frac{\alpha\lambda_1}{6}\Big)^{T} \big(\big\|\overline{\theta}_{0} \!-\! \theta^*\big\|^2 \!+\!  \big\|\overline{w}_{0} \!-\! w_0^*\big\|^2\big) \nonumber\\
	&\quad + \frac{18 c_9\beta}{\lambda_1N\alpha } + \beta \sigma_2^{2L}\Big( c_{10} (3^{T})+\frac{c_{11} }{\lambda_1 \alpha} \Big). \label{sum_err2}
	\end{align}
	Furthermore, after $T'$ iterations of local averaging, the local models of all agents satisfy that: for all $m=1,...,M$,
	\begin{align}
	\mathbb{E}\big[\|\theta_{T+T'}^{(m)}-\overline{\theta}_{T}\|^2\big] 
	&\le \sigma_2^{2T'} \Big(\|\Delta \Theta_{0}\|_F^2 + \|\Delta W_{0}\|_F^2 \nonumber\\
	&+ c_{20}  \beta^2 \sigma_2^{2L} (3^{T}) + \frac{3c_{19}\beta^2 }{1-\sigma_2^2} \Big).\label{consensus_err2} 
	\end{align}
\end{restatable}

Although \Cref{alg: 2} uses inexact global importance sampling ratios $\{\widehat{\rho}_t^{(m)}\}_{m=1}^M$, we show that they converge to the exact ratio exponentially fast through local averaging. Hence, the update rules of \Cref{alg: 2} are close to those of \Cref{alg: 1} and the proof follows.
Based on the above finite-time error bound, we further obtain the following complexity results of \Cref{alg: 2}. 

\begin{restatable}{proposition}{proplocal}\label{prop_local}
	Under the same conditions as those of Theorem \ref{thm_local} and choosing
	$N=\mathcal{O}(\epsilon^{-1})$, $T, T', L=\mathcal{O}(\ln\epsilon^{-1})$, we have  that $\mathbb{E}(\|\theta_{T+T'}^{(m)}-\theta^*\|^2)\le \epsilon$ for all $m$. Moreover, the overall communication complexity for model parameters is $T+T'=\mathcal{O}(\ln\epsilon^{-1})$, the overall communication complexity for importance ratio is $LT=\mathcal{O}(\ln^2 \epsilon^{-1})$, and the total sample complexity is $NT=\mathcal{O}(\epsilon^{-1}\ln\epsilon^{-1})$.
\end{restatable}
}

Compared with \Cref{alg: 1}, \Cref{alg: 2} achieves orderwise the same sample complexity and the same communication complexity for model parameters. Moreover, \Cref{alg: 2} does not require full synchronization of the global importance sampling ratio, and hence is simpler and more effective for large networks (when $M\ge \mathcal{O}(\ln\epsilon^{-1})$). 



\section{Experiments}
We simulate a multi-agent MDP with 10 fully decentralized agents. The shared state space contains 10 states and each agent can take 2 actions. All behavior policies are uniform policies (i.e., each agent takes all actions with equal probability), and the target policies are obtained by first perturbing the corresponding behavior policies with Gaussian noises sampled from $\mathcal{N}(0, 0.05)$ and then performing a proper normalization. The entries of the transition kernel and the reward functions are independently generated from the uniform distribution on $[0,1]$ (with proper normalization for the transition kernel). We generate all state features with dimension 5 independently from the standard Gaussian distribution and normalize them to have unit norm. The discount factor is $\gamma=0.95$. 

We consider two types of network topologies: a fully connected network with communication matrix $V$ having diagonal entries $0.8$ and off-diagonal entries $1/45$, and a ring network with communication matrix $V$ having diagonal entries $0.8$ and entries 0.1 for adjacent agents. We implement and compare two algorithms in these networks: the decentralized TD(0) with batch size $N=1$ \cite{sun2020finite} and our decentralized TDC with batch sizes $N=10,20,50,100$.

\subsection{Exact Global Importance Sampling}
We first test these algorithms with exact global importance sampling ratios and compare their sample complexities and communication complexities. We set learning rate $\alpha=0.2$ for the decentralized TD(0) and $\alpha=0.2*N$, $\beta=0.002*N$ for our decentralized TDC with varying batch sizes $N=10,20,50,100$. All algorithms are repeated 100 times using the same set of 100 MDP trajectories, each of which has 20k Markovian samples.

We first implement these algorithms in the fully connected network. 
Figure \ref{fig_exact} plots the relative convergence error $\|\overline{\theta}_t-\theta^*\|/\|\theta^*\|$ v.s. sample complexity ($tN$) and communication complexity ($t$). For each curve, its upper and lower envelopes denote the 95\% and 5\% percentiles of the 100 convergence errors, respectively. It can be seen that our decentralized TDC with different batch sizes achieve almost the same sample complexity as that of the decentralized TD(0), demonstrating the sample-efficiency of our algorithms. On the other hand, our decentralized TDCs require much less communication complexities than the decentralized TD(0), and the required communication becomes lighter as batch size increases. All these results match our theoretical analysis well. 
\begin{figure}[tbh]
	\begin{minipage}{.23\textwidth}
		\centering
		\includegraphics[width=\textwidth]{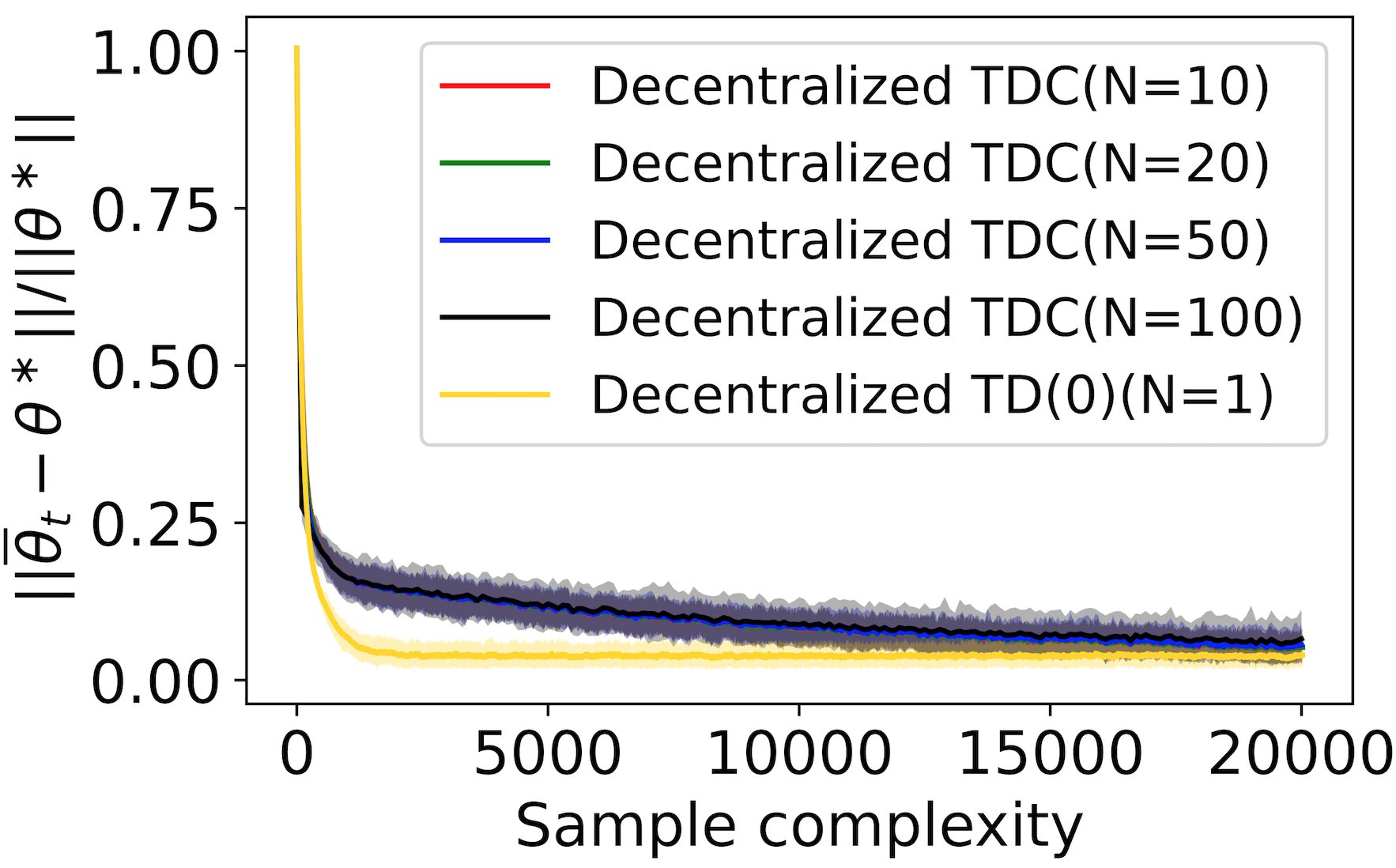}
	\end{minipage} 
	\begin{minipage}{.23\textwidth}
		\centering
		\includegraphics[width=\textwidth]{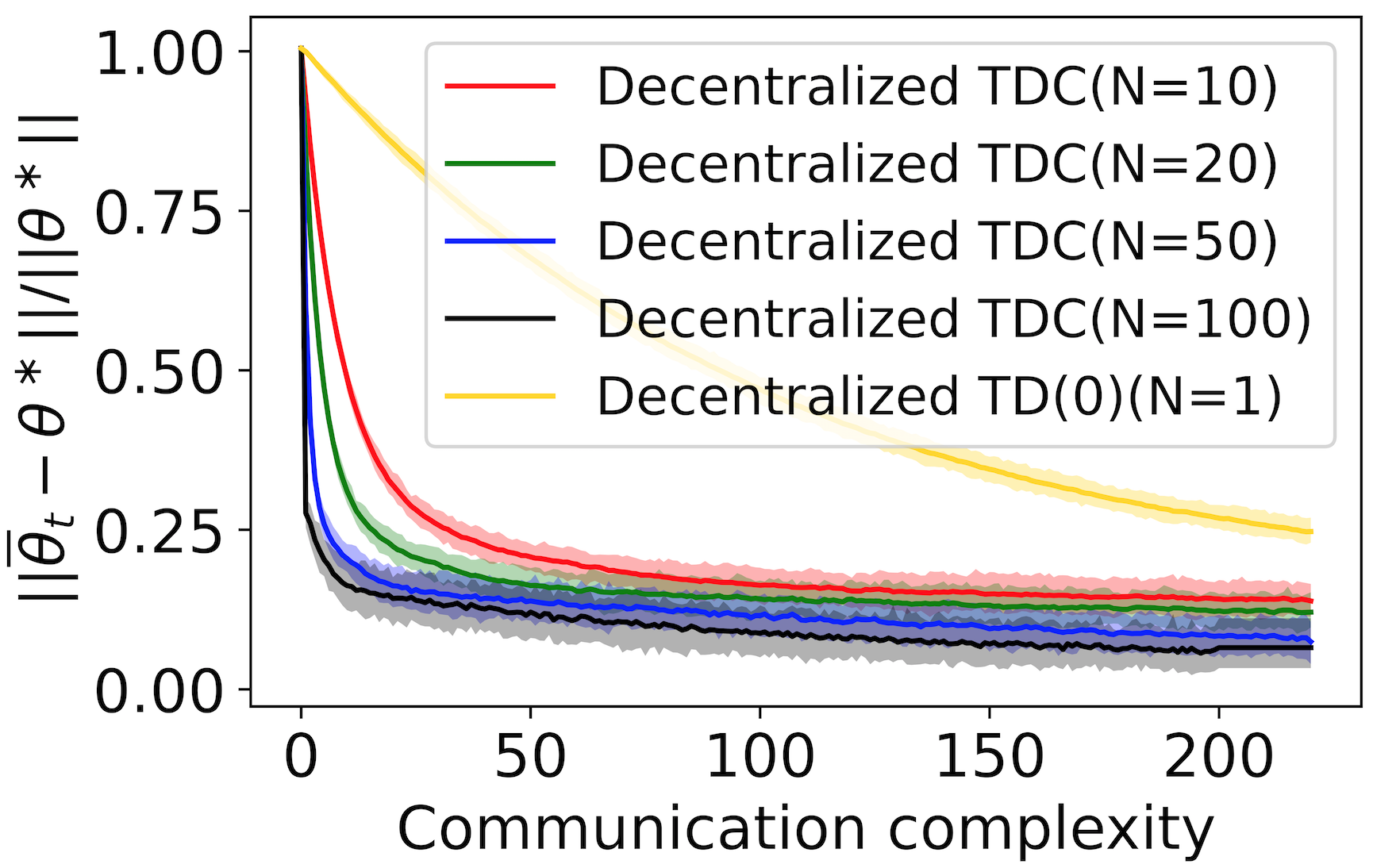}
	\end{minipage} 
	\caption{Comparison between decentralized TDC with varying batch sizes and decentralized TD(0) under exact global importance sampling.}
	\label{fig_exact}
\end{figure}

We further implement these algorithms in the ring network. The comparison results are exactly the same as those in \Cref{fig_exact}, since the update rule of $\overline{\theta}_t$ does not rely on the network topology under exact global importance sampling (See eqs. \eqref{theta_iter_avg1}\&\eqref{w_iter_avg1} in Appendix \ref{supp: notation}).

\subsection{Inexact Global Importance Sampling}
In the second experiment, we test our decentralized TDC with inexact global importance sampling ratios using varying communication rounds $L=1,3,5,7$. We use a fixed batch size $N=100$ and set learning rates $\alpha=5$, $\beta=0.05$, and repeat each algorithm 100 times using the set of 100 MDP trajectories. We also implement the decentralized TDC with exact global importance sampling ratios as a baseline. Figure \ref{fig_inexact1} plots the relative convergence error v.s. communication complexity in the fully-connected network (Left) and ring network (Right). It can be seen that in both networks, the asymptotic convergence error of the decentralized TDC with inexact $\rho$ decreases as the number of communication rounds $L$ for synchronizing the global importance sampling ratio increases. In particular, with a single communication round $L=1$, decentralized TDC diverges asymptotically due to inaccurate estimation of the global importance sampling ratio. As $L$ increases to more than 5, the convergence error is as small as that under exact global importance sampling. 
\begin{figure}[tbh]
	\begin{minipage}{.23\textwidth}
		\centering
		\includegraphics[width=\textwidth]{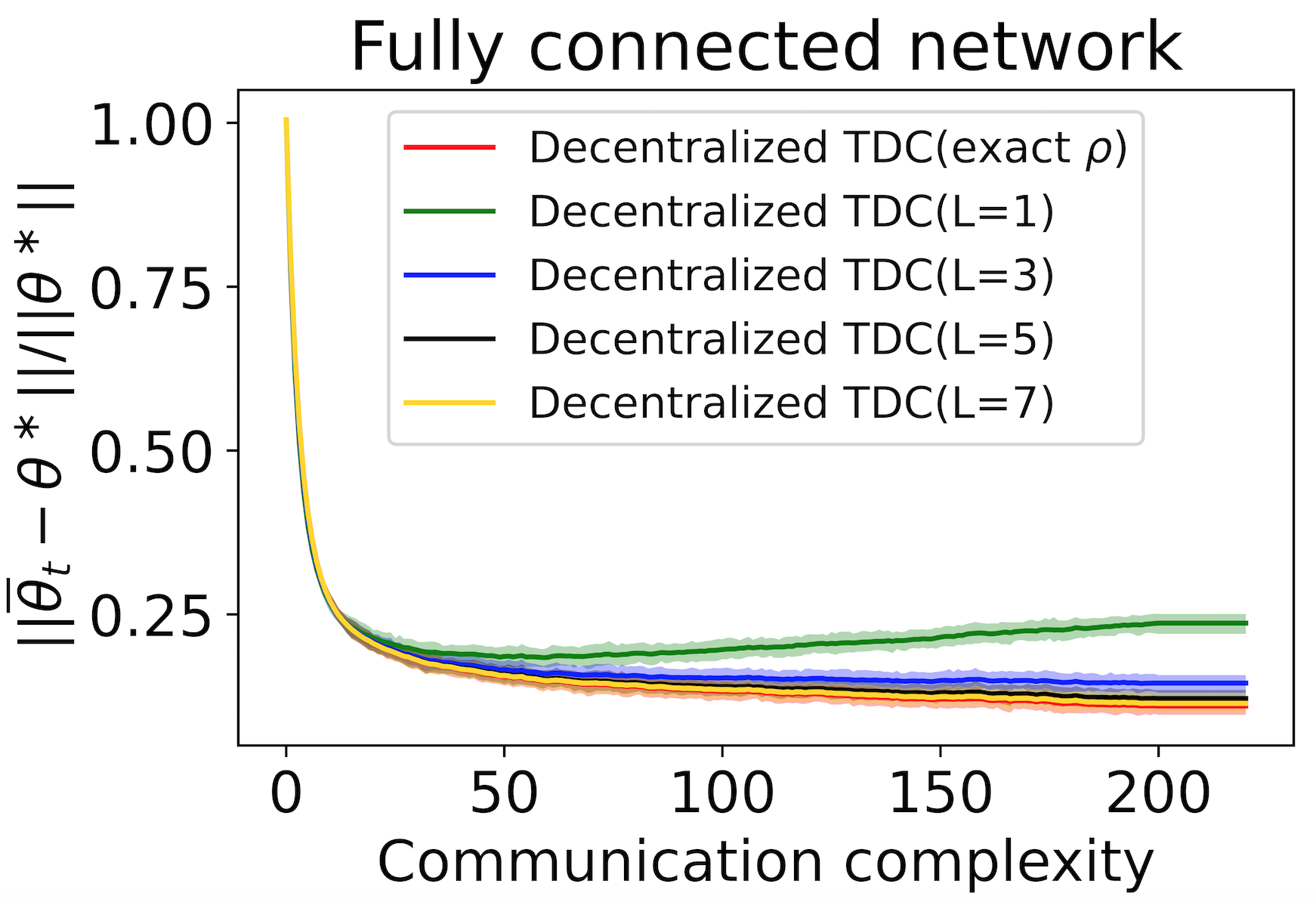}	
	\end{minipage} 
	\begin{minipage}{.23\textwidth}
		\centering
		\includegraphics[width=\textwidth]{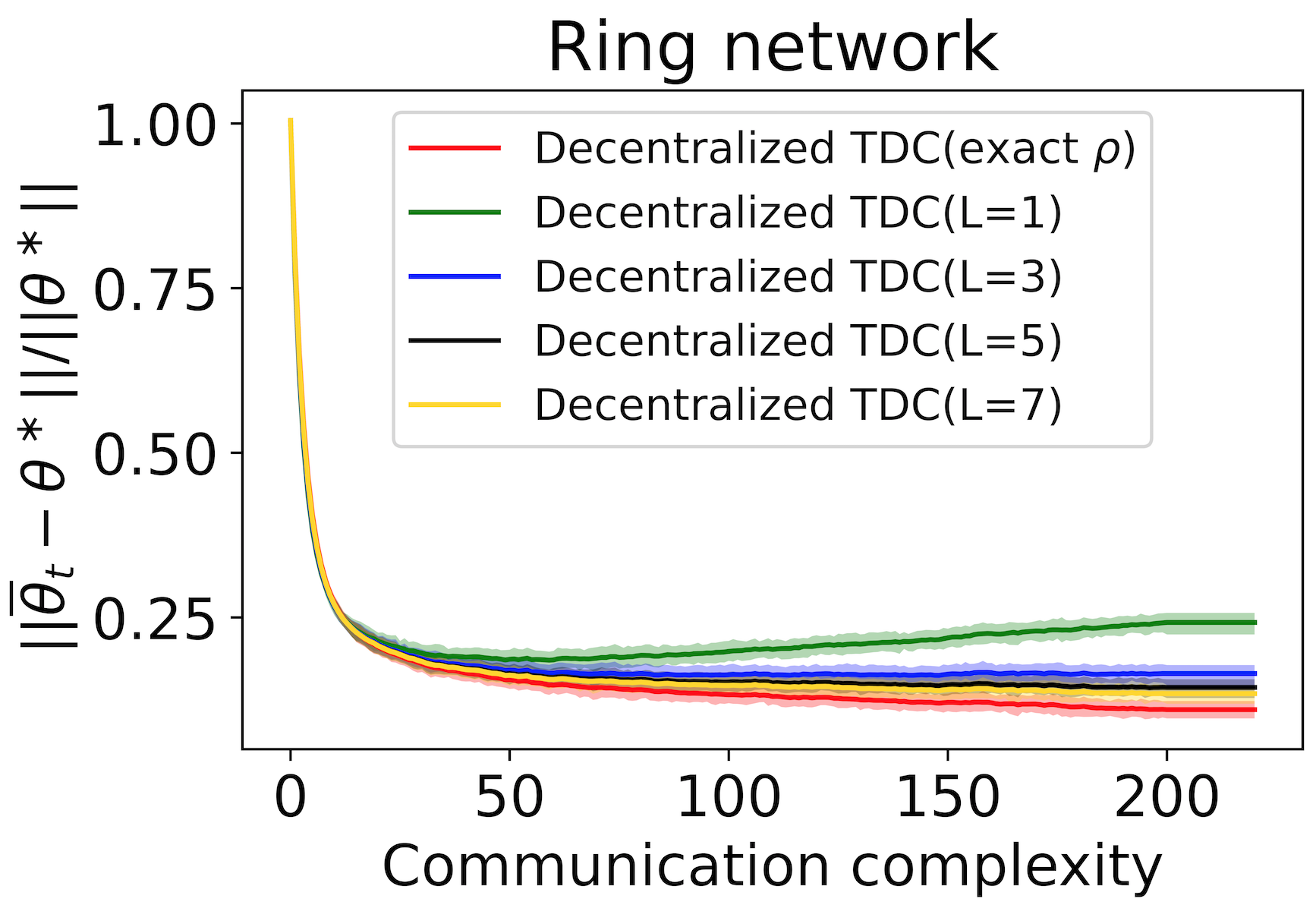}
	\end{minipage} 
	\caption{Impact of total communication rounds $L$ for estimating inexact global importance sampling ratios on asymptotic convergence error.}
	\label{fig_inexact1}
\end{figure}

We further plot the maximum relative consensus error among all agents $\max_m \|\theta_t^{(m)}-\overline{\theta}_t\|/\|\overline{\theta}^*\|$ v.s. communication complexity ($t$) in the fully-connected network (Left) and ring network (Right) in \Cref{fig_inexact2}, where the tails in both figures correspond to the extra $T'=20$ local model averaging steps. In both networks, one can see that the consensus error decreases as $L$ increases, and the extra local model averaging steps are necessary to achieve consensus. Moreover, it can be seen that the consensus errors achieved in the fully connected network are slightly smaller than those achieved in the ring network, as denser connections facilitate achieving the global consensus.

\begin{figure}[tbh]
	\begin{minipage}{.23\textwidth}
		\centering
		\includegraphics[width=\textwidth]{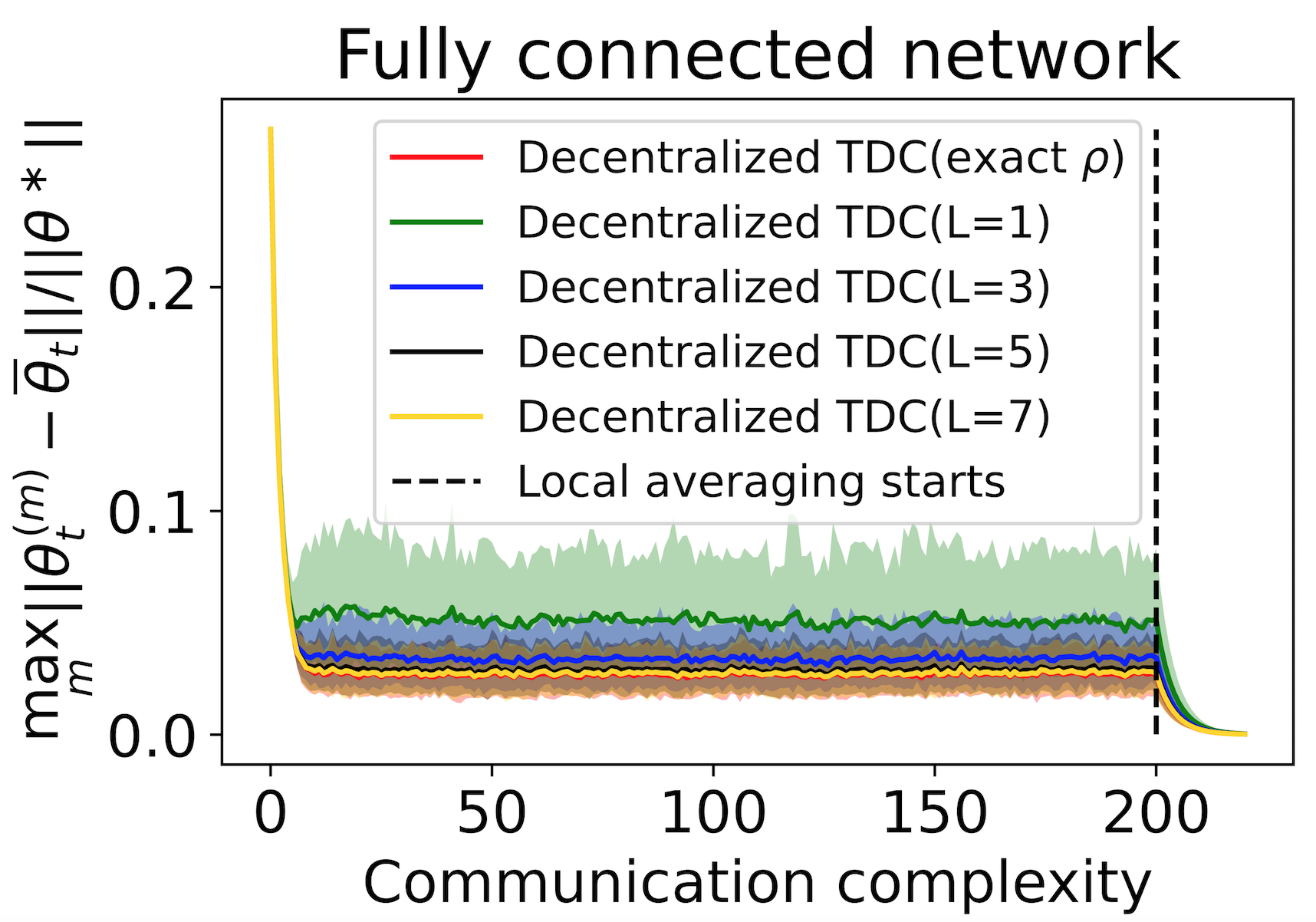}
	\end{minipage} 
	\begin{minipage}{.23\textwidth}
		\centering
		\includegraphics[width=\textwidth]{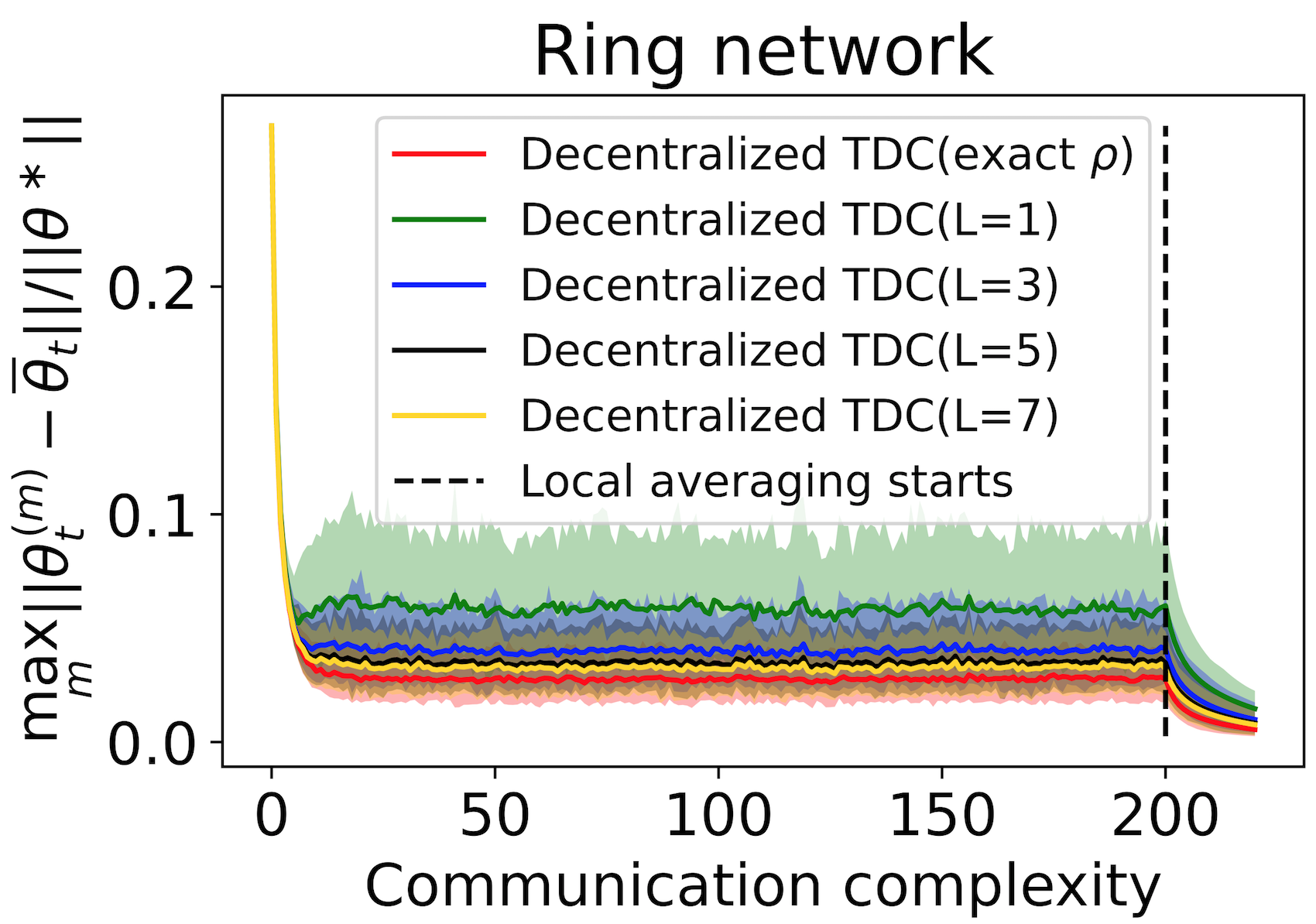}
	\end{minipage} 
	\caption{Impact of total communication rounds $L$ for estimating inexact global importance sampling ratios on consensus error.}
	\label{fig_inexact2}
\end{figure}

\section{Conclusion}
In this paper, we develop two sample-efficient and comm-unication-efficient decentralized TDC algorithms for multi-agent off-policy evaluation. Our algorithms synchronize the local importance sampling ratios among the agents and adopt mini-batch stochastic updates to save communication. In particular, the algorithms keep full privacy of the agents' local information. We prove that the proposed decentralized TDC algorithms achieve a near-optimal sample complexity as well as an optimal communication complexity that improves over the existing decentralized TD(0). In the future, we expect that our algorithms can serve as a fundamental component in the design of advanced policy optimization algorithms for MARL.

\bibliographystyle{icml2020}
\bibliography{decentralTDC.bib} 

\newpage
\onecolumn
\appendix

\addcontentsline{toc}{section}{Appendix} 
\part{Appendix} 
\parttoc 
\allowdisplaybreaks

\section{Notations, Filtration and Summary of Constants}\label{supp: notation}
\subsection*{Notations to rewrite update rules in Algorithm \ref{alg: 1}}
For convenience of convergence analysis of Algorithm \ref{alg: 1}, we define the following notations.
\begin{align}
\overline{\theta}_t=&\frac{1}{M}\sum_{m=1}^{M} \theta_{t}^{(m)}, \quad \overline{w}_t=\frac{1}{M}\sum_{m=1}^{M} w_{t}^{(m)}\nonumber\\
A_t=&\rho_t \phi(s_t) [\gamma \phi(s_{t+1})-\phi(s_t)]^{\top}, \quad B_t=-\gamma \rho_t \phi(s_{t+1}) \phi(s_t)^{\top}, \nonumber\\
C_t=&-\phi(s_t) \phi(s_t)^{\top}, \quad b_{t}^{(m)}=\rho_t R_{t}^{(m)} \phi(s_t), \quad \overline{b}_t=\frac{1}{M}\sum_{m=1}^{M} b_{t}^{(m)} \nonumber\\ 
\overline{A}_t=&\frac{1}{N}\sum_{i=tN}^{(t+1)N-1} A_i, \quad 
\overline{B}_t=\frac{1}{N}\sum_{i=tN}^{(t+1)N-1} B_i, \quad \overline{C}_t=\frac{1}{N}\sum_{i=tN}^{(t+1)N-1} C_i \nonumber\\ 
\overline{b}_{t}^{(m)}=&\frac{1}{N}\sum_{i=tN}^{(t+1)N-1} b_{i}^{(m)}, \quad \overline{\overline{b}}_t=\frac{1}{N}\sum_{i=tN}^{(t+1)N-1} \overline{b}_i=\frac{1}{M}\sum_{m=1}^{M} \overline{b}_{t}^{(m)}.\label{ABCb1} 
\end{align}


With the above notations, the update rules in \eqref{TDC1-theta}\&\eqref{TDC1-w} can be rewritten as follows.
\begin{align}
\theta_{t+1}^{(m)}&=\sum_{m'\in \mathcal{N}_m}V_{m,m'} \theta_{t}^{(m')} + \alpha\big(\overline{A}_t \theta_{t}^{(m)}+\overline{b}_{t}^{(m)}+\overline{B}_t w_{t}^{(m)}\big),\label{theta_iter1b}\\
w_{t+1}^{(m)}&=\sum_{m'\in \mathcal{N}_m}V_{m,m'}w_{t}^{(m')} + \beta\big(\overline{A}_t \theta_{t}^{(m)}+\overline{b}_{t}^{(m)}+\overline{C}_t w_{t}^{(m)}\big).\label{w_iter1b}
\end{align}
Then, by averaging the update rules in \eqref{theta_iter1b}\&\eqref{w_iter1b} over all the agents and using the notations above, we obtain the following update rules of the model average $\overline{\theta}_t, \overline{w}_t$.
\begin{align}
	\overline{\theta}_{t+1}=&\overline{\theta}_t + \alpha\big(\overline{A}_t\overline{\theta}_t+\overline{\overline{b}}_t+\overline{B}_t\overline{w}_t\big), \label{theta_iter_avg1}\\
	\overline{w}_{t+1}=&\overline{w}_t + \beta\big(\overline{A}_t\overline{\theta}_t+\overline{\overline{b}}_t+\overline{C}_t\overline{w}_t\big). \label{w_iter_avg1}
\end{align}

\subsection*{Notations to rewrite update rules in Algorithm \ref{alg: 2}}
Similarly, for convenience of analyzing the Algorithm \ref{alg: 2}, we introduce the following notations.
\begin{align}
A_{t}^{(m)}=&\widehat{\rho}_{t}^{(m)} \phi(s_t) [\gamma \phi(s_{t+1})-\phi(s_t)]^{\top}, \quad B_{t}^{(m)}=-\gamma \widehat{\rho}_{t}^{(m)} \phi(s_{t+1}) \phi(s_t)^{\top}, \nonumber\\
C_t=&-\phi(s_t) \phi(s_t)^{\top}, \quad \widetilde{b}_{t}^{(m)}=\widehat{\rho}_{t}^{(m)} R_{t}^{(m)} \phi(s_t), \nonumber\\ 
\overline{A}_{t}^{(m)}=&\frac{1}{N}\sum_{i=tN}^{(t+1)N-1} A_{i}^{(m)}, \quad 
\overline{B}_{t}^{(m)}=\frac{1}{N}\sum_{i=tN}^{(t+1)N-1} B_{i}^{(m)}, \nonumber\\
\overline{C}_{t}=&\frac{1}{N}\sum_{i=tN}^{(t+1)N-1} C_{i}, \quad,
\overline{\widetilde{b}}_{t}^{(m)}=\frac{1}{N}\sum_{i=tN}^{(t+1)N-1} \widetilde{b}_{i}^{(m)}.\label{ABCb2}
\end{align}

With the above notations, the update rules in \eqref{theta_iter2}\&\eqref{w_iter2} can be rewritten as follows.
\begin{align}
\theta_{t+1}^{(m)}&=\sum_{m'\in \mathcal{N}_m}V_{m,m'} \theta_{t}^{(m')} + \alpha\big(\overline{A}_{t}^{(m)} \theta_{t}^{(m)}+\overline{\widetilde{b}}_{t}^{(m)}+\overline{B}_{t}^{(m)} w_{t}^{(m)}\big),\label{theta_iter2b}\\
w_{t+1}^{(m)}&=\sum_{m'\in \mathcal{N}_m}V_{m,m'}w_{t}^{(m')} +\beta\big(\overline{A}_{t}^{(m)} \theta_{t}^{(m)}+\overline{\widetilde{b}}_{t}^{(m)}+\overline{C}_{t} w_{t}^{(m)}\big).\label{w_iter2b}
\end{align}

Then, by averaging the update rules in \eqref{theta_iter2b}\&\eqref{w_iter2b} over all the agents and using the notations above, we obtain the following update rules of the model average $\overline{\theta}_t,\overline{w}_t$:
\begin{align}
\overline{\theta}_{t+1}=&\overline{\theta}_t+\frac{\alpha}{M}\sum_{m=1}^{M} \big(\overline{A}_{t}^{(m)} \theta_{t}^{(m)}+\overline{\widetilde{b}}_{t}^{(m)}+\overline{B}_{t}^{(m)} w_{t}^{(m)}\big), \label{theta_iter_avg2}\\
\overline{w}_{t+1}=&\overline{w}_t+\frac{\beta}{M}\sum_{m=1}^{M} \big(\overline{A}_{t}^{(m)} \theta_{t}^{(m)}+\overline{\widetilde{b}}_{t}^{(m)}+\overline{C}_{t} w_{t}^{(m)}\big). \label{w_iter_avg2}
\end{align}

\subsection*{Filtration}
We define the filtration $\mathcal{F}_t=\sigma\big(\{s_{t'},a_{t'}\}_{t'=1}^{tN-1}\cup\{s_{tN}\}\big)$. It can be verified that
\begin{align}
	\overline{A}_t, \overline{B}_t, \overline{C}_t, \overline{b}_t^{(m)}, \overline{\overline{b}}_t, 	\overline{A}_t^{(m)}, \overline{B}_t^{(m)}, \overline{\widetilde{b}}_t^{(m)} \in& \mathcal{F}_{t+1}/\mathcal{F}_t. \label{ABCb_filter}
\end{align}
Hence, for both Algorithm \ref{alg: 1} and Algorithm \ref{alg: 2}, their parameters satisfy
\begin{align}
	\theta_{t}^{(m)}, \overline{\theta}_{t}, w_{t}^{(m)}, \overline{w}_t, w_t^* \in& \mathcal{F}_t/\mathcal{F}_{t-1}, \label{para_filter}
\end{align}
where $w_t^* :=-C^{-1}(A\overline{\theta}_t+b)$.

\subsection*{Additional Notations and Constants}
\begin{itemize}
\item Agent index set: $\mathcal{M}=\{1,2,\ldots,M\}$.
\item Collections of actions and policies among agents: $a_t=\{a_{t}^{(m)}\}_{m=1}^M$, $\pi=\{\pi^{(m)}\}_{m=1}^M$, $\pi_b=\{\pi_{b}^{(m)}\}_{m=1}^M$.
\item Local importance sampling ratio: $\rho^{(m)}(s,a)=\pi^{(m)}(a|s)/\pi_{b}^{(m)}(a|s)\in [\rho_{\min},\rho_{\max}]$, \\
$\rho_{t}^{(m)} = \rho^{(m)}(s_t,a_{t}^{(m)})=\pi^{(m)}(a_{t}^{(m)}|s_t)/\pi_{b}^{(m)}(a_{t}^{(m)}|s_t)\in [\rho_{\min},\rho_{\max}]$.
\item Global importance sampling ratio: $\rho(s,\{a^{(m)}\}_{m})=\prod_{m=1}^{M}\rho^{(m)}(s,a^{(m)})$,\\
$\rho_{t}=\rho(s_t,a_t)=\prod_{m=1}^{M}\rho_{t}^{(m)}=\prod_{m=1}^{M}\frac{\pi^{(m)}(a_{t}^{(m)}|s_t)}{\pi_{b}^{(m)}(a_{t}^{(m)}|s_t)}\in [\rho_{\min},\rho_{\max}]$, \\
$\widehat{\rho}_t^{(m)}$ is the estimated global importance sampling ratio obtained by agent $m$.
\item Expected values: $A\red{\stackrel{\triangle}{=}}\mathbb{E}_{\pi_b}[A_t]\red{=\mathbb{E}_{\pi_b}[\overline{A}_t]}$, $B\stackrel{\triangle}{=}\mathbb{E}_{\pi_b}[B_t]\red{=\mathbb{E}_{\pi_b}[\overline{B}_t]}$, $C\stackrel{\triangle}{=}\mathbb{E}_{\pi_b}[C_t]\red{=\mathbb{E}_{\pi_b}[\overline{C}_t]}$, $b^{(m)}\red{\stackrel{\triangle}{=}\mathbb{E}_{\pi_b}\big[b_{t}^{(m)}\big]}=\mathbb{E}_{\pi_b}\big[\overline{b}_{t}^{(m)}\big]$, $b\stackrel{\triangle}{=}\red{\mathbb{E}_{\pi_b}\big[b_{t}\big]=}\mathbb{E}_{\pi_b}\big[\overline{b}_{t}\big]=\mathbb{E}_{\pi_b}\big[\overline{\overline{b}}_{t}\big]=\frac{1}{M}\sum_{m=1}^{M} b^{(m)}$, 
where $\mathbb{E}_{\pi_b}$ denotes the expectation when $s_t\sim\mu_{\pi_{b}}$ and $a_t\sim \pi_b(\cdot|s_t)$. We will use the relationship $C=A^{\top}+B=A+B^{\top}$ later. 
\item Target parameter values: $\theta^*=-A^{-1}b$ and $w_t^*=-C^{-1}(A\overline{\theta}_{t}+b)$.
\item Parameter matrices: 
$\Theta_t=[\theta_t^{(1)};\theta_t^{(2)};\ldots;\theta_t^{(M)}]^{\top}\red{\in \mathbb{R}^{M\times d}}$, $W_t=[w_t^{(1)};w_t^{(2)};\ldots;w_t^{(M)}]^{\top}\red{\in \mathbb{R}^{M\times d}}$.
\item $\Omega_A=\rho_{\max}(1+\gamma)$, $\Omega_B=\rho_{\max}\gamma$, 
$\Omega_b=\rho_{\max}R_{\max}$.
\item \red{$D_A=(1+\gamma)^2(\rho_{\max}^2/\rho_{\min})\ln^2 (\rho_{\max}/\rho_{\min})$, $D_B=\gamma^2(\rho_{\max}^2/\rho_{\min})\ln^2 (\rho_{\max}/\rho_{\min})$, \\
$D_b=R_{\max}^2 (\rho_{\max}^2/\rho_{\min})\ln^2 (\rho_{\max}/\rho_{\min})$.}
\item \red{$\widetilde{\Omega}_A=\Omega_A+\sqrt{D_A}$, $\widetilde{\Omega}_B=\Omega_B+\sqrt{D_B}$, $\widetilde{\Omega}_b=\Omega_b+\sqrt{D_b}$.}
\item $\lambda_{1}=-\lambda_{\max}(A^{\top}C^{-1}A)>0$, $\lambda_{2}=-\lambda_{\max}(C)>0$. $\sigma_2\in(0,1)$ is the second largest singular value of $V$. 
\item Difference matrix: $\Delta=I-\frac{1}{M}\mathbf{1}\mathbf{1}^{\top} \in \mathbb{R}^{M\times M}$. 
\item $c_{\text{sd}}=\frac{2\nu}{1-\delta}$, $c_{\text{var}}=\frac{8(\nu+1-\delta)}{1-\delta}$, $c_{\text{var,2}}=\frac{8(1+\|C^{-1}\|_F^2)(\nu+1-\delta)}{1-\delta}$,  $c_{\text{var,3}}=\frac{8(1+\Omega_{B}^2\|C^{-1}\|_F^2)(\nu+1-\delta)}{1-\delta}$.
\item  $c_{\text{para}}=\max_{m\in\mathcal{M}}\|\theta_{0}^{(m)}\| + \max_{m\in\mathcal{M}}\|w_{0}^{(m)}\| + \red{\frac{2\widetilde{\Omega}_{b}}{\widetilde{\Omega}_{A}+1}}$
\item $c_1=\frac{20}{\lambda_2}\Omega_{A}^2\Omega_{B}^2\|C^{-1}\|_F^2$, $c_2=\frac{14 c_{\text{var,2}}\Omega_{A}^2}{\lambda_{2}}$, $c_3= \frac{40}{\lambda_2}\Omega_{A}^4\|C^{-1}\|_F^2 \big(1+\Omega_{B}^2\|C^{-1}\|_F^2\big)$, $c_4=\frac{14 c_{\text{var,2}}}{\lambda_{2}} \big(\Omega_{A}^2\big\|\theta^*\big\|^2 + \Omega_{b}^2\big)$,\\ $c_5=\frac{20}{\lambda_2}\Omega_{A}^2\|C^{-1}\|_F^2 \big[c_{\text{var,3}}\Omega_{A}^2\big\|\theta^*\big\|^2+c_{\text{var}}\Omega_{b}^2\big(1+\Omega_{B}^2\|C^{-1}\|_F^2\big)\big]$, $c_6=\frac{9\Omega_{B}^2}{\lambda_{1}}$,\\ $c_7=\frac{9}{\lambda_{1}}  \big[c_{\text{var,3}}\Omega_{A}^2\big\|\theta^*\big\|^2+c_{\text{var}}\Omega_{b}^2\big(1+\Omega_{B}^2\|C^{-1}\|_F^2\big)\big]$, $c_8=\frac{28c_{\text{var,2}}\Omega_A^2}{\lambda_2}$, $c_9=\frac{28 c_{\text{var,2}}}{\lambda_{2}} \big(\Omega_{A}^2\big\|\theta^*\big\|^2 + \Omega_{b}^2\big)$,\\
$c_{10}=\frac{156M^2}{\lambda_2}(D_A+D_B)c_{\text{para}}^2$, $c_{11}=\frac{488M^2D_b}{\lambda_2}$, $c_{12}=\|\Delta \Theta_{0}\|_F+\|\Delta W_{0}\|_F + \frac{4\beta\Omega_b\sqrt{M}}{1-\sigma_2}$, \\
\red{$c_{13}=2\max(\widetilde{\Omega}_{A},\widetilde{\Omega}_{B},1)$, $c_{14}=2M \big[\widetilde{\Omega}_A +(\widetilde{\Omega}_B+1)(1+\Omega_A\|C^{-1}\|_F)\big]$,\\ $c_{15}=2M\big[\widetilde{\Omega}_{b}+\Omega_{b}(\widetilde{\Omega}_B+1)\|C^{-1}\|_F+\big(\widetilde{\Omega}_A+\Omega_{A}(\widetilde{\Omega}_{B}+1)\|C^{-1}\|_F\big)\|\theta^*\|\big]$, $c_{16}=\frac{12Mc_{13}^2(1+\sigma_2^2)}{1-\sigma_2^2}$,\\ $c_{17}=\frac{12c_{14}^2(1+\sigma_2^2)}{1-\sigma_2^2}$, $c_{18}=\frac{6c_{15}^2(1+\sigma_2^2)}{1-\sigma_2^2}$, $c_{19}=c_{18}+c_{17}\Big( \big\|\overline{\theta}_{0}-\theta^*\big\|^2+\big\|\overline{w}_{0}-w_0^*\big\|^2 + \frac{2c_9}{c_8} +\frac{c_{11}}{\lambda_1}\Big)$, $c_{20}=\frac{\lambda_2 c_{10}c_{17}}{16}$.}
\end{itemize}

\section{Proof of Theorem \ref{thm_sync}} \label{appendix_thm_sync}

\thmsync*

\begin{proof}
	We first bound the tracking error. Note that
	\begin{align}
	&\mathbb{E}\big[\|\overline{w}_{t+1}-w_t^*\|^2\big|\mathcal{F}_t\big]\nonumber\\
	&\red{\stackrel{(i)}{=}} \mathbb{E}\big[\|\overline{w}_{t}+\beta(\overline{A}_{t}\overline{\theta}_{t}+\overline{\overline{b}}_{t}+\overline{C}_{t}\overline{w}_{t})-w_t^*\|^2 \big|\mathcal{F}_t\big] \nonumber\\
	&= \|\overline{w}_{t}-w_t^*\|^2 + 2\beta(\overline{w}_{t}-w_t^*)^{\top}\mathbb{E}\big[\overline{A}_{t}\overline{\theta}_{t}+\overline{\overline{b}}_{t}+\overline{C}_{t}\overline{w}_{t}\big|\mathcal{F}_t\big] + \beta^2 \mathbb{E}\big[\big\|\overline{A}_{t}\overline{\theta}_{t}+\overline{\overline{b}}_{t}+\overline{C}_{t}\overline{w}_{t}\big\|^2\big|\mathcal{F}_t\big] \nonumber\\
	&\stackrel{(ii)}{\le} \|\overline{w}_{t}-w_t^*\|^2 + \beta\Big(\frac{2c_{\text{sd}}}{N}-\lambda_{2}\Big) \big\|\overline{w}_{t}-w_t^*\big\|^2 + \frac{3\beta c_{\text{var,2}}\Omega_{A}^2}{N\lambda_{2}} \big\|\overline{\theta}_{t}-\theta^*\big\|^2 + \frac{3\beta c_{\text{var,2}}}{N\lambda_{2}} \big(\Omega_{A}^2\big\|\theta^*\big\|^2 + \Omega_{b}^2\big) \nonumber\\
	&\quad +\frac{4\beta^2 c_{\text{var,2}}\Omega_{A}^2}{N} \|\overline{\theta}_t-\theta^*\|^2 + 4\beta^2 \big\|\overline{w}_{t}-w_t^*\big\|^2  + \frac{4\beta^2 c_{\text{var,2}}}{N} (\Omega_{A}^2\|\theta^*\|^2 + \Omega_{b}^2) \nonumber\\
	&\le \Big[1+\beta\Big(\frac{2c_{\text{sd}}}{N}-\lambda_{2}+4\beta\Big)\Big] \|\overline{w}_{t}-w_t^*\|^2 + \frac{\beta c_{\text{var,2}}\Omega_{A}^2}{N} \Big(\frac{3}{\lambda_{2}}+4\beta\Big) \big\|\overline{\theta}_{t}-\theta^*\big\|^2 + \frac{\beta c_{\text{var,2}}}{N} \Big(\frac{3}{\lambda_{2}}+4\beta\Big) \big(\Omega_{A}^2\big\|\theta^*\big\|^2 + \Omega_{b}^2\big) \nonumber\\
	&\stackrel{(iii)}{\le} \Big(1-\frac{\beta\lambda_{2}}{2}\Big) \|\overline{w}_{t}-w_t^*\|^2 + \frac{7\beta c_{\text{var,2}}\Omega_{A}^2}{N\lambda_{2}} \big\|\overline{\theta}_{t}-\theta^*\big\|^2 + \frac{7\beta c_{\text{var,2}}}{N\lambda_{2}} \big(\Omega_{A}^2\big\|\theta^*\big\|^2 + \Omega_{b}^2\big) \label{wdelay_err}
	\end{align}
	where \red{(i) uses eq. \eqref{w_iter_avg1}}, (ii) uses eqs. \eqref{SGbias_w} \& \eqref{SGvar_w}, and (iii) uses the conditions that $N\ge\frac{8c_{\text{sd}}}{\lambda_2}$, and $\beta\le\min\Big(\frac{\lambda_2}{16}, \frac{1}{\lambda_{2}}\Big)$. Hence, we further obtain that
	\begin{align}
	&\mathbb{E}\big[\|\overline{w}_{t+1}-w_{t+1}^*\|^2\big|\mathcal{F}_t\big]\nonumber\\
	&\stackrel{(i)}{\le} \Big(1+\frac{1}{2[2/(\beta\lambda_{2})-1]}\Big) \mathbb{E}\big[\|\overline{w}_{t+1}-w_t^*\|^2\big|\mathcal{F}_t\big] + \big[1+2\big(2/(\beta\lambda_{2})-1\big)\big] \mathbb{E}\big[\|w_{t+1}^*-w_t^*\|^2\big|\mathcal{F}_t\big] \nonumber\\
	&\stackrel{(ii)}{\le} \frac{4/(\beta\lambda_{2})-1}{2[2/(\beta\lambda_{2})-1]} \Big[\Big(1-\frac{\beta\lambda_{2}}{2}\Big) \|\overline{w}_{t}-w_t^*\|^2 + \frac{7\beta c_{\text{var,2}}\Omega_{A}^2}{N\lambda_{2}} \big\|\overline{\theta}_{t}-\theta^*\big\|^2 + \frac{7\beta c_{\text{var,2}}}{N\lambda_{2}} \big(\Omega_{A}^2\big\|\theta^*\big\|^2 + \Omega_{b}^2\big) \Big]\nonumber\\ 
	&\quad+\Big(\frac{4}{\beta\lambda_2}-1\Big) \mathbb{E}\big[\|C^{-1}A(\overline{\theta}_{t+1}-\overline{\theta}_{t})\|^2\big|\mathcal{F}_t\big] \nonumber\\
	&\stackrel{(iii)}{\le} \Big(1-\frac{\beta\lambda_{2}}{4}\Big)\|\overline{w}_{t}-w_t^*\|^2 + \frac{14\beta c_{\text{var,2}}\Omega_{A}^2}{N\lambda_{2}} \big\|\overline{\theta}_{t}-\theta^*\big\|^2 + \frac{14\beta c_{\text{var,2}}}{N\lambda_{2}} \big(\Omega_{A}^2\big\|\theta^*\big\|^2 + \Omega_{b}^2\big) \nonumber\\
	&\quad +\Omega_{A}^2\|C^{-1}\|_F^2 \Big(\frac{4}{\beta\lambda_2}-1\Big) \mathbb{E}\big[ \big\|\alpha\big(\overline{A}_t\overline{\theta}_t+\overline{\overline{b}}_t+\overline{B}_t\overline{w}_t\big)\big\|^2 \big|\mathcal{F}_t\big] \nonumber\\
	&\stackrel{(iv)}{=} \Big(1-\frac{\beta\lambda_{2}}{4}\Big)\|\overline{w}_{t}-w_t^*\|^2 + \frac{14\beta c_{\text{var,2}}\Omega_{A}^2}{N\lambda_{2}} \big\|\overline{\theta}_{t}-\theta^*\big\|^2 + \frac{14\beta c_{\text{var,2}}}{N\lambda_{2}} \big(\Omega_{A}^2\big\|\theta^*\big\|^2 + \Omega_{b}^2\big) \nonumber\\
	&\quad +\frac{4\alpha^2}{\beta\lambda_2}\Omega_{A}^2\|C^{-1}\|_F^2 \Big[10\Omega_{A}^2\big(1+\Omega_{B}^2\|C^{-1}\|_F^2\big)\big\|\overline{\theta}_{t}-\theta^*\big\|^2 + 5\Omega_{B}^2\big\|\overline{w}_{t}-w_t^*\big\|^2 \nonumber\\
	&\quad +\frac{5}{N} 
	\big[c_{\text{var,3}}\Omega_{A}^2\big\|\theta^*\big\|^2+c_{\text{var}}\Omega_{b}^2\big(1+\Omega_{B}^2\|C^{-1}\|_F^2\big)\big] \Big] \nonumber\\
	&\stackrel{(v)}{=} \Big(1-\frac{\beta\lambda_{2}}{4}+\frac{\alpha^2 c_1}{\beta}\Big)\|\overline{w}_{t}-w_t^*\|^2 + \Big(\frac{\beta c_2}{N} + \frac{\alpha^2 c_3}{\beta}\Big) \big\|\overline{\theta}_{t}-\theta^*\big\|^2 + \Big(\frac{\beta c_4}{N} + \frac{\alpha^2 c_5}{N\beta}\Big) \nonumber\\
	&\stackrel{(vi)}{\le} \Big(1-\frac{\beta\lambda_{2}}{8}\Big)\|\overline{w}_{t}-w_t^*\|^2 + \Big(\frac{\beta c_2}{N} + \frac{\alpha^2 c_3}{\beta}\Big) \big\|\overline{\theta}_{t}-\theta^*\big\|^2 + \frac{2\beta c_4}{N} \label{w_err_iter}
	\end{align}
	where (i) uses the inequality that $\|a_1+a_2\|^2\le (1+\sigma)\|a_1\|^2+(1+\sigma^{-1})\|a_2\|^2$ for any $a,b\in\mathbb{R}^d$ and $\sigma>0$, (ii) uses eq. \eqref{wdelay_err} as well as the notation that $w_t^*=-C^{-1}(A\overline{\theta}_t+b)$, (iii) uses $\beta\le\frac{1}{\lambda_{2}}$ (this implies that $\frac{4/(\beta\lambda_{2})-1}{2[2/(\beta\lambda_{2})-1]}\le\frac{3}{2}<2$) and eqs. \eqref{theta_iter_avg1} \& \eqref{bound_A}, (iv) uses eq. \eqref{SGvar_theta}, and (v) denotes that $c_1=\frac{20}{\lambda_2}\Omega_{A}^2\Omega_{B}^2\|C^{-1}\|_F^2$, $c_2=\frac{14 c_{\text{var,2}}\Omega_{A}^2}{\lambda_{2}}$, $c_3= \frac{40}{\lambda_2}\Omega_{A}^4\|C^{-1}\|_F^2 \big(1+\Omega_{B}^2\|C^{-1}\|_F^2\big)$, $c_4=\frac{14 c_{\text{var,2}}}{\lambda_{2}} \big(\Omega_{A}^2\big\|\theta^*\big\|^2 + \Omega_{b}^2\big)$, $c_5=\frac{20}{\lambda_2}\Omega_{A}^2\|C^{-1}\|_F^2 \big[c_{\text{var,3}}\Omega_{A}^2\big\|\theta^*\big\|^2+c_{\text{var}}\Omega_{b}^2\big(1+\Omega_{B}^2\|C^{-1}\|_F^2\big)\big]$, and (vi) uses the conditions that $\alpha\le \beta \min\Big(\frac{1}{3}\sqrt{\frac{\lambda_2}{c_1}},  \sqrt{\frac{c_4}{c_5}}\Big)$. 
	
	On the other hand, we can bound the convergence error of the model parameters as follows.
	\begin{align}
	&\mathbb{E}\big[\|\overline{\theta}_{t+1}-\theta^*\|^2\big|\mathcal{F}_t\big]\nonumber\\
	&\red{\stackrel{(i)}{=}} \mathbb{E}\big[\|\overline{\theta}_{t}+\alpha\big(\overline{A}_{t}\overline{\theta}_{t}+\overline{\overline{b}}_{t}+\overline{B}_{t}\overline{w}_{t}\big)-\theta^*\|^2 \big|\mathcal{F}_t\big] \nonumber\\
	&= \|\overline{\theta}_{t}-\theta^*\|^2 + 2\alpha\big(\overline{\theta}_{t}-\theta^*\big)^{\top}\mathbb{E}\big[\overline{A}_{t}\overline{\theta}_{t}+\overline{\overline{b}}_{t}+\overline{B}_{t}\overline{w}_{t}\big|\mathcal{F}_t\big] + \alpha^2 \mathbb{E}\big[\big\|\overline{A}_{t}\overline{\theta}_{t}+\overline{\overline{b}}_{t}+\overline{B}_{t}\overline{w}_{t}\big\|^2\big|\mathcal{F}_t\big] \nonumber\\
	&\stackrel{(ii)}{\le} \Big[1+\alpha\Big( \frac{2c_{\text{var,3}}\Omega_{A}^2}{N\lambda_{1}}-\lambda_{1} \Big)+10\alpha^2\Omega_{A}^2\big(1+\Omega_{B}^2\|C^{-1}\|_F^2\big)\Big]\|\overline{\theta}_{t}-\theta^*\|^2 +\alpha\Omega_{B}^2 \Big(\frac{4}{\lambda_{1}}+5\alpha\Big)\big\|\overline{w}_{t}-w_t^*\big\|^2 \nonumber\\ 
	&\quad + \frac{\alpha}{N} \Big(\frac{4}{\lambda_{1}}+5\alpha\Big)
	\big[c_{\text{var,3}}\Omega_{A}^2\big\|\theta^*\big\|^2+c_{\text{var}}\Omega_{b}^2\big(1+\Omega_{B}^2\|C^{-1}\|_F^2\big)\big] \nonumber\\
	&\stackrel{(iii)}{\le} \Big(1-\frac{\alpha\lambda_{1}}{2}\Big)\|\overline{\theta}_{t}-\theta^*\|^2 + \frac{9\alpha\Omega_{B}^2}{\lambda_{1}} \big\|\overline{w}_{t}-w_t^*\big\|^2 + \frac{9\alpha}{N\lambda_{1}} \big[c_{\text{var,3}}\Omega_{A}^2\big\|\theta^*\big\|^2+c_{\text{var}}\Omega_{b}^2\big(1+\Omega_{B}^2\|C^{-1}\|_F^2\big)\big] \nonumber\\
	&\stackrel{(iv)}{\le} \Big(1-\frac{\alpha\lambda_{1}}{2}\Big)\|\overline{\theta}_{t}-\theta^*\|^2 + \alpha c_6 \big\|\overline{w}_{t}-w_t^*\big\|^2 + \frac{\alpha c_7}{N} \label{theta_err_iter}
	\end{align}
	where \red{(i) uses eq. \eqref{theta_iter_avg1},} (ii) uses eqs. \eqref{SGbias_theta} \& \eqref{SGvar_theta}, (iii) uses the conditions that $N\ge \frac{8c_{\text{var,3}}\Omega_{A}^2}{\lambda_{1}^2}$ and that $\alpha\le \min\Big[\frac{\lambda_1}{40\Omega_{A}^2\big(1+\Omega_{B}^2\|C^{-1}\|_F^2\big)},\frac{1}{\lambda_{1}}\Big]$, and (iv) uses the notations that $c_6=\frac{9\Omega_{B}^2}{\lambda_{1}}$ and that $c_7=\frac{9}{\lambda_{1}}  \big[c_{\text{var,3}}\Omega_{A}^2\big\|\theta^*\big\|^2+c_{\text{var}}\Omega_{b}^2\big(1+\Omega_{B}^2\|C^{-1}\|_F^2\big)\big]$.
	
	Taking expectation on both sides of eqs. \eqref{w_err_iter}\&\eqref{theta_err_iter} and then summing them up yields that 
	\begin{align}
	&{\mathbb{E}(\|\overline{\theta}_{t+1}-\theta^*\|^2)} + {\mathbb{E}(\|\overline{w}_{t+1}-w_{t+1}^*\|^2)} \nonumber\\
	&\le \Big(1-\frac{\alpha\lambda_{1}}{2}+\frac{\beta c_2}{N} + \frac{\alpha^2 c_3}{\beta}\Big) \mathbb{E}(\|\overline{\theta}_{t}-\theta^*\|^2) + \Big(1-\frac{\beta\lambda_{2}}{8}+\alpha c_6\Big) \mathbb{E}(\|\overline{w}_{t}-w_{t}^*\|^2) + \frac{\alpha c_7+2\beta c_4}{N} \nonumber\\
	&\stackrel{(i)}{\le} \Big(1-\frac{\alpha\lambda_{1}}{4}\Big) \mathbb{E}(\|\overline{\theta}_{t}-\theta^*\|^2) + \Big(1-\frac{\beta\lambda_{2}}{16}\Big) \mathbb{E}(\|\overline{w}_{t}-w_{t}^*\|^2) + \frac{3\beta c_4}{N} \nonumber\\
	&\stackrel{(ii)}{\le} \Big(1-\frac{\alpha\lambda_{1}}{4}\Big) \big[\mathbb{E}(\|\overline{\theta}_{t}-\theta^*\|^2)+\mathbb{E}(\|\overline{w}_{t}-w_{t}^*\|^2)\big] + \frac{3\beta c_4}{N}, \nonumber
	\end{align}
	where (i) uses the conditions that $N\ge \frac{8\beta c_2}{\alpha \lambda_1}$ and that $\alpha\le\beta \min \Big(\frac{\lambda_2}{16 c_6}, \frac{\lambda_1}{8c_3}, \frac{c_4}{c_7}\Big)$, and (ii) uses the condition that $\alpha\le \frac{\beta\lambda_2}{4\lambda_1}$. Iterating the inequality above proves eq. \eqref{sum_err}. 
	
Next, we prove eq. \eqref{consensus_err1}. 
Note that the local model averaging iterations can be rewritten into the matrix-vector form as $\Theta_{t+1}=V\Theta_t$ where $T\le t\le T+T'$ and $\Theta_t\stackrel{\triangle}{=}[\theta_{t}^{(1)}; \theta_{t}^{(2)}; \ldots; \theta_{t}^{(M)}]^{\top}$. Hence, it can be derived from Lemma \ref{lemma_V} that
	\begin{align}
		\|\Delta\Theta_{T+T'}\|_F=\|\Delta V^{T'}\Theta_T\|_F=\|V^{T'}\Delta\Theta_T\|_F \le\sigma_2^{T'}\|\Delta\Theta_T\|_F.  \label{post_consensus_decay}
	\end{align}
	To obtain an upper bound of $\|\Delta\Theta_T\|_F$, we rewrite the update rules \eqref{theta_iter1b}\&\eqref{w_iter1b} of Algorithm \ref{alg: 1} into the following matrix-vector form.
	\begin{align}
	\Theta_{t+1}=&V\Theta_t + \alpha\big(\Theta_t \overline{A}_t^{\top}+\big[\overline{b}_t^{(1)};\ldots;\overline{b}_t^{(M)}\big]^{\top}+W_t\overline{B}_t^{\top}\big), \label{Theta1}\\
	W_{t+1}=&VW_t + \beta\big(\Theta_t \overline{A}_t^{\top}+\big[\overline{b}_t^{(1)};\ldots;\overline{b}_t^{(M)}\big]^{\top}+W_t\overline{C}_t^{\top}\big).\label{W1}
	\end{align}
	Hence, eq. \eqref{Theta1} implies that
	\begin{align}
		\|\Delta\Theta_{t+1}\|_F&\stackrel{(i)}{\le} \|V\Delta\Theta_t\|_F+\alpha\big(\|\Delta\Theta_t\|_F \|\overline{A}_t\|_F+\|\Delta\big[\overline{b}_t^{(1)};\ldots;\overline{b}_t^{(M)}\big]^{\top}\|_F+\|\Delta W_t\|_F \|\overline{B}_t\|_F\big) \nonumber\\
		&\stackrel{(ii)}{\le} (\sigma_2+\alpha\Omega_A)\|\Delta\Theta_t\|_F + \alpha\sqrt{\sum_{m=1}^{M} \|\overline{b}_t^{(m)}-\overline{\overline{b}}_t\|^2} + \alpha\Omega_B\|\Delta W_t\|_F \nonumber\\
		&\stackrel{(iii)}{\le} (\sigma_2+\alpha\Omega_A)\|\Delta\Theta_t\|_F + \alpha\sqrt{\sum_{m=1}^{M} \|\overline{b}_t^{(m)}\|^2} + \alpha\Omega_B\|\Delta W_t\|_F \nonumber\\
		&\stackrel{(iv)}{\le} (\sigma_2+\alpha\Omega_A)\|\Delta\Theta_t\|_F + \alpha\Omega_b\sqrt{M} + \alpha\Omega_B\|\Delta W_t\|_F, \label{Theta_consensus1}
	\end{align}
	where (i) uses the item 1 of Lemma \ref{lemma_V} that $\Delta V=V\Delta$, (ii) uses the item 3 of Lemma \ref{lemma_V} as well as eqs \eqref{bound_A}\&\eqref{bound_B}, (iii) uses the inequality that $\mathbb{E}\|X-\mathbb{E}X\|^2 \le \mathbb{E}\|X\|^2$ for any random vector $X$, (iv) uses eq. \eqref{bound_b}. Similarly, we can obtain that
	\begin{align}
		\|\Delta W_{t+1}\|_F&\le \beta\Omega_A\|\Delta\Theta_t\|_F + \beta\Omega_b\sqrt{M} + (\sigma_2+\beta)\|\Delta W_t\|_F. \label{W_consensus1}
	\end{align}
	Summing up eqs. \eqref{Theta_consensus1}\&\eqref{W_consensus1} yields that
	\begin{align}
		\|\Delta \Theta_{t+1}\|_F+\|\Delta W_{t+1}\|_F&\le [\sigma_2+(\alpha+\beta)\Omega_A]\|\Delta\Theta_t\|_F + (\alpha+\beta)\Omega_b\sqrt{M} + (\sigma_2+\beta+\alpha\Omega_B)\|\Delta W_t\|_F \nonumber\\
		&\stackrel{(i)}{\le} [\sigma_2+\beta(\Omega_A+1)]\big(\|\Delta \Theta_{t}\|_F+\|\Delta W_{t}\|_F\big) +2\beta\Omega_b\sqrt{M} \nonumber\\
		&\stackrel{(ii)}{\le} \Big(\frac{1+\sigma_2}{2}\Big)\big(\|\Delta \Theta_{t}\|_F+\|\Delta W_{t}\|_F\big) +2\beta\Omega_b\sqrt{M}, \nonumber
	\end{align}
	where (i) uses the condition that $\alpha\le \min\Big(\frac{\Omega_A \beta}{\Omega_B},\frac{\beta}{\Omega_A},\beta\Big)$, and (ii) uses the condition that $\beta\le\frac{1-\sigma_2}{2(\Omega_A+1)}$. Iterating the above inequality yields that
	\begin{align}
		\|\Delta \Theta_{T}\|_F&\le \|\Delta \Theta_{T}\|_F+\|\Delta W_{T}\|_F \nonumber\\
		&\le \Big(\frac{1+\sigma_2}{2}\Big)^T (\|\Delta \Theta_{0}\|_F+\|\Delta W_{0}\|_F) + \frac{4\beta\Omega_b\sqrt{M}}{1-\sigma_2} \nonumber\\
		&\le \|\Delta \Theta_{0}\|_F+\|\Delta W_{0}\|_F + \frac{4\beta\Omega_b\sqrt{M}}{1-\sigma_2} \stackrel{\triangle}{=} c_{12}. \label{pre_consensus_decay}
	\end{align}
	
	Substituting eq. \eqref{pre_consensus_decay} into eq. \eqref{post_consensus_decay} yields that $\|\Delta\Theta_{T+T'}\|_F \le \sigma_2^{T'}c_{12}$. 
	Then, eq. \eqref{consensus_err1} is proved as follows. 
	\begin{align}
	\mathbb{E}(\|\theta_{T+T'}^{(m)}-\overline{\theta}_{T}\|^2) &\stackrel{(i)}{=} \mathbb{E}(\|\theta_{T+T'}^{(m)}-\overline{\theta}_{T+T'}\|^2) \le \mathbb{E}(\|\Delta\Theta_{T+T'}\|_F^2) \le \sigma_2^{2T'}c_{12}^2, \nonumber
	\end{align}
	where (i) uses the fact that the model average does not change, i.e., $\overline{\theta}_{t+1}=\overline{\theta}_t$ for all $T\le t\le T+T'-1$. 
	
	To summarize, the following conditions of the hyperparameters have been used in the proof of Theorem \ref{thm_sync}. 
	\begin{align}
	\alpha&\le \min\Big[\frac{\lambda_1}{40\Omega_{A}^2\big(1+\Omega_{B}^2\|C^{-1}\|_F^2\big)}, \frac{1}{\lambda_{1}}, \frac{\beta\lambda_2}{4\lambda_1}, \frac{\beta \lambda_2}{16 c_6}, \frac{\beta\lambda_1}{8c_3}, \frac{\beta c_4}{c_7}, \frac{\beta}{3}\sqrt{\frac{\lambda_2}{c_1}}, \beta\sqrt{\frac{c_4}{c_5}},\frac{\beta\Omega_A}{\Omega_B},\frac{\beta}{\Omega_A},\beta\Big]\nonumber\\
	&\red{=\min \{\mathcal{O} (1), \mathcal{O} (\beta)\}},\label{alpha_cond1}\\ \beta&\le\min\Big(\frac{\lambda_2}{16},\frac{1}{\lambda_2},\frac{1-\sigma_2}{2(\Omega_A+1)}\Big)\red{=\mathcal{O} (1)},\label{beta_cond1}\\
	N&\ge \max\Big[\frac{8c_{\text{sd}}}{\lambda_2},\frac{8c_{\text{var,3}}\Omega_{A}^2}{\lambda_{1}^2}, \frac{8\beta c_2}{\alpha \lambda_1}\Big]\red{=\max \{\mathcal{O} (1), \mathcal{O} (\beta/\alpha)\}}.\label{N_cond1}
	\end{align} 
\end{proof}

\section{Proof of Proposition \ref{prop_sync}} \label{appendix_prop_sync}

\propsync*

\begin{proof}
	We choose the following hyperparameter values.
	\begin{align}
		N&=\left\lceil \max\Big[\frac{8c_{\text{sd}}}{\lambda_2},\frac{8c_{\text{var,3}}\Omega_{A}^2}{\lambda_{1}^2}, \frac{8\beta c_2}{\alpha \lambda_1},\frac{96\beta c_4}{\alpha \lambda_1\epsilon}\Big] \right\rceil =\mathcal{O}(\epsilon^{-1}) \ge \frac{96\beta c_4}{\alpha \lambda_1\epsilon}, \label{N1}\\
		T&=\left\lceil \frac{\ln[8\epsilon^{-1}(\|\overline{\theta}_{0}-\theta^*\|^2+\|\overline{w}_{0}-w_{0}^*\|^2)]}{\ln(1+\alpha\lambda_{1}/4)} \right\rceil=\mathcal{O}[\ln(\epsilon^{-1})]. \label{T1}
	\end{align}  
	Then, eq. \eqref{sum_err} implies that
	\begin{align}
	&\mathbb{E}(\|\overline{\theta}_{T}-\theta^*\|^2)\nonumber\\
	&\le \frac{12\beta c_4}{\alpha N\lambda_1} + \Big(1-\frac{\alpha\lambda_{1}}{4}\Big)^{T}\big(\|\overline{\theta}_{0}-\theta^*\|^2+\|\overline{w}_{0}-w_{0}^*\|^2\big)\nonumber\\ 
	&\le \frac{\epsilon}{8}+\exp\Big[\ln\big(\|\overline{\theta}_{0}-\theta^*\|^2+\|\overline{w}_{0}-w_{0}^*\|^2\big)+\frac{\ln[8\epsilon^{-1}(\|\overline{\theta}_{0}-\theta^*\|^2+\|\overline{w}_{0}-w_{0}^*\|^2)]}{\ln(1+\alpha\lambda_{1}/4)} \ln\Big(1-\frac{\alpha\lambda_{1}}{4}\Big)\Big] \nonumber\\
	&\stackrel{(i)}{\le}  \frac{\epsilon}{8}+\exp\Big[\ln\big(\|\overline{\theta}_{0}-\theta^*\|^2+\|\overline{w}_{0}-w_{0}^*\|^2\big)+\frac{\ln[8\epsilon^{-1}(\|\overline{\theta}_{0}-\theta^*\|^2+\|\overline{w}_{0}-w_{0}^*\|^2)]}{-\ln(1-\alpha\lambda_{1}/4)} \ln\Big(1-\frac{\alpha\lambda_{1}}{4}\Big)\Big] \nonumber\\
	&= \frac{\epsilon}{8}+\exp\big[-\ln(8\epsilon^{-1})\big]=\frac{\epsilon}{4}, \nonumber
	\end{align} 
	where (i) uses the inequalities that $1+\alpha\lambda_1/4\le(1-\alpha\lambda_1/4)^{-1}$ and that $\ln(1-\alpha\lambda_1/4)<0$. 
	
	Furthermore, we choose $T'=\left\lceil \frac{\ln(4c_{12}^2/\epsilon)}{2\ln(1/\sigma_2)} \right\rceil=\mathcal{O}\big[\ln(\epsilon^{-1})\big]$, under which eq. \eqref{consensus_err1} implies that
	\begin{align}
		\mathbb{E}(\|\theta_{T+T'}^{(m)}-\overline{\theta}_{T}\|^2) &\le \sigma_2^{2T'}c_{12}^2 = c_{12}^2\exp[2T'\ln(\sigma_2)] \le c_{12}^2\exp\Big[2\ln(\sigma_2)\frac{\ln(4c_{12}^2/\epsilon)}{2\ln(1/\sigma_2)}\Big]=\frac{\epsilon}{4}.
	\end{align}
	Hence, $\mathbb{E}(\|\theta_{T+T'}^{(m)}-\theta^*\|^2) \le 2\mathbb{E}(\|\theta_{T+T'}^{(m)}-\overline{\theta}_{T}\|^2) + 2\mathbb{E}(\|\overline{\theta}_{T}-\theta^*\|^2)\le \epsilon$. 
\end{proof}
\section{Proof of Theorem \ref{thm_local}}\label{appendix_thm_local}
\thmlocal*

\begin{proof}
	We first bound the tracking error. Note that
	\begin{align}
		&\mathbb{E}\big[\|\overline{w}_{t+1}-w_t^*\|^2\big|\mathcal{F}_t\big]\nonumber\\
		&\red{\stackrel{(i)}{=}} \mathbb{E}\Big[\Big\|\overline{w}_{t}+\frac{\beta}{M}\sum_{m=1}^{M} (\overline{A}_{t}^{(m)}\theta_{t}^{(m)}+\overline{\widetilde{b}}_{t}^{(m)}+\overline{C}_{t}w_{t}^{(m)})-w_t^*\Big\|^2 \Big|\mathcal{F}_t\Big] \nonumber\\
		&=\mathbb{E}\Big[\Big\|\overline{w}_{t} +\beta(\overline{A}_{t}\overline{\theta}_{t}+\overline{\overline{b}}_{t}+\overline{C}_{t}\overline{w}_{t})-w_t^* +\frac{\beta}{M}\sum_{m=1}^{M} \big[(\overline{A}_{t}^{(m)}-\overline{A}_{t})\theta_{t}^{(m)}+\overline{\widetilde{b}}_{t}^{(m)}-\overline{\overline{b}}_t\big]\Big\|^2 \Big|\mathcal{F}_t\Big] \nonumber\\
		&\stackrel{(ii)}{\le} \Big(1+\frac{1}{6/(\beta\lambda_2)-3}\Big) \mathbb{E}\big[\|\overline{w}_{t}+\beta(\overline{A}_{t}\overline{\theta}_{t}+\overline{\overline{b}}_{t}+\overline{C}_{t}\overline{w}_{t})-w_t^*\|^2 \big|\mathcal{F}_t\big] \nonumber\\
		&\quad+ \Big(1+\frac{6}{\beta\lambda_2}-3\Big) \mathbb{E}\Big[\Big\|\frac{\beta}{M}\sum_{m=1}^{M} \big[(\overline{A}_{t}^{(m)}-\overline{A}_{t})\theta_{t}^{(m)}+\overline{\widetilde{b}}_{t}^{(m)}-\overline{\overline{b}}_t\big]\Big\|^2 \Big|\mathcal{F}_t\Big] \nonumber\\
		&\stackrel{(iii)}{\le} \Big(1+\frac{1}{6/(\beta\lambda_2)-3}\Big) \Big[\Big(1-\frac{\beta\lambda_{2}}{2}\Big) \|\overline{w}_{t}-w_t^*\|^2 + \frac{7\beta c_{\text{var,2}}\Omega_{A}^2}{N\lambda_{2}} \big\|\overline{\theta}_{t}-\theta^*\big\|^2 + \frac{7\beta c_{\text{var,2}}}{N\lambda_{2}} \big(\Omega_{A}^2\big\|\theta^*\big\|^2 + \Omega_{b}^2\big)\Big] \nonumber\\
		&\quad+ \beta^2\Big(\frac{6}{\beta\lambda_2}-2\Big) \mathbb{E}\Big[\frac{1}{M}\sum_{m=1}^{M} \Big(\big\|(\overline{A}_{t}^{(m)}-\overline{A}_{t}) \theta_{t}^{(m)}+\overline{\widetilde{b}}_{t}^{(m)}-\overline{\overline{b}}_t\big\|^2\Big) \Big|\mathcal{F}_t\Big]\nonumber\\
		&\stackrel{(iv)}{\le} \Big(1-\frac{\beta\lambda_{2}}{3}\Big) \|\overline{w}_{t}-w_t^*\|^2 + \frac{14\beta c_{\text{var,2}}\Omega_{A}^2}{N\lambda_{2}} \big\|\overline{\theta}_{t}-\theta^*\big\|^2 + \frac{14\beta c_{\text{var,2}}}{N\lambda_{2}} \big(\Omega_{A}^2\big\|\theta^*\big\|^2 + \Omega_{b}^2\big) \nonumber\\
		&\quad+ 4\beta\Big(\frac{3}{\lambda_2}-\beta\Big) \mathbb{E}\Big[\frac{1}{M}\sum_{m=1}^{M} \Big(\big\|\overline{A}_{t}^{(m)}-\overline{A}_{t}\big\|_F^2 \big\|\theta_{t}^{(m)}\big\|^2+\big\|\overline{\widetilde{b}}_{t}^{(m)}-\overline{\overline{b}}_t\big\|^2\Big) \Big|\mathcal{F}_t\Big]\nonumber\\
		&\stackrel{(v)}{\le} \Big(1-\frac{\beta\lambda_{2}}{3}\Big) \|\overline{w}_{t}-w_t^*\|^2 + \frac{14\beta c_{\text{var,2}}\Omega_{A}^2}{N\lambda_{2}} \big\|\overline{\theta}_{t}-\theta^*\big\|^2 + \frac{14\beta c_{\text{var,2}}}{N\lambda_{2}} \big(\Omega_{A}^2\big\|\theta^*\big\|^2 + \Omega_{b}^2\big) \nonumber\\
		&\quad+ \frac{12\beta}{\lambda_2} M^2 \sigma_2^{2L} \Big( D_A\big[1+\beta(\red{\widetilde{\Omega}_{A}}+1)\big]^{2t}c_{\text{para}}^2+D_b \Big),\label{wdelay_err2}
	\end{align}
	where \red{(i) uses eq. \eqref{w_iter_avg2},} (ii) uses the inequality that $\|a_1+a_2\|^2\le (1+\sigma)\|a_1\|^2+(1+\sigma^{-1})\|a_2\|^2$ for any $a_1,a_2\in\mathbb{R}^d$ and $\sigma>0$, (iii) applies Jensen's inequality to the convex function $\|\cdot\|^2$ and uses eq. \eqref{wdelay_err} which holds under the conditions that $N\ge\frac{8c_{\text{sd}}}{\lambda_2}$, and $\beta\le\min\Big(\frac{\lambda_2}{16}, \frac{1}{\lambda_{2}}\Big)$, (iv) uses the condition that $\beta\le\frac{1}{\lambda_2}$ which implies that $1+\frac{1}{6/(\beta\lambda_2)-3}\le 2$, as well as the inequality that $\|a_1+a_2\|^2\le 2\|a_1\|^2+2\|a_2\|^2$ for any $a_1,a_2\in\mathbb{R}^d$, and (v) uses eqs. \eqref{Adev2}, \eqref{bdev2}\&\eqref{para_bound2}, \red{which hold under the conditions that $L\ge \frac{3\ln M}{2\ln(\sigma_2^{-1})}$, $\alpha\le \beta\min\Big(\frac{\widetilde{\Omega}_{A}}{\widetilde{\Omega}_{B}},\frac{1}{\widetilde{\Omega}_{A}},1\Big)$}. 
	Then, we obtain that
	\begin{align}
	&\mathbb{E}\big[\|\overline{w}_{t+1}-w_{t+1}^*\|^2\big|\mathcal{F}_t\big]\nonumber\\
	&\stackrel{(i)}{\le} \Big(1+\frac{1}{2[3/(\beta\lambda_{2})-1]}\Big) \mathbb{E}\big[\|\overline{w}_{t+1}-w_t^*\|^2\big|\mathcal{F}_t\big] + \big[1+2\big(3/(\beta\lambda_{2})-1\big)\big] \mathbb{E}\big[\|w_{t+1}^*-w_t^*\|^2\big|\mathcal{F}_t\big] \nonumber\\
	&\stackrel{(ii)}{\le} \frac{6/(\beta\lambda_{2})-1}{2[3/(\beta\lambda_{2})-1]} \Big[\Big(1-\frac{\beta\lambda_{2}}{3}\Big) \|\overline{w}_{t}-w_t^*\|^2 + \frac{14\beta c_{\text{var,2}}\Omega_{A}^2}{N\lambda_{2}} \big\|\overline{\theta}_{t}-\theta^*\big\|^2 + \frac{14\beta c_{\text{var,2}}}{N\lambda_{2}} \big(\Omega_{A}^2\big\|\theta^*\big\|^2 + \Omega_{b}^2\big) \nonumber\\
	&\quad+ \frac{12\beta}{\lambda_2} M^2 \sigma_2^{2L} \Big( D_A\big[1+\beta(\red{\widetilde{\Omega}_A}+1)\big]^{2t}c_{\text{para}}^2+D_b \Big) \Big] +\frac{6}{\beta\lambda_2} \mathbb{E}\big[\|C^{-1}A(\overline{\theta}_{t+1}-\overline{\theta}_{t})\|^2\big|\mathcal{F}_t\big] \nonumber\\
	&\stackrel{(iii)}{\le} \Big(1-\frac{\beta\lambda_{2}}{6}\Big)\|\overline{w}_{t}-w_t^*\|^2 + \frac{28\beta c_{\text{var,2}}\Omega_{A}^2}{N\lambda_{2}} \big\|\overline{\theta}_{t}-\theta^*\big\|^2 + \frac{28\beta c_{\text{var,2}}}{N\lambda_{2}} \big(\Omega_{A}^2\big\|\theta^*\big\|^2 + \Omega_{b}^2\big) \nonumber\\
	&\quad+ \frac{24\beta}{\lambda_2} M^2\sigma_2^{2L} \Big( D_A\big[1+\beta(\widetilde{\Omega}_A+1)\big]^{2t}c_{\text{para}}^2+D_b \Big) \nonumber\\
	&\quad +\frac{6}{\beta\lambda_2}\Omega_{A}^2\|C^{-1}\|_F^2 \mathbb{E}\Big[ \Big\| \frac{\alpha}{M}\sum_{m=1}^{M} \big(\overline{A}_{t}^{(m)} \theta_{t}^{(m)}+\overline{\widetilde{b}}_{t}^{(m)}+\overline{B}_{t}^{(m)} w_{t}^{(m)}\big) \Big\|^2 \Big|\mathcal{F}_t\Big]  \nonumber\\
	&\stackrel{(iv)}{\le} \Big(1-\frac{\beta\lambda_{2}}{6}\Big)\|\overline{w}_{t}-w_t^*\|^2 + \frac{28\beta c_{\text{var,2}}\Omega_{A}^2}{N\lambda_{2}} \big\|\overline{\theta}_{t}-\theta^*\big\|^2 + \frac{28\beta c_{\text{var,2}}}{N\lambda_{2}} \big(\Omega_{A}^2\big\|\theta^*\big\|^2 + \Omega_{b}^2\big) \nonumber\\
	&\quad+ \frac{24\beta}{\lambda_2} M^2\sigma_2^{2L} \Big( D_A\big[1+\beta(\widetilde{\Omega}_A+1)\big]^{2t}c_{\text{para}}^2+D_b \Big) \nonumber\\
	&\quad +\frac{12\alpha^2}{\beta\lambda_2} \Omega_{A}^2\|C^{-1}\|_F^2 \mathbb{E}\Big[ \Big\| \frac{1}{M}\sum_{m=1}^{M} \big[(\overline{A}_{t}^{(m)}-\overline{A}_t)\theta_{t}^{(m)}+(\overline{\widetilde{b}}_{t}^{(m)}-\overline{b}_{t}^{(m)})+(\overline{B}_{t}^{(m)}-\overline{B}_t)w_{t}^{(m)}\big] \Big\|^2 \Big|\mathcal{F}_t\Big] \nonumber\\
	&\quad +\frac{12\alpha^2}{\beta\lambda_2} \Omega_{A}^2\|C^{-1}\|_F^2 \mathbb{E}\big[ \big\| \overline{A}_t\overline{\theta}_{t}+\overline{\overline{b}}_{t}+\overline{B}_t \overline{w}_{t} \big\|^2 \big|\mathcal{F}_t\big] \nonumber\\
	&\stackrel{(v)}{\le} \Big(1-\frac{\beta\lambda_{2}}{6}\Big)\|\overline{w}_{t}-w_t^*\|^2 + \frac{28\beta c_{\text{var,2}}\Omega_{A}^2}{N\lambda_{2}} \big\|\overline{\theta}_{t}-\theta^*\big\|^2 + \frac{28\beta c_{\text{var,2}}}{N\lambda_{2}} \big(\Omega_{A}^2\big\|\theta^*\big\|^2 + \Omega_{b}^2\big) \nonumber\\
	&\quad+ \frac{24\beta}{\lambda_2} M^2\sigma_2^{2L} \Big( D_A\big[1+\beta(\widetilde{\Omega}_A+1)\big]^{2t}c_{\text{para}}^2+D_b \Big) \nonumber\\
	&\quad +\frac{36\alpha^2}{\beta\lambda_2} \Omega_{A}^2\|C^{-1}\|_F^2 \mathbb{E}\Big[ \frac{1}{M}\sum_{m=1}^{M} \big[\|\overline{A}_{t}^{(m)}-\overline{A}_t\|_F^2 \|\theta_{t}^{(m)}\|^2+\big\|\overline{\widetilde{b}}_{t}^{(m)}-\overline{b}_{t}^{(m)}\big\|^2+\|\overline{B}_{t}^{(m)}-\overline{B}_t\|_F^2 \|w_{t}^{(m)}\|^2\big] \Big|\mathcal{F}_t\Big] \nonumber\\
	&\quad +\frac{12\alpha^2}{\beta\lambda_2} \Omega_{A}^2\|C^{-1}\|_F^2 \Big[10\Omega_{A}^2\big(1+\Omega_{B}^2\|C^{-1}\|_F^2\big)\big\|\overline{\theta}_{t}-\theta^*\big\|^2 + 5\Omega_{B}^2\big\|\overline{w}_{t}-w_t^*\big\|^2 \nonumber\\
	&\quad + \frac{5}{N} 
	\big[c_{\text{var,3}}\Omega_{A}^2\big\|\theta^*\big\|^2+c_{\text{var}}\Omega_{b}^2\big(1+\Omega_{B}^2\|C^{-1}\|_F^2\big)\big]\Big] \nonumber\\
	&\stackrel{(vi)}{\le} \Big(1-\frac{\beta\lambda_{2}}{6}+\frac{3c_1\alpha^2}{\beta}\Big)\|\overline{w}_{t}-w_t^*\|^2 + \Big(\frac{\beta c_8}{N} + \frac{3c_3\alpha^2}{\beta}\Big) \big\|\overline{\theta}_{t}-\theta^*\big\|^2 +\frac{\beta c_9}{N} \nonumber\\ 
	&\quad+ \frac{24\beta}{\lambda_2} M^2\sigma_2^{2L} \Big( D_A\big[1+\beta(\widetilde{\Omega}_A+1)\big]^{2t}c_{\text{para}}^2+D_b \Big) \nonumber\\
	&\quad +\frac{36\alpha^2}{\beta\lambda_2} \Omega_{A}^2\|C^{-1}\|_F^2 M^2\sigma_2^{2L} \Big( (D_A+D_B)\big[1+\beta(\widetilde{\Omega}_A+1)\big]^{2t}c_{\text{para}}^2+D_b \Big) + \frac{3\alpha^2 c_5}{N\beta}\nonumber\\
	&\le \Big(1-\frac{\beta\lambda_{2}}{6}+\frac{3c_1\alpha^2}{\beta}\Big)\|\overline{w}_{t}-w_t^*\|^2 + \Big(\frac{\beta c_8}{N} + \frac{3c_3\alpha^2}{\beta}\Big) \big\|\overline{\theta}_{t}-\theta^*\big\|^2 +\frac{\beta c_9}{N} \nonumber\\ 
	&\quad +\Big(\frac{36\alpha^2}{\beta\lambda_2} \Omega_{A}^2\|C^{-1}\|_F^2+ \frac{24\beta}{\lambda_2}\Big) M^2\sigma_2^{2L} \Big( (D_A+D_B)\big[1+\beta(\widetilde{\Omega}_A+1)\big]^{2t}c_{\text{para}}^2+D_b \Big) + \frac{3\alpha^2 c_5}{N\beta}, \label{w_err_iter2}
	\end{align}
	where (i) uses the inequality that $\|a_1+a_2\|^2\le (1+\sigma)\|a_1\|^2+(1+\sigma^{-1})\|a_2\|^2$ for any $a_1,a_2\in\mathbb{R}^d$ and $\sigma>0$, (ii) uses eq. \eqref{wdelay_err2} as well as the notation that $w_t^*=-C^{-1}(A\overline{\theta}_t+b)$, (iii) uses $\beta\le\frac{1}{\lambda_{2}}$ (this implies that $\frac{6/(\beta\lambda_{2})-1}{2[3/(\beta\lambda_{2})-1]}\le 1+\frac{1}{4}<2$) and eqs. \eqref{theta_iter_avg2} \& \eqref{bound_A}, (iv) uses the notation that $\overline{\overline{b}}_t=\frac{1}{M}\sum_{m=1}^{M} \overline{b}_{t}^{(m)}$ as well as the inequality that $\|a_1+a_2\|^2\le 2(\|a_1\|^2+\|a_2\|^2)$ for any $a_1,a_2\in\mathbb{R}^d$, (v) uses eq. \eqref{SGvar_theta} as well as the inequality that $\|a_1+a_2+a_3\|^2\le 3(\|a_1\|^2+\|a_2\|^2+\|a_3\|^2)$ for any $a_1,a_2,a_3\in\mathbb{R}^d$, and (vi) uses \red{eqs. \eqref{Adev2}, \eqref{Bdev2}, \eqref{bdev2}\&\eqref{para_bound2}} as well as the notations that $c_1=\frac{20}{\lambda_2}\Omega_{A}^2\Omega_{B}^2\|C^{-1}\|_F^2$, $c_3= \frac{40}{\lambda_2}\Omega_{A}^4\|C^{-1}\|_F^2 \big(1+\Omega_{B}^2\|C^{-1}\|_F^2\big)$, $c_5=\frac{20}{\lambda_2}\Omega_{A}^2\|C^{-1}\|_F^2 \big[c_{\text{var,3}}\Omega_{A}^2\big\|\theta^*\big\|^2+c_{\text{var}}\Omega_{b}^2\big(1+\Omega_{B}^2\|C^{-1}\|_F^2\big)\big]$, $c_8=\frac{28c_{\text{var,2}}\Omega_A^2}{\lambda_2}$, $c_9=\frac{28 c_{\text{var,2}}}{\lambda_{2}} \big(\Omega_{A}^2\big\|\theta^*\big\|^2 + \Omega_{b}^2\big)$. 
	
	On the other hand, the convergence error of the model parameters can be bounded as follows.
	\begin{align}
	&\mathbb{E}\big[\|\overline{\theta}_{t+1}-\theta^*\|^2\big|\mathcal{F}_t\big]\nonumber\\
	&\red{\stackrel{(i)}{=}} \mathbb{E}\Big[\Big\|\overline{\theta}_{t}+\frac{\alpha}{M}\sum_{m=1}^{M} \big(\overline{A}_{t}^{(m)} \theta_{t}^{(m)}+\overline{\widetilde{b}}_{t}^{(m)}+\overline{B}_{t}^{(m)} w_{t}^{(m)}\big)-\theta^*\Big\|^2 \Big|\mathcal{F}_t\Big] \nonumber\\
	&=\mathbb{E}\Big[\Big\|\overline{\theta}_{t} +\alpha(\overline{A}_{t}\overline{\theta}_{t}+\overline{\overline{b}}_{t}+\overline{B}_{t}\overline{w}_{t})-\theta^* +\frac{\alpha}{M}\sum_{m=1}^{M} \big[(\overline{A}_{t}^{(m)}-\overline{A}_{t})\theta_{t}^{(m)}+\overline{\widetilde{b}}_{t}^{(m)}-\overline{\overline{b}}_t+(\overline{B}_{t}^{(m)}-\overline{B}_{t})w_{t}^{(m)}\big]\Big\|^2 \Big|\mathcal{F}_t\Big] \nonumber\\
	&\stackrel{(ii)}{\le} \Big(1+\frac{1}{6/(\alpha\lambda_1)-3}\Big) \mathbb{E}\big[\|\overline{\theta}_{t} +\alpha(\overline{A}_{t}\overline{\theta}_{t}+\overline{\overline{b}}_{t}+\overline{B}_{t}\overline{w}_{t})-\theta^*\|^2 \big|\mathcal{F}_t\big] \nonumber\\
	&\quad+ \Big(1+\frac{6}{\alpha\lambda_1}-3\Big) \mathbb{E}\Big[\Big\|\frac{\alpha}{M}\sum_{m=1}^{M} \big[(\overline{A}_{t}^{(m)}-\overline{A}_{t})\theta_{t}^{(m)}+\overline{\widetilde{b}}_{t}^{(m)}-\overline{\overline{b}}_t+(\overline{B}_{t}^{(m)}-\overline{B}_{t})w_{t}^{(m)}\big]\Big\|^2 \Big|\mathcal{F}_t\Big] \nonumber\\
	&\stackrel{(iii)}{\le} \Big(1+\frac{1}{6/(\alpha\lambda_1)-3}\Big) \Big[\Big(1-\frac{\alpha\lambda_{1}}{2}\Big)\|\overline{\theta}_{t}-\theta^*\|^2 + \alpha c_6 \big\|\overline{w}_{t}-w_t^*\big\|^2 + \frac{\alpha c_7}{N}\Big] \nonumber\\
	&\quad+ \alpha^2\Big(\frac{6}{\alpha\lambda_1}-2\Big) \mathbb{E}\Big[\frac{1}{M}\sum_{m=1}^{M} \Big[\big\|(\overline{A}_{t}^{(m)}-\overline{A}_{t}) \theta_{t}^{(m)}+\overline{\widetilde{b}}_{t}^{(m)}-\overline{\overline{b}}_t+(\overline{B}_{t}^{(m)}-\overline{B}_{t})w_{t}^{(m)}\big\|^2\Big] \Big|\mathcal{F}_t\Big]\nonumber\\
	&\stackrel{(iv)}{\le} \Big(1-\frac{\alpha\lambda_{1}}{3}\Big) \big\|\overline{\theta}_{t}-\theta^*\big\|^2 + 2\alpha c_6 \|\overline{w}_{t}-w_t^*\|^2 + \frac{2\alpha c_{7}}{N} \nonumber\\
	&\quad+ \frac{18\alpha}{\lambda_1} \mathbb{E}\Big[\frac{1}{M}\sum_{m=1}^{M} \Big(\big\|\overline{A}_{t}^{(m)}-\overline{A}_{t}\big\|_F^2 \big\|\theta_{t}^{(m)}\big\|^2+\big\|\overline{\widetilde{b}}_{t}^{(m)}-\overline{\overline{b}}_t\big\|^2 + \big\|\overline{B}_{t}^{(m)}-\overline{B}_{t}\big\|_F^2 \big\|w_{t}^{(m)}\big\|^2\Big) \Big|\mathcal{F}_t\Big]\nonumber\\
	&\stackrel{(v)}{\le} \Big(1-\frac{\alpha\lambda_{1}}{3}\Big) \big\|\overline{\theta}_{t}-\theta^*\big\|^2 + 2\alpha c_6 \|\overline{w}_{t}-w_t^*\|^2 + \frac{2\alpha c_{7}}{N} \nonumber\\
	&\quad+ \frac{18\alpha}{\lambda_1} M^2\sigma_2^{2L} \Big( (D_A+D_B)\big[1+\beta(\widetilde{\Omega}_A+1)\big]^{2t}c_{\text{para}}^2+D_b \Big),\label{theta_err_iter2}
	\end{align}
	where \red{(i) uses eq. \eqref{theta_iter_avg2},} (ii) uses the inequality that $\|a_1+a_2\|^2\le (1+\sigma)\|a_1\|^2+(1+\sigma^{-1})\|a_2\|^2$ for any $a_1,a_2\in\mathbb{R}^d$ and $\sigma>0$, (iii) applies Jensen's inequality to the convex function $\|\cdot\|^2$ and uses eq. \eqref{theta_err_iter}, \red{which holds under the conditions \eqref{alpha_cond1}-\eqref{N_cond1}}, (iv) uses the inequality that $\|a_1+a_2+a_3\|^2\le 3(\|a_1\|^2+\|a_2\|^2+\|a_3\|^2)$ for any $a_1,a_2,a_3\in\mathbb{R}^d$ as well as the condition that $\alpha\le \frac{1}{\lambda_1}$ which implies $1+\frac{1}{6/(\alpha\lambda_1)-3}\le 2$, and (v) uses \red{eqs. \eqref{Adev2}, \eqref{Bdev2}, \eqref{bdev2}\&\eqref{para_bound2}}.
	
	Taking expectation on both sides of eqs. \eqref{w_err_iter2}\&\eqref{theta_err_iter2} and summing up the two inequalities yields that 
	\begin{align}
	&\mathbb{E}(\|\overline{\theta}_{t+1}-\theta^*\|^2) + \mathbb{E}(\|\overline{w}_{t+1}-w_{t+1}^*\|^2) \nonumber\\
	&\le \Big(1-\frac{\alpha\lambda_{1}}{3}\Big) \mathbb{E}\big\|\overline{\theta}_{t}-\theta^*\big\|^2 + 2\alpha c_6 \mathbb{E}\|\overline{w}_{t}-w_t^*\|^2 + \frac{2\alpha c_{7}}{N} \nonumber\\
	&\quad+ \frac{18\alpha}{\lambda_2} M^2\sigma_2^{2L} \Big( (D_A+D_B)\big[1+\beta(\widetilde{\Omega}_A+1)\big]^{2t}c_{\text{para}}^2+D_b \Big) \nonumber\\
	&\quad+ \Big(1-\frac{\beta\lambda_{2}}{6}+\frac{3c_1\alpha^2}{\beta}\Big)\mathbb{E}\|\overline{w}_{t}-w_t^*\|^2 + \Big(\frac{\beta c_8}{N} + \frac{3c_3\alpha^2}{\beta}\Big) \mathbb{E}\big\|\overline{\theta}_{t}-\theta^*\big\|^2 +\frac{\beta c_9}{N} \nonumber\\ 
	&\quad +\Big(\frac{36\alpha^2}{\beta\lambda_2}\Omega_{A}^2\|C^{-1}\|_F^2+ \frac{24\beta}{\lambda_2}\Big) M^2\sigma_2^{2L} \Big( (D_A+D_B)\big[1+\beta(\widetilde{\Omega}_A+1)\big]^{2t}c_{\text{para}}^2+D_b \Big) + \frac{3\alpha^2 c_5}{N\beta}\nonumber\\
	&= \Big(1-\frac{\alpha\lambda_1}{3}+\frac{\beta c_8}{N} + \frac{3c_3\alpha^2}{\beta}\Big) \mathbb{E}\big\|\overline{\theta}_{t}-\theta^*\big\|^2 + \Big(1-\frac{\beta\lambda_{2}}{6}+\frac{3c_1\alpha^2}{\beta} + 2\alpha c_6\Big) \mathbb{E}\|\overline{w}_{t}-w_t^*\|^2 +\frac{2\alpha c_7+\beta c_9}{N} + \frac{3\alpha^2 c_5}{N\beta}\nonumber\\ 
	&\quad +\Big(\frac{36\alpha^2}{\beta\lambda_2}\Omega_{A}^2\|C^{-1}\|_F^2 + \frac{24\beta}{\lambda_2} + \frac{18\alpha}{\lambda_2}\Big) M^2\sigma_2^{2L} \Big( (D_A+D_B)\big[1+\beta(\widetilde{\Omega}_A+1)\big]^{2t}c_{\text{para}}^2+D_b \Big)\nonumber\\
	&\stackrel{(i)}{\le} \Big(1-\frac{\alpha\lambda_1}{6}\Big)\big[\mathbb{E}(\|\overline{\theta}_{t}-\theta^*\|^2) + \mathbb{E}(\|\overline{w}_{t}-w_{t}^*\|^2)\big] + \frac{3\beta c_9}{N}\nonumber\\
	&\quad +\frac{78\beta}{\lambda_2} M^2\sigma_2^{2L} \Big( (D_A+D_B)\big[1+\beta(\widetilde{\Omega}_A+1)\big]^{2t}c_{\text{para}}^2+D_b \Big),\nonumber
	\end{align}
	where (i) uses the conditions that $N\ge \frac{12\beta c_8}{\alpha \lambda_1}$ and that $\alpha\le\min\Big(\frac{\beta\lambda_1}{36c_3},\frac{\beta c_6}{3c_1},\frac{\beta\lambda_2}{18c_6+\lambda_1},\frac{\beta c_9}{2c_7},\beta\sqrt{\frac{c_9}{3c_5}},\frac{\beta}{\Omega_{A}\|C^{-1}\|_F},\beta\Big)$ \red{($\alpha\le \frac{\beta c_6}{3c_1}\Rightarrow \frac{3c_1\alpha^2}{\beta}\le \alpha c_6$, $\alpha\le \frac{\beta\lambda_2}{18c_6+\lambda_1} \Rightarrow -\frac{\beta\lambda_2}{6}\le -\alpha\big(3c_6+\frac{\lambda_1}{6}\big)$) (Explain: This is not easy to see, so I added here)}. Iterating the inequality above yields that
	\begin{align}
		\mathbb{E}(\big\|\overline{\theta}_{T}-\theta^*\big\|^2) &\le \mathbb{E}(\big\|\overline{\theta}_{T}-\theta^*\big\|^2+\big\|\overline{w}_{T}-w^*\big\|^2) \nonumber\\
		&\le \Big(1-\frac{\alpha\lambda_{1}}{6}\Big)^{T} \big(\big\|\overline{\theta}_{0}-\theta^*\big\|^2 +  \|\overline{w}_{0}-w_0^*\|^2\big) \nonumber\\
		&\quad + \sum_{k=0}^{T-1} \Big(1-\frac{\alpha\lambda_{1}}{6}\Big)^{T-1-k} \Big[\frac{3\beta c_9}{N} + \frac{78\beta}{\lambda_2} M^2\sigma_2^{2L} \Big( (D_A+D_B)[1+\beta(\Omega_{A}+1)]^{2k}c_{\text{para}}^2+D_b \Big)\Big] \nonumber\\
		&\stackrel{(i)}{\le} \Big(1-\frac{\alpha\lambda_{1}}{6}\Big)^{T} \big(\big\|\overline{\theta}_{0}-\theta^*\big\|^2 +  \|\overline{w}_{0}-w_0^*\|^2\big) + \frac{18\beta c_9}{\alpha N\lambda_1} \nonumber\\
		&\quad + \beta \sigma_2^{2L} \sum_{k=1}^{T} \Big(1-\frac{\alpha\lambda_{1}}{6}\Big)^{T-k}  \Big(\frac{c_{10}}{2}[1+\beta(\Omega_{A}+1)]^{2k}+\frac{c_{11}}{6} \Big) \nonumber\\
		&\stackrel{(ii)}{\le} \Big(1-\frac{\alpha\lambda_{1}}{6}\Big)^{T} \big(\big\|\overline{\theta}_{0}-\theta^*\big\|^2 +  \|\overline{w}_{0}-w_0^*\|^2\big) + \frac{18\beta c_9}{\alpha N\lambda_1} \nonumber\\
		&\quad + \beta \sigma_2^{2L} \Big[\frac{c_{10}}{2}\Big(1-\frac{\alpha\lambda_{1}}{6}\Big)^{T}\sum_{k=1}^{T} 3^{k} +\frac{c_{11}}{\alpha\lambda_1} \Big] \nonumber\\
		&{\le} \Big(1-\frac{\alpha\lambda_1}{6}\Big)^{T} \big(\big\|\overline{\theta}_{0}-\theta^*\big\|^2 +  \big\|\overline{w}_{0}-w_0^*\big\|^2\big) + \frac{18\beta c_9}{\alpha N\lambda_1} + \beta \sigma_2^{2L}\Big( c_{10} (\red{3}^{T})+\frac{c_{11}}{\alpha\lambda_1} \Big) \nonumber
	\end{align}
	where (i) uses the notations that $c_{10}=\frac{156M^2}{\lambda_2}(D_A+D_B)c_{\text{para}}^2$, $c_{11}=\frac{488M^2D_b}{\lambda_2}$, (ii) uses the conditions that $\alpha\le \frac{1}{\lambda_1}$ and that $\beta\le\frac{1-\sigma_2}{2(\Omega_A+1)}$ which respectively imply that $1-\frac{\alpha\lambda_1}{6}\ge \frac{5}{6}$ and that $1+\beta(\Omega_A+1)\le \frac{3}{2}$. This proves eq. \eqref{sum_err2}. 
	
	To prove eq. \eqref{consensus_err2}, notice that eq. \eqref{post_consensus_decay} still holds as a result of local model averaging, so we only need to obtain an upper bound of $\mathbb{E}\|\Delta\Theta_T\|^2$. Subtracting eq. \eqref{theta_iter_avg2} from \eqref{theta_iter2b} yields that for any $0\le t\le T-1$, 
	\begin{align}
		\theta_{t+1}^{(m)}-\overline{\theta}_{t+1}&=\sum_{m'\in \mathcal{N}_m}V_{m,m'} (\theta_{t}^{(m')}-\overline{\theta}_t)+\frac{M-1}{M}\alpha\big(\overline{A}_{t}^{(m)} \theta_{t}^{(m)}+\overline{\widetilde{b}}_{t}^{(m)}+\overline{B}_{t}^{(m)} w_{t}^{(m)}\big)\nonumber\\
		&\quad -\frac{\alpha}{M}\sum_{m'=1,m'\ne m}^{M} \big(\overline{A}_{t}^{(m')} \theta_{t}^{(m')}+\overline{\widetilde{b}}_{t}^{(m')}+\overline{B}_{t}^{(m')} w_{t}^{(m')}\big). \nonumber
	\end{align}
	\red{This can be rewritten into the following matrix-vector form,}
	\begin{align}
		\red{\Delta\Theta_{t+1}=V\Delta\Theta_{t}+[h_1;h_2;\ldots;h_M]^{\top},}
	\end{align}
	\red{where $h_m\stackrel{\triangle}{=} \frac{M-1}{M}\alpha\big(\overline{A}_{t}^{(m)} \theta_{t}^{(m)}+\overline{\widetilde{b}}_{t}^{(m)}+\overline{B}_{t}^{(m)} w_{t}^{(m)}\big)-\frac{\alpha}{M}\sum_{m'=1,m'\ne m}^{M} \big(\overline{A}_{t}^{(m')} \theta_{t}^{(m')}+\overline{\widetilde{b}}_{t}^{(m')}+\overline{B}_{t}^{(m')} w_{t}^{(m')}\big)$.}
	
	Using the item 3 of Lemma \ref{lemma_V} yields that for any $0\le t\le T-1$, 
	\begin{align}
		\|\Delta \Theta_{t+1}\|_F&\le \sigma_2 \|\Delta \Theta_{t}\|_F + \sqrt{\sum_{m=1}^{M} \|h_m\|^2} \le \sigma_2 \|\Delta \Theta_{t}\|_F + \sum_{m=1}^{M}  \|h_m\|.\label{consensus_iter}
	\end{align}
	Then, using \red{eqs. \eqref{bound_A2}, \eqref{bound_B2} \& \eqref{bound_b2}} yields that
	\begin{align}
		\sum_{m=1}^{M}\|h_m\|&\le \sum_{m=1}^{M} \Big[ \frac{M-1}{M}\alpha\big(\red{\widetilde{\Omega}_A} \|\theta_{t}^{(m)}\|+\red{\widetilde{\Omega}_b}+\red{\widetilde{\Omega}_B}\|w_{t}^{(m)}\|\big) + \frac{\alpha}{M}\sum_{m'=1,m'\ne m}^{M} \big(\widetilde{\Omega}_A \|\theta_{t}^{(m')}\|+\widetilde{\Omega}_b+\widetilde{\Omega}_B \|w_{t}^{(m')}\|\big) \Big] \nonumber\\
		&= \frac{M-1}{M}(2\alpha)\sum_{m=1}^{M} \big(\widetilde{\Omega}_A \|\theta_{t}^{(m)}\|+\widetilde{\Omega}_{b}+\widetilde{\Omega}_B \|w_{t}^{(m)}\|\big) \nonumber\\
		&\stackrel{(i)}{\le} 2M\alpha\widetilde{\Omega}_{b} + 2\alpha\sum_{m=1}^{M} \big(\widetilde{\Omega}_A \|\theta_{t}^{(m)}-\overline{\theta}_t\|+\widetilde{\Omega}_B \|w_{t}^{(m)}-\overline{w}_t\|\big) \nonumber\\
		&\quad + 2M\alpha \big(\widetilde{\Omega}_A \|\overline{\theta}_t-\theta^*\|+\widetilde{\Omega}_B \|\overline{w}_t-w_t^*\|\big) +2M\alpha \big(\widetilde{\Omega}_A \|\theta^*\|+\widetilde{\Omega}_B \|C^{-1}(A\overline{\theta}_t+b)\|\big) \nonumber\\
		&\stackrel{(ii)}{\le} 2M\alpha\big(\widetilde{\Omega}_{b}+\Omega_{b}\widetilde{\Omega}_B\|C^{-1}\|_F+\widetilde{\Omega}_A\|\theta^*\|\big) + 2\alpha\sum_{m=1}^{M} \big(\widetilde{\Omega}_A \|\theta_{t}^{(m)}-\overline{\theta}_t\|+\widetilde{\Omega}_B \|w_{t}^{(m)}-\overline{w}_t\|\big) \nonumber\\
		&\quad +2M\alpha \big(\widetilde{\Omega}_A \|\overline{\theta}_t-\theta^*\|+\widetilde{\Omega}_B \|\overline{w}_t-w_t^*\|\big) + 2M\alpha \Omega_A\widetilde{\Omega}_{B}\|C^{-1}\|_F (\|\overline{\theta}_t-\theta^*\| + \|\theta^*\|) \nonumber\\
		&\red{\le \alpha c_{13}\sum_{m=1}^{M} \big( \|\theta_{t}^{(m)}-\overline{\theta}_t\|+ \|w_{t}^{(m)}-\overline{w}_t\|\big) +\alpha c_{14} \big(\|\overline{\theta}_t-\theta^*\|+\|\overline{w}_t-w_t^*\|\big) + \alpha c_{15}}, \nonumber
	\end{align}
	where (i) uses the notations that  $w_t^*=-C^{-1}(A\overline{\theta}_t+b)$,  \red{(ii) uses eqs. \eqref{bound_A} \& \eqref{bound_b}, and (iii) uses the notations that $c_{13}=2\max(\widetilde{\Omega}_{A},\widetilde{\Omega}_{B},1)$, $c_{14}=2M \big[\widetilde{\Omega}_A +(\widetilde{\Omega}_B+1)(1+\Omega_A\|C^{-1}\|_F)\big]$, $c_{15}=2M\big[\widetilde{\Omega}_{b}+\Omega_{b}(\widetilde{\Omega}_B+1)\|C^{-1}\|_F+\big(\widetilde{\Omega}_A+\Omega_{A}(\widetilde{\Omega}_{B}+1)\|C^{-1}\|_F\big)\|\theta^*\|\big]$.} Hence, we obtain that
	\begin{align}
	\mathbb{E}(\|\Delta \Theta_{t+1}\|_F^2) &\stackrel{(i)}{\le} \Big(1+\frac{\sigma_2^{-2}-1}{2}\Big)\sigma_2^2 	\mathbb{E}\big(\|\Delta \Theta_{t}\|_F^2\big) + \Big(1+\frac{2}{\sigma_2^{-2}-1}\Big) 	\mathbb{E}\left[\left(\sum_{m=1}^{M} \|h_m\|\right)^2\right], \nonumber\\
	&\stackrel{(ii)}{\le} \frac{1+\sigma_2^{2}}{2} \mathbb{E}\big(\|\Delta \Theta_{t}\|_F^2\big) + \frac{\red{3\alpha^2}(1+\sigma_2^2)}{1-\sigma_2^2} \nonumber\\
	&\quad \red{\Big[ 2Mc_{13}^2\sum_{m=1}^{M} \big( \|\theta_{t}^{(m)}-\overline{\theta}_t\|^2+ \|w_{t}^{(m)}-\overline{w}_t\|^2\big) + 2c_{14}^2 \big(\|\overline{\theta}_t-\theta^*\|^2+\|\overline{w}_t-w_t^*\|^2\big) + c_{15}^2\Big],} \label{theta_consensus_decay2}
	\end{align}
	where (i) \red{uses eq. \eqref{consensus_iter}} and the inequality that $(u+v)^2\le (1+\sigma)u^2+(1+\sigma^{-1})v^2$ for any $u,v,\sigma\ge 0$, (ii) uses the inequality that $(\sum_{i=1}^{n} q_i)^2 \le n\sum_{i=1}^{n} q_i^2$ for any $q_i\in\mathbb{R}$ \red{and $n\in \mathbb{N}^+$}. Similarly, it can be obtained from eqs. \eqref{w_iter2b} and \eqref{w_iter_avg2} that 
	\begin{align}
	\mathbb{E}(\|\Delta W_{t+1}\|_F^2) &\le \frac{1+\sigma_2^{2}}{2} \mathbb{E}\big(\|\Delta W_{t}\|_F^2\big) + \frac{\red{3\alpha^2}(1+\sigma_2^2)}{1-\sigma_2^2} \nonumber\\
	&\quad \red{\Big[ 2Mc_{13}^2\sum_{m=1}^{M} \big( \|\theta_{t}^{(m)}-\overline{\theta}_t\|^2+ \|w_{t}^{(m)}-\overline{w}_t\|^2\big) + 2c_{14}^2 \big(\|\overline{\theta}_t-\theta^*\|^2+\|\overline{w}_t-w_t^*\|^2\big) + c_{15}^2\Big]}, \label{w_consensus_decay2}
	\end{align}
	Summing up eqs. \eqref{theta_consensus_decay2}\&\eqref{w_consensus_decay2} yields that
	\begin{align}
		&\mathbb{E}\big(\|\Delta \Theta_{t+1}\|_F^2 + \|\Delta W_{t+1}\|_F^2\big)\nonumber\\ 
		&\stackrel{(i)}{\le} \Big(\frac{1+\sigma_2^{2}}{2} + \alpha^2 \red{c_{16}}\Big) \mathbb{E}\big(\|\Delta \Theta_{t}\|_F^2 + \|\Delta W_{t}\|_F^2\big) \nonumber\\
		&\quad + \alpha^2  \red{c_{17}\Big[\Big(1 \!-\! \frac{\alpha\lambda_1}{6}\Big)^{t} \big(\big\|\overline{\theta}_{0} \!-\! \theta^*\big\|^2 \!+\!  \big\|\overline{w}_{0} \!-\! w_0^*\big\|^2\big) + \frac{18 c_9\beta}{\lambda_1N\alpha } + \beta \sigma_2^{2L}\Big( c_{10} (3^{t})+\frac{c_{11} }{\lambda_1 \alpha} \Big)\Big]} + \alpha^2 \red{c_{18}} \nonumber\\
		&\stackrel{(ii)}{\le} \frac{2+\sigma_2^2}{3} \mathbb{E}\big(\|\Delta \Theta_{t}\|_F^2 + \|\Delta W_{t}\|_F^2\big) + \beta^2 c_{18} \nonumber\\
		&\quad + \beta^2 c_{17}\Big[ \big\|\overline{\theta}_{0}-\theta^*\big\|^2+\big\|\overline{w}_{0}-w_0^*\big\|^2 + \frac{2c_9}{c_8} + \beta \sigma_2^{2L} c_{10} (3^{t})\Big] + \alpha\beta \frac{c_{11}c_{17}}{\lambda_1} \nonumber\\
		&\stackrel{(iii)}{\le} \frac{2+\sigma_2^2}{3} \mathbb{E}\big(\|\Delta \Theta_{t}\|_F^2 + \|\Delta W_{t}\|_F^2\big) + \beta^2 c_{18} \nonumber\\
		&\quad + \beta^2 c_{17}\Big[ \big\|\overline{\theta}_{0}-\theta^*\big\|^2+\big\|\overline{w}_{0}-w_0^*\big\|^2 + \frac{2c_9}{c_8} + \frac{\lambda_2}{16} \sigma_2^{2L} c_{10} (3^{t})\Big]+\beta^2 \frac{c_{11}c_{17}}{\lambda_1} \nonumber\\
		&\stackrel{(iv)}{\le} \frac{2+\sigma_2^2}{3} \mathbb{E}\big(\|\Delta \Theta_{t}\|_F^2 + \|\Delta W_{t}\|_F^2\big) + \beta^2 c_{19} +\beta^2 \sigma_2^{2L} c_{20} (3^{t}), \label{sum_preconsensus_decay}
	\end{align}
	where (i) uses \red{eq. \eqref{sum_err2} and the notations that $c_{16}=\frac{12Mc_{13}^2(1+\sigma_2^2)}{1-\sigma_2^2}$, $c_{17}=\frac{12c_{14}^2(1+\sigma_2^2)}{1-\sigma_2^2}$, $c_{18}=\frac{6c_{15}^2(1+\sigma_2^2)}{1-\sigma_2^2}$,} (ii) uses the conditions that $\alpha\le\min\Big(\sqrt{\frac{1-\sigma_2^2}{6\red{c_{16}}}},\beta\Big)$, $N\ge \frac{12\beta c_8}{\alpha \lambda_1}$, (iii) uses the conditions that $\alpha\le \beta\le\frac{\lambda_2}{16}$, (iv) uses the notations that \red{$c_{19}=c_{18}+c_{17}\Big( \big\|\overline{\theta}_{0}-\theta^*\big\|^2+\big\|\overline{w}_{0}-w_0^*\big\|^2 + \frac{2c_9}{c_8} +\frac{c_{11}}{\lambda_1}\Big)$, $c_{20}=\frac{\lambda_2 c_{10}c_{17}}{16}$}.
	Iterating eq. \eqref{sum_preconsensus_decay} yields that
	\begin{align}
	\mathbb{E}\big(\|\Delta \Theta_{T}\|_F^2\big)&\le\mathbb{E}\big(\|\Delta \Theta_{T}\|_F^2 + \|\Delta W_{T}\|_F^2\big) \nonumber\\
	&\le \Big(\frac{2+\sigma_2^2}{3}\Big)^T \mathbb{E}\big(\|\Delta \Theta_{0}\|_F^2 + \|\Delta W_{0}\|_F^2\big) + \sum_{k=1}^{T} \Big(\frac{2+\sigma_2^2}{3}\Big)^{T-k} \big[\beta^2 c_{19} +\beta^2 \sigma_2^{2L} c_{20} (3^{k-1})\big] \nonumber\\
	&\le \mathbb{E}\big(\|\Delta \Theta_{0}\|_F^2 + \|\Delta W_{0}\|_F^2\big) + \beta^2 \sigma_2^{2L} c_{20} (3^{T}) + \frac{3\beta^2 c_{19}}{1-\sigma_2^2}.  \nonumber
	\end{align}
	Substituting the above inequality into eq. \eqref{post_consensus_decay} yields that
	\begin{align}
	\mathbb{E}(\|\theta_{T+T'}^{(m)}-\overline{\theta}_{T}\|^2)&\stackrel{(i)}{=}\mathbb{E}(\|\theta_{T+T'}^{(m)}-\overline{\theta}_{T+T'}\|^2)\nonumber\\
	&\le\mathbb{E}\big(\|\Delta\Theta_{T+T'}\|_F^2\big)\nonumber\\ &\le\sigma_2^{2T'}\mathbb{E}\big(\|\Delta\Theta_T\|_F^2\big) \nonumber\\
	&\le \sigma_2^{2T'} \Big[\mathbb{E}\big(\|\Delta \Theta_{0}\|_F^2 + \|\Delta W_{0}\|_F^2\big) + \beta^2 \sigma_2^{2L} c_{20} (3^{T}) + \frac{3\beta^2 c_{19}}{1-\sigma_2^2} \Big], \nonumber
	\end{align}
	where (i) uses the fact that $\overline{\theta}_{T+T'}=\overline{\theta}_T$. This proves eq. \eqref{consensus_err2}. 
	
	To summarize, the following conditions of the hyperparameters are used in the proof of Theorem \ref{thm_local}. \red{Since Theorem \ref{thm_sync} and Lemma \ref{lemma_parabound2} are used, all their conditions of hyperparameters are also included.} 
	
	\begin{align}
	\alpha&\le \min\Big[\frac{\lambda_1}{40\Omega_{A}^2\big(1+\Omega_{B}^2\|C^{-1}\|_F^2\big)}, \frac{1}{\lambda_{1}}, \sqrt{\frac{1-\sigma_2^2}{6c_{16}}}, \frac{\beta\lambda_2}{4\lambda_1}, \frac{\beta \lambda_2}{16 c_6}, \frac{\beta\lambda_1}{36c_3}, \frac{\beta c_4}{c_7},\frac{\beta c_6}{3c_1},\frac{\beta\lambda_2}{18c_6+\lambda_1},\nonumber\\
	&\quad\quad\quad\quad\frac{\beta c_9}{2c_7},\frac{\beta\Omega_A}{\Omega_B},\frac{\beta}{\Omega_A},\red{\frac{\beta\widetilde{\Omega}_A}{\widetilde{\Omega}_B},\frac{\beta}{\widetilde{\Omega}_A}},\beta,\beta\sqrt{\frac{c_9}{3c_5}}, \frac{\beta}{3}\sqrt{\frac{\lambda_2}{c_1}}, \beta\sqrt{\frac{c_4}{c_5}},\red{\frac{\beta}{\Omega_{A}\|C^{-1}\|_F}}\Big]\red{=\min \{\mathcal{O} (1), \mathcal{O} (\beta)\}},\label{alpha_cond2}\\
	\beta&\le\min\Big(\frac{\lambda_2}{16},\frac{1}{\lambda_2},\frac{1-\sigma_2}{2(\Omega_A+1)}\Big)\red{=\mathcal{O}(1)},\label{beta_cond2}\\ 
	N&\ge \max\Big(\frac{8c_{\text{sd}}}{\lambda_2},\frac{8c_{\text{var,3}}\Omega_{A}^2}{\lambda_{1}^2}, \frac{8\beta c_2}{\alpha \lambda_1},\frac{12\beta c_8}{\alpha \lambda_1}\Big)\red{=\max \{\mathcal{O} (1), \mathcal{O} (\beta/\alpha)\}}.\label{N_cond2}
	\end{align}

\end{proof}

\section{Proof of Proposition \ref{prop_local}}\label{appendix_prop_local}

\proplocal*

\begin{proof}

We choose the following hyperparameter values.
\begin{align}
N&=\left\lceil \max\Big(\frac{8c_{\text{sd}}}{\lambda_2},\frac{8c_{\text{var,3}}\Omega_{A}^2}{\lambda_{1}^2}, \frac{8\beta c_2}{\alpha \lambda_1},\frac{12\beta c_8}{\alpha \lambda_1}, \frac{288\beta c_9}{\alpha\epsilon\lambda_1}\Big) \right\rceil=\mathcal{O}\Big(\frac{1}{\epsilon}\Big) \ge \frac{288\beta c_9}{\alpha\epsilon\lambda_1}, \label{N2}\\
T&=\left\lceil \frac{\ln[16\epsilon^{-1}(\|\overline{\theta}_{0}-\theta^*\|^2+\|\overline{w}_{0}-w_{0}^*\|^2)]}{\ln(1+\alpha\lambda_{1}/6)} \right\rceil=\mathcal{O}\Big[\ln\Big(\frac{1}{\epsilon}\Big)\Big], \label{T2} \\
L&=\left\lceil \frac{1}{2\ln(1/\sigma_2)} \max\Big[ \ln\Big(\frac{16\beta c_{10}}{\epsilon}\Big) + T\ln 3, \ln\Big(\frac{16\beta c_{11}}{\alpha\epsilon\lambda_1}\Big),T\ln 3+\ln\Big(\frac{c_{20}}{c_{19}}\Big)\Big] \right\rceil =\mathcal{O}\Big[\ln\Big(\frac{1}{\epsilon}\Big)\Big]. \label{L2}
\end{align}
Then, eq. \eqref{sum_err2} implies that
\begin{align}
\mathbb{E}(\big\|\overline{\theta}_{T}-\theta^*\big\|^2) &\le \Big(1-\frac{\alpha\lambda_1}{6}\Big)^{T} \big(\big\|\overline{\theta}_{0}-\theta^*\big\|^2 +  \big\|\overline{w}_{0}-w_0^*\big\|^2\big) + \frac{18 c_9\beta}{\lambda_1N\alpha } + \beta \sigma_2^{2L}\Big( c_{10} (\red{3}^{T})+\frac{c_{11}}{\lambda_1 \alpha} \Big) \nonumber\\
&\le \exp\Big[T\ln\Big(1-\frac{\alpha\lambda_1}{6}\Big)+\ln\big(\big\|\overline{\theta}_{0}-\theta^*\big\|^2 +  \big\|\overline{w}_{0}-w_0^*\big\|^2\big) \Big] + \frac{\epsilon}{16} \nonumber\\
&\quad+ \beta c_{10} (3^{T}) \exp[2L\ln(\sigma_2)] + \frac{\beta c_{11}}{\alpha\lambda_1}\exp[2L\ln(\sigma_2)] \nonumber\\
&\stackrel{(i)}{\le} \exp\Big[-\frac{\ln[16\epsilon^{-1}(\|\overline{\theta}_{0}-\theta^*\|^2+\|\overline{w}_{0}-w_{0}^*\|^2)]}{\ln(1+\alpha\lambda_{1}/6)} \ln\Big(1+\frac{\alpha\lambda_1}{6}\Big)+\ln\big(\big\|\overline{\theta}_{0}-\theta^*\big\|^2 +  \big\|\overline{w}_{0}-w_0^*\big\|^2\big) \Big] \nonumber\\
&\quad + \frac{\epsilon}{16} + \beta c_{10} (3^{T}) \exp\big[-\ln(16\beta c_{10}/\epsilon) - T\ln 3\big] + \frac{\beta c_{11}}{\alpha\lambda_1}\exp\Big[-\ln\Big(\frac{16\beta c_{11}}{\alpha\epsilon\lambda_1}\Big)\Big]= \frac{\epsilon}{4},
\end{align}
where (i) uses the inequality that $1-\frac{\alpha\lambda_1}{6}\le\Big(1+\frac{\alpha\lambda_1}{6}\Big)^{-1}$. Furthermore, we choose
\begin{align}
	T'=\left\lceil \frac{1}{2\ln(1/\sigma_2)}\ln\Big[4\epsilon^{-1}\Big(\mathbb{E}\big(\|\Delta \Theta_{0}\|_F^2 + \|\Delta W_{0}\|_F^2\big) + \frac{4\beta^2 c_{19}}{1-\sigma_2^2} \Big)\Big] \right\rceil=\mathcal{O}\big[\ln(1/\epsilon)\big], \label{num_post2}
\end{align}
and then eq. \eqref{consensus_err2} implies that
\begin{align}
&\mathbb{E}(\|\theta_{T+T'}^{(m)}-\overline{\theta}_{T}\|^2) \nonumber\\
&\le \exp(2T'\ln\sigma_2) \Big[\mathbb{E}\big(\|\Delta \Theta_{0}\|_F^2 + \|\Delta W_{0}\|_F^2\big) + \beta^2 \sigma_2^{2L} c_{20} (3^{T}) + \frac{3\beta^2 c_{19}}{1-\sigma_2^2} \Big] \nonumber\\
&\le \exp \Big(- \ln\Big[4\epsilon^{-1}\Big(\mathbb{E}\big(\|\Delta \Theta_{0}\|_F^2 + \|\Delta W_{0}\|_F^2\big) + \frac{4\beta^2 c_{19}}{1-\sigma_2^2} \Big)\Big]\Big) \nonumber\\
&\quad\Big[\mathbb{E}\big(\|\Delta \Theta_{0}\|_F^2 + \|\Delta W_{0}\|_F^2\big) + \beta^2 c_{20} (3^T) \exp(2L\ln\sigma_2) + \frac{3\beta^2 c_{19}}{1-\sigma_2^2} \Big] \nonumber\\
&\stackrel{(i)}{\le} \frac{\epsilon}{4}\Big[\mathbb{E}\big(\|\Delta \Theta_{0}\|_F^2 + \|\Delta W_{0}\|_F^2\big) + \frac{4\beta^2 c_{19}}{1-\sigma_2^2} \Big]^{-1} \nonumber\\
&\quad\Big[\mathbb{E}\big(\|\Delta \Theta_{0}\|_F^2 + \|\Delta W_{0}\|_F^2\big) + \beta^2 c_{20} (3^T) \exp[-T\ln 3-\ln(c_{20}/c_{19})] + \frac{3\beta^2 c_{19}}{1-\sigma_2^2} \Big] \nonumber\\
&\le \frac{\epsilon}{4}\Big[\mathbb{E}\big(\|\Delta \Theta_{0}\|_F^2 + \|\Delta W_{0}\|_F^2\big) + \frac{4\beta^2 c_{19}}{1-\sigma_2^2} \Big]^{-1} \Big[\mathbb{E}\big(\|\Delta \Theta_{0}\|_F^2 + \|\Delta W_{0}\|_F^2\big) + \beta^2 c_{19} + \frac{3\beta^2 c_{19}}{1-\sigma_2^2} \Big] \le \frac{\epsilon}{4},
\end{align}
where (i) uses the inequality that $L\ge 
\frac{T\ln 3+\ln(c_{20}/c_{19})}{2\ln(1/\sigma_2)}$ based on eq. \eqref{L2}. Hence, \\
$\mathbb{E}(\|\theta_{T+T'}^{(m)}-\theta^*\|^2) \le 2\mathbb{E}(\|\theta_{T+T'}^{(m)}-\overline{\theta}_{T}\|^2) + 2\mathbb{E}(\|\overline{\theta}_{T}-\theta^*\|^2)\le \epsilon$. 
\end{proof}
\section{Supporting Lemmas}\label{supp: Lemmas}
In this section, we prove some supporting lemmas that are used throughout the analysis of Algorithms \ref{alg: 1} \& \ref{alg: 2}. 
\begin{lemma}\label{lemma_bound_ABCt}
	Regarding the terms defined in \Cref{supp: notation}, they have the following upper bounds.
	\begin{align}
	&\|A_t\|_F, \|\overline{A}_t\|_F, \|A\|_F\le \Omega_A\stackrel{\triangle}{=}\rho_{\max}(1+\gamma), \label{bound_A}\\
	&\|B_t\|_F, \|\overline{B}_t\|_F, \|B\|_F\le \Omega_B\stackrel{\triangle}{=}\rho_{\max}\gamma, \label{bound_B}\\
	&\|C_t\|_F, \|\overline{C}_t\|_F, \|C\|_F\le 1, \label{bound_C}\\
	&\|b_{t}^{(m)}\|, \|\overline{b}_{t}^{(m)}\|, \|\overline{b}_t\|, \|\overline{\overline{b}}_{t}\|, \|b\|\le \Omega_b\stackrel{\triangle}{=}\rho_{\max}R_{\max}. \label{bound_b}
	\end{align}
\end{lemma}
\begin{proof}
	Consider any two vectors $u, v\in\mathbb{R}^d$, we have that $\|uv^{\top}\|_F=\sqrt{\text{tr}(vu^{\top}uv^{\top})}=\|u\|\|v\|$. Therefore, \red{by \Cref{assum_phi_le1}, we obtain that}
	\begin{align}
	\|A_t\|_F\le& \rho_t \|\phi(s_t)\| \|\gamma \phi(s_{t+1})-\phi(s_t)\| \nonumber\\
	\le& \rho_{\max} \big[\gamma \|\phi(s_{t+1})\|+\|\phi(s_t)\|\big]\nonumber\\
	\le& \rho_{\max}(1+\gamma) = \Omega_A, \nonumber
	\end{align}
	\begin{align}
	\|B_t\|_F\le& \gamma \rho_t \|\phi(s_{t+1})\| \|\phi(s_t)\|\le \rho_{\max}\gamma = \Omega_B, \nonumber
	\end{align}
	\begin{align}
	\|C_t\|_F\le& \|\phi(s_t)\|^2 \le 1, \nonumber
	\end{align}
	\begin{align}
	\|b_{t}^{(m)}\|\le& \rho_t R_{t}^{(m)} \|\phi(s_t)\|\le \rho_{\max} R_{\max}=\Omega_b. \nonumber
	\end{align}
	
	The proof for $\|A_{t}^{(m)}\|_F$, $\|\widetilde{b}_{t}^{(m)}\|$, etc. is similar. 
	
	On the other hand, by Jensen's inequality, we obtain that
	\begin{align}
		\|A\|_F=&\|\mathbb{E}_{\pi_b}[A_{t}]\|_F\le \mathbb{E}_{\pi_b}\|A_{t}\|_F\le \Omega_A, \nonumber
	\end{align}
	\begin{align}
		\|\overline{A}_{t}\|_F=&\Big\|\frac{1}{N}\sum_{i=tN}^{(t+1)N-1} A_{i}\Big\|_F \le \frac{1}{N}\sum_{i=tN}^{(t+1)N-1} \|A_{i}\|_F\le\Omega_{A}. \nonumber
	\end{align} 
	The proof for the other remaining matrices is similar by using the Jensen's inequality.
\end{proof}


\begin{lemma}\label{lemma_var}
	Suppose the MDP trajectory $\{s_t,a_t\}_{t\ge 0}$ is generated following a behavioral policy $\pi_b$ where $a_t\stackrel{\triangle}{=}\{a_t^{(m)}\}_m$. For any deterministic mappings $X:\mathcal{S}\times \mathcal{A}_1\times\ldots\times\mathcal{A}_M\times\mathcal{S}\to\mathbb{R}^d$ and $Y:\mathcal{S}\times \mathcal{A}_1\times\ldots\times\mathcal{A}_M\times\mathcal{S}\to\mathbb{R}^{d\times d}$ such that $\|X(s,a,s')\|\le C_x, \|Y(s,a,s')\|_F\le C_y, \forall s,s'\in\mathcal{S}, a^{(m)}\in\mathcal{A}_m$ where $a=\{a^{(m)}\}_{m}$, we have
	\begin{align}
		\Big\|\mathbb{E}\Big[\frac{1}{N}\sum_{i=tN}^{(t+1)N-1} X(s_i,a_i,s_{i+1})\Big|\mathcal{F}_t\Big]-\overline{X}\Big\| \le& \frac{2\nu C_x}{N(1-\delta)},\label{vec_bias}\\
		\Big\|\mathbb{E}\Big[\frac{1}{N}\sum_{i=tN}^{(t+1)N-1} Y(s_i,a_i,s_{i+1})\Big|\mathcal{F}_t\Big]-\overline{Y}\Big\|_F \le& \frac{2\nu C_y}{N(1-\delta)},\label{matrix_bias}\\
		\mathbb{E}\Big[\Big\|\frac{1}{N}\sum_{i=tN}^{(t+1)N-1} X(s_i,a_i,s_{i+1})-\overline{X}\Big\|^2\Big|\mathcal{F}_t\Big]\le& \frac{8C_x^2(\nu+1-\delta)}{N(1-\delta)}, \label{vec_var} \\
		\mathbb{E}\Big[\Big\|\frac{1}{N}\sum_{i=tN}^{(t+1)N-1} Y(s_i,a_i,s_{i+1})-\overline{Y}\Big\|_F^2\Big|\mathcal{F}_t\Big]\le& \frac{8C_y^2(\nu+1-\delta)}{N(1-\delta)}, \label{matrix_var}
	\end{align}
	where $\overline{X}=\mathbb{E} X(s_i,a_i,s_{i+1})$, $\overline{Y}=\mathbb{E} Y(s_i,a_i,s_{i+1})$.
\end{lemma}

\textbf{Note:} A simplified version of the above lemma has been proposed and proved in \cite{xu2020improving}, where $a_i$ and $s_{i+1}$ are omitted in the eqs.  \eqref{vec_var}\&\eqref{matrix_var}. We add $a_i$ and $s_{i+1}$ so that this lemma can be better applied to the quantities $A_i$, $B_i$, $C_i$ and $b_{i}^{(m)}$ which rely on $s_i$ as well as $a_i$ and $s_{i+1}$. The proof logic is similar to that of \cite{xu2020improving}.
\begin{proof}
	Assume $i\ge j$, we obtain that
	\begin{align}
		&\big\|\mathbb{E}  \big[X(s_i,a_i,s_{i+1})\big|s_j\big]-\overline{X}\big\| \nonumber\\
		&=\Big\|\int \mathbb{E}\big[X(s_i,a_i,s_{i+1})\big|s_i=s\big] d\mu_{i|j}(s) - \int \mathbb{E} \big[X(s_i,a_i,s_{i+1})\big|s_i=s\big] d\mu_{\pi_b}(s)\Big\| \nonumber\\
		&\le \int \big\|\mathbb{E}\big[X(s_i,a_i,s_{i+1})\big|s_i=s\big]\big\| d|\mu_{i|j}(s)-\mu_{\pi_b}(s)| \nonumber\\
		&\le 2C_x d_{TV}(\mu_{i|j},\mu_{\pi_b}) \nonumber\\
		&\stackrel{(i)}{\le} 2\nu C_x\delta^{i-j}, \label{Ei_fromj}
	\end{align}
	where $\mu_{i|j}$ is the conditional probability distribution of $s_i$ given $s_j$, and (i) uses Assumption \ref{assum_TV_exp}. We also obtain that
	\begin{align}
		& \Big\|\mathbb{E}\Big[\frac{1}{N}\sum_{i=tN}^{(t+1)N-1} X(s_i,a_i,s_{i+1})\Big|\mathcal{F}_t\Big]-\overline{X}\Big\|\nonumber\\
		&\le \frac{1}{N} \sum_{i=tN}^{(t+1)N-1} \big\|\mathbb{E}  \big[X(s_i,a_i,s_{i+1})\big|\mathcal{F}_t\big]-\overline{X}\big\| \nonumber\\
		&\stackrel{(i)}{=} \frac{1}{N} \sum_{i=tN}^{(t+1)N-1} \big\|\mathbb{E}  \big[X(s_i,a_i,s_{i+1})\big|s_{tN}\big]-\overline{X}\big\| \nonumber\\
		&\stackrel{(ii)}{\le} \frac{2\nu C_x}{N}\sum_{i=tN}^{(t+1)N-1} \delta^{i-tN}\nonumber\\
		&\le \frac{2\nu C_x}{N(1-\delta)}, \nonumber
	\end{align}
	where (i) uses the Markovian property and (ii) uses eq. \eqref{Ei_fromj}. This proves eq. \eqref{vec_bias}. The eq. \eqref{matrix_bias} can be proved in the same way and we omit the proof. Next, we obtain that
	\begin{align}
		&\mathbb{E}\Big[\Big\|\frac{1}{N}\sum_{i=tN}^{(t+1)N-1} X(s_i,a_i,s_{i+1})-\overline{X}\Big\|^2\Big|\mathcal{F}_t\Big]\nonumber\\
		&=\mathbb{E}\Big[ \Big<\frac{1}{N}\sum_{i=tN}^{(t+1)N-1} \big[X(s_i,a_i,s_{i+1})-\overline{X}\big], \frac{1}{N}\sum_{j=tN}^{(t+1)N-1} \big[X(s_j,a_j,s_{j+1})-\overline{X}\big]\Big>\Big|\mathcal{F}_t\Big] \nonumber\\
		&=\frac{1}{N^2} \mathbb{E} \Big[ \sum_{i=tN}^{(t+1)N-1} \big\|X(s_i,a_i,s_{i+1})-\overline{X}\big\|^2 \Big|\mathcal{F}_t \Big] \nonumber\\
		&\quad+ \frac{2}{N^2} \sum_{j=tN}^{(t+1)N-2} \sum_{i=j+1}^{(t+1)N-1} \mathbb{E} \Big[ \big<X(s_i,a_i,s_{i+1})-\overline{X}, X(s_j,a_j,s_{j+1})-\overline{X}\big> \Big|\mathcal{F}_t \Big] \nonumber\\
		&\le \frac{N(2C_x)^2}{N^2} + \frac{2}{N^2} \sum_{j=tN}^{(t+1)N-2} \sum_{i=j+1}^{(t+1)N-1} \mathbb{E} \Big[ \mathbb{E} \big[\big<X(s_i,a_i,s_{i+1})-\overline{X}, X(s_j,a_j,s_{j+1})-\overline{X}\big> \big| \mathcal{F}_t \cup \sigma\big(\{s_j,a_j,s_{j+1}\}\big) \big] \Big|\mathcal{F}_t \Big] \nonumber\\
		&=\frac{4C_x^2}{N} + \frac{2}{N^2} \sum_{j=tN}^{(t+1)N-2} \sum_{i=j+1}^{(t+1)N-1} \mathbb{E} \Big[\big<X(s_j,a_j,s_{j+1})-\overline{X}, \mathbb{E} \big[X(s_i,a_i,s_{i+1})-\overline{X} \big| \mathcal{F}_t \cup \sigma\big(\{s_j,a_j,s_{j+1}\}\big) \big] \big> \Big|\mathcal{F}_t \Big] \nonumber\\
		&\stackrel{(i)}{\le} \frac{4C_x^2}{N} + \frac{2}{N^2} \sum_{j=tN}^{(t+1)N-2} \sum_{i=j+1}^{(t+1)N-1} \mathbb{E} \Big[\big\|X(s_j,a_j,s_{j+1})-\overline{X}\big\| \big\|\mathbb{E} \big[X(s_i,a_i,s_{i+1})\big| s_{j+1} \big]-\overline{X} \big\| \Big|\mathcal{F}_t \Big] \nonumber\\
		&\stackrel{(ii)}{\le} \frac{4C_x^2}{N} + \frac{2}{N^2} \sum_{j=tN}^{(t+1)N-2} \sum_{i=j+1}^{(t+1)N-1} \mathbb{E} \Big[(2C_x)(2\nu C_x \delta^{i-j-1}) \Big|\mathcal{F}_t \Big] \nonumber\\
		&\le \frac{4C_x^2}{N} + \frac{8\nu C_x^2}{N^2} \sum_{j=tN}^{(t+1)N-2} \frac{1}{1-\delta} \nonumber\\
		&\le \frac{8C_x^2(\nu+1-\delta)}{N(1-\delta)},
	\end{align}
	where (i) uses the Markovian property as well as Cauchy-Schwartz inequality, and (ii) uses eq. \eqref{Ei_fromj}. This proves eq. \eqref{vec_var}. The \cref{matrix_var} can be proved in the same way and we omit the proof.
\end{proof}

\begin{coro}\label{coro_Ediff_ABC}
	Regarding the terms defined in \Cref{supp: notation}, they have the following upper bounds.
	\begin{align}
		\mathbb{E}\big[\|\overline{A}_t-A\|_F \big|\mathcal{F}_t\big]\le&\frac{2\nu\Omega_{A}}{N(1-\delta)} \stackrel{\triangle}{=} \frac{c_{\text{sd}}\Omega_{A}}{N} \label{avgAbias}\\
		\mathbb{E}\big[\|\overline{B}_t-B\|_F \big|\mathcal{F}_t\big]\le&\frac{2\nu\Omega_{B}}{N(1-\delta)} \stackrel{\triangle}{=} \frac{c_{\text{sd}}\Omega_{B}}{N} \label{avgBbias}\\
		\mathbb{E}\big[\|\overline{C}_t-C\|_F \big|\mathcal{F}_t\big]\le&\frac{2\nu}{N(1-\delta)} \stackrel{\triangle}{=} \frac{c_{\text{sd}}}{N} \label{avgCbias}\\
		\mathbb{E}\big[\big\|\overline{A}_t-A\big\|_F^2\big|\mathcal{F}_t\big]\le&
		\frac{8\Omega_{A}^2(\nu+1-\delta)}{N(1-\delta)} \stackrel{\triangle}{=}\frac{c_{\text{var}}\Omega_A^2}{N} \label{avgAvar}\\
		\mathbb{E}\big[\big\|\overline{B}_t-B\big\|_F^2\big|\mathcal{F}_t\big]\le& \frac{8\Omega_{B}^2(\nu+1-\delta)}{N(1-\delta)} \stackrel{\triangle}{=} \frac{c_{\text{var}}\Omega_B^2}{N} \label{avgBvar}\\ 
		\mathbb{E}\big[\big\|\overline{C}_t-C\big\|_F^2\big|\mathcal{F}_t\big]\le& \frac{8(\nu+1-\delta)}{N(1-\delta)} \stackrel{\triangle}{=} \frac{c_{\text{var}}}{N} \label{avgCvar}\\
		\mathbb{E}\big[\big\|\overline{\overline{b}}_t-b\big\|^2\big|\mathcal{F}_t\big]\le& \frac{8\Omega_{b}^2(\nu+1-\delta)}{N(1-\delta)} \stackrel{\triangle}{=} \frac{c_{\text{var}}\Omega_b^2}{N} \label{avgbvar}\\
		\mathbb{E}\big[\big\|\overline{A}_t-\overline{C}_tC^{-1}A\big\|_F^2\big|\mathcal{F}_t\big]\le& \frac{8\Omega_{A}^2(1+\|C^{-1}\|_F^2)(\nu+1-\delta)}{N(1-\delta)} \stackrel{\triangle}{=} \frac{c_{\text{var},2}\Omega_{A}^2}{N} \label{avg2Avar}\\
		\mathbb{E}\big[\big\|\overline{\overline{b}}_t -\overline{C}_tC^{-1} b\big\|^2\big|\mathcal{F}_t\big]\le& \frac{8\Omega_{b}^2(1+\|C^{-1}\|_F^2)(\nu+1-\delta)}{N(1-\delta)} \stackrel{\triangle}{=} \frac{c_{\text{var},2}\Omega_{b}^2}{N}
		 \label{avg2bvar}\\
		\mathbb{E}\big[\big\|\overline{A}_t-\overline{B}_tC^{-1}A-A^{\top}C^{-1}A\big\|_F^2\big|\mathcal{F}_t\big]\le& \frac{8\Omega_{A}^2(1+\Omega_{B}^2\|C^{-1}\|_F^2)(\nu+1-\delta)}{N(1-\delta)} \stackrel{\triangle}{=} \frac{c_{\text{var},3}\Omega_{A}^2}{N} \label{avg3Avar}
	\end{align}
\end{coro}
\begin{proof}
	Let $Y(s,a,s')=\rho(s,a) \phi(s) [\gamma \phi(s')-\phi(s)]^{\top}$ in Lemma \ref{lemma_var}. Then it can be checked that 
	\begin{align}
	&Y(s_t, a_t, s_{t+1})=A_t\nonumber\\ &\frac{1}{N}\sum_{i=tN}^{(t+1)N-1} Y(s_i,a_i,s_{i+1})=\overline{A}_t\nonumber\\ &C_y=\Omega_{A} \nonumber\\
	&\overline{Y}=\mathbb{E}_{\pi_b} Y(s_i,a_i,s_{i+1})=A\nonumber
	\end{align}
	Substituting these equations into eqs. \eqref{matrix_bias}\&\eqref{matrix_var} proves eqs. \eqref{avgAbias}\&\eqref{avgAvar} respectively. The eqs. \eqref{avgBbias}, \eqref{avgCbias}, \eqref{avgBvar}, \eqref{avgCvar} \& \eqref{avgbvar} can be proved in a similar way. 
	
	Let $Y(s,a,s')=\rho(s,a) \phi(s) [\gamma \phi(s')-\phi(s)]^{\top}+\gamma \rho(s,a) \phi(s') \phi(s)^{\top} C^{-1}A$. Then, it can be checked that
	\begin{align}
	&Y(s_t, a_t, s_{t+1})=A_t-B_tC^{-1}A,\nonumber\\ &\frac{1}{N}\sum_{i=tN}^{(t+1)N-1} Y(s_i,a_i,s_{i+1})=\overline{A}_t-\overline{B}_tC^{-1}A, \nonumber\\
	&\|Y(s,a,s')\|_F\le \rho_{\max}(\gamma+1)+\gamma\rho_{\max}\|C^{-1}\|_F\|A\|_F \le \Omega_{A}(1+\Omega_{B}\|C^{-1}\|_F)\Rightarrow C_y=\Omega_{A}(1+\Omega_{B}\|C^{-1}\|_F), \nonumber\\ 
	&\overline{Y}=\mathbb{E}_{\pi_b} Y(s_i,a_i,s_{i+1})=A-BC^{-1}A=(C-B)C^{-1}A=A^{\top}C^{-1}A.\nonumber
	\end{align} 
	Substituting these equations into eq. \eqref{matrix_var} proves eq. \eqref{avg3Avar}. The equations \eqref{avg2Avar}\&\eqref{avg2bvar} can be proved in a similar way.
\end{proof}

\begin{coro}\label{coro_SG}
	The following inequalities hold for both Algorithm \ref{alg: 1} and Algorithm \ref{alg: 2}
	\begin{align}
		2\big(\overline{\theta}_{t}-\theta^*\big)^{\top}\mathbb{E}\big[\overline{A}_{t}\overline{\theta}_{t}+\overline{\overline{b}}_{t}+\overline{B}_{t}\overline{w}_{t}\big|\mathcal{F}_t\big]\le& \Big(\frac{2c_{\text{var},3}\Omega_{A}^2}{N}-\lambda_{1}\Big) \big\|\overline{\theta}_{t}-\theta^*\big\|^2 +\frac{4}{\lambda_{1}} \Omega_{B}^2\big\|\overline{w}_{t}-w_t^*\big\|^2 \nonumber\\
		&+ \frac{4}{N\lambda_{1}} \big[c_{\text{var},3}\Omega_{A}^2\big\|\theta^*\big\|^2 + c_{\text{var}}\Omega_{b}^2\big(1+\Omega_{B}^2\|C^{-1}\|_F^2\big)\big] \label{SGbias_theta}
	\end{align}
	\begin{align}
		2\big(\overline{w}_{t}-w_t^*\big)^{\top}\mathbb{E}\big[\overline{A}_{t}\overline{\theta}_{t}+\overline{\overline{b}}_{t}+\overline{C}_{t}\overline{w}_{t}\big|\mathcal{F}_t\big] \le& \Big(\frac{2c_{\text{sd}}}{N}-\lambda_{2}\Big) \big\|\overline{w}_{t}-w_t^*\big\|^2 + \frac{3c_{\text{var,2}}\Omega_{A}^2}{N\lambda_{2}} \big\|\overline{\theta}_{t}-\theta^*\big\|^2 \nonumber\\
		& + \frac{3c_{\text{var,2}}}{N\lambda_{2}} \big(\Omega_{A}^2\big\|\theta^*\big\|^2 + \Omega_{b}^2\big) \label{SGbias_w}
	\end{align}
	\begin{align}
		\mathbb{E}\big[\big\|\overline{A}_{t}\overline{\theta}_{t}+\overline{\overline{b}}_{t}+\overline{B}_{t}\overline{w}_{t}\big\|^2\big|\mathcal{F}_t\big] \le& 10\Omega_{A}^2\big(1+\Omega_{B}^2\|C^{-1}\|_F^2\big)\big\|\overline{\theta}_{t}-\theta^*\big\|^2 + 5\Omega_{B}^2\big\|\overline{w}_{t}-w_t^*\big\|^2 \nonumber\\
		&+ \frac{5}{N} 
		\big[c_{\text{var,3}}\Omega_{A}^2\big\|\theta^*\big\|^2+c_{\text{var}}\Omega_{b}^2\big(1+\Omega_{B}^2\|C^{-1}\|_F^2\big)\big] \label{SGvar_theta}
	\end{align}
	\begin{align}
		\mathbb{E}\big[\big\|\overline{A}_{t}\overline{\theta}_{t}+\overline{\overline{b}}_{t}+\overline{C}_{t}\overline{w}_{t}\big\|^2\big|\mathcal{F}_t\big] \le \frac{4c_{\text{var,2}}\Omega_{A}^2}{N} \|\overline{\theta}_t-\theta^*\|^2 + 4\big\|\overline{w}_{t}-w_t^*\big\|^2  + \frac{4c_{\text{var,2}}}{N} (\Omega_{A}^2\|\theta^*\|^2 + \Omega_{b}^2) \label{SGvar_w}
	\end{align}
\end{coro}

\begin{proof}
	We first prove eq. \eqref{SGbias_theta}. Note that
	\begin{align}
	&2\big(\overline{\theta}_{t}-\theta^*\big)^{\top}\mathbb{E}\big[\overline{A}_{t}\overline{\theta}_{t}+\overline{\overline{b}}_{t}+\overline{B}_{t}\overline{w}_{t}\big|\mathcal{F}_t\big] \nonumber\\
	&\stackrel{(i)}{=} 2\big(\overline{\theta}_{t}-\theta^*\big)^{\top}\mathbb{E}\big[\overline{A}_{t}\overline{\theta}_{t}+\overline{\overline{b}}_{t}+\overline{B}_{t}(\overline{w}_{t}-w_t^*)-\overline{B}_{t}C^{-1}(A\overline{\theta}_{t}+b)\big|\mathcal{F}_t\big] \nonumber\\
	&=2\big(\overline{\theta}_{t}-\theta^*\big)^{\top}\mathbb{E}\big[(\overline{A}_{t}-\overline{B}_{t}C^{-1}A)(\overline{\theta}_{t}-\theta^*) +(\overline{A}_{t}-\overline{B}_{t}C^{-1}A)\theta^*  \nonumber\\ 
	&\quad +\overline{B}_{t}(\overline{w}_{t}-w_t^*) +\overline{\overline{b}}_{t}-b -(\overline{B}_{t}-B)C^{-1}b+(I-BC^{-1})b\big|\mathcal{F}_t\big] \nonumber\\
	&\stackrel{(ii)}{=}2\big(\overline{\theta}_{t}-\theta^*\big)^{\top}\mathbb{E}\big[(\overline{A}_{t}-\overline{B}_{t}C^{-1}A)(\overline{\theta}_{t}-\theta^*) + (\overline{A}_{t}-\overline{B}_{t}C^{-1}A)\theta^* \nonumber\\ 
	&\quad +\overline{B}_{t}(\overline{w}_{t}-w_t^*) +\overline{\overline{b}}_{t}-b -(\overline{B}_{t}-B)C^{-1}b-(C-B)C^{-1}A\theta^*\big|\mathcal{F}_t\big] \nonumber\\
	&\stackrel{(iii)}{=}2\big(\overline{\theta}_{t}-\theta^*\big)^{\top}\mathbb{E}\big[(\overline{A}_{t}-\overline{B}_{t}C^{-1}A-A^{\top}C^{-1}A)(\overline{\theta}_{t}-\theta^*) +A^{\top}C^{-1}A(\overline{\theta}_{t}-\theta^*) \nonumber\\ 
	&\quad + (\overline{A}_{t}-\overline{B}_{t}C^{-1}A-A^{\top}C^{-1}A)\theta^* +\overline{B}_{t}(\overline{w}_{t}-w_t^*) +\overline{\overline{b}}_{t}-b -(\overline{B}_{t}-B)C^{-1}b\big|\mathcal{F}_t\big] \nonumber\\
	&\stackrel{(iv)}{\le}2\big(\overline{\theta}_{t}-\theta^*\big)^{\top}\mathbb{E}\big[\overline{A}_{t}-\overline{B}_{t}C^{-1}A-A^{\top}C^{-1}A\big|\mathcal{F}_t\big](\overline{\theta}_{t}-\theta^*) + 2\big(\overline{\theta}_{t}-\theta^*\big)^{\top}A^{\top}C^{-1}A(\overline{\theta}_{t}-\theta^*) \nonumber\\ 
	&\quad +\lambda_{1}\|\overline{\theta}_{t}-\theta^*\|^2 + \frac{1}{\lambda_{1}} \big\|\mathbb{E}\big[ (\overline{A}_{t}-\overline{B}_{t}C^{-1}A-A^{\top}C^{-1}A)\theta^* +\overline{B}_{t}(\overline{w}_{t}-w_t^*) +\overline{\overline{b}}_{t}-b -(\overline{B}_{t}-B)C^{-1}b\big|\mathcal{F}_t\big]\big\|^2 \nonumber\\
	&\stackrel{(v)}{\le}2\mathbb{E}\big[\big\|\overline{A}_{t}-\overline{B}_{t}C^{-1}A-A^{\top}C^{-1}A\big\| \big|\mathcal{F}_t\big] \big\|\overline{\theta}_{t}-\theta^*\big\|^2 - 2\lambda_{1} \big\|\overline{\theta}_{t}-\theta^*\big\|^2 \nonumber\\ 
	&\quad +\lambda_{1}\|\overline{\theta}_{t}-\theta^*\|^2 + \frac{1}{\lambda_{1}} \mathbb{E}\big[ \big\|(\overline{A}_{t}-\overline{B}_{t}C^{-1}A-A^{\top}C^{-1}A)\theta^* +\overline{B}_{t}(\overline{w}_{t}-w_t^*) +\overline{\overline{b}}_{t}-b -(\overline{B}_{t}-B)C^{-1}b\big\|^2 \big|\mathcal{F}_t\big]  \nonumber\\
	&\stackrel{(vi)}{\le} \frac{2c_{\text{var},3}\Omega_{A}^2}{N} \big\|\overline{\theta}_{t}-\theta^*\big\|^2 - \lambda_{1} \big\|\overline{\theta}_{t}-\theta^*\big\|^2 \nonumber\\ 
	&\quad + \frac{4}{\lambda_{1}} \mathbb{E}\big[ \big\|(\overline{A}_{t}-\overline{B}_{t}C^{-1}A-A^{\top}C^{-1}A)\theta^*\big\|^2 +\big\|\overline{B}_{t}(\overline{w}_{t}-w_t^*)\big\|^2 +\big\|\overline{\overline{b}}_{t}-b\big\|^2 +\big\|(\overline{B}_{t}-B)C^{-1}b\big\|^2 \big|\mathcal{F}_t\big] \nonumber\\
	&\stackrel{(vii)}{\le} \Big(\frac{2c_{\text{var},3}\Omega_{A}^2}{N}-\lambda_{1}\Big) \big\|\overline{\theta}_{t}-\theta^*\big\|^2 + \frac{4}{\lambda_{1}} \mathbb{E}\big[ \big\|\overline{A}_{t}-\overline{B}_{t}C^{-1}A-A^{\top}C^{-1}A\big\|_F^2 \big|\mathcal{F}_t\big]\big\|\theta^*\big\|^2\nonumber\\ 
	&\quad +\frac{4}{\lambda_{1}} \Omega_{B}^2\big\|\overline{w}_{t}-w_t^*\big\|^2 + \frac{4}{\lambda_{1}} \mathbb{E}\big[\big\|\overline{\overline{b}}_{t}-b\big\|^2 \big|\mathcal{F}_t\big] + \frac{4}{\lambda_{1}}\Omega_{b}^2\|C^{-1}\|_F^2 \mathbb{E}\big[\big\|\overline{B}_{t}-B\big\|_F^2 \big|\mathcal{F}_t\big] \nonumber\\
	&\stackrel{(viii)}{\le} \Big(\frac{2c_{\text{var},3}\Omega_{A}^2}{N}-\lambda_{1}\Big) \big\|\overline{\theta}_{t}-\theta^*\big\|^2 + \frac{4}{\lambda_{1}} \frac{c_{\text{var},3}\Omega_{A}^2}{N}\big\|\theta^*\big\|^2 +\frac{4}{\lambda_{1}} \Omega_{B}^2\big\|\overline{w}_{t}-w_t^*\big\|^2 + \frac{4}{\lambda_{1}} \frac{c_{\text{var}}\Omega_{b}^2}{N} + \frac{4}{\lambda_{1}}\Omega_{b}^2\|C^{-1}\|_F^2 \frac{c_{\text{var}}\Omega_{B}^2}{N} \nonumber\\
	&=\Big(\frac{2c_{\text{var},3}\Omega_{A}^2}{N}-\lambda_{1}\Big) \big\|\overline{\theta}_{t}-\theta^*\big\|^2 +\frac{4}{\lambda_{1}} \Omega_{B}^2\big\|\overline{w}_{t}-w_t^*\big\|^2 + \frac{4}{N\lambda_{1}} \big[c_{\text{var},3}\Omega_{A}^2\big\|\theta^*\big\|^2 + c_{\text{var}}\Omega_{b}^2\big(1+\Omega_{B}^2\|C^{-1}\|_F^2\big)\big] \nonumber
	\end{align}
	where (i) uses the notation that $w_t^*=-C^{-1}(A\overline{\theta}_{t}+b)$, (ii) uses the notation that $b=-A\theta^*$, (iii) uses the relation that $C-B=A^{\top}$, (iv) uses the inequality that $2a_1^{\top} a_2\le \sigma^{-1}\|a_1\|^2+\sigma\|a_2\|^2$ for any $a_1,a_2\in\mathbb{R}^d$ and $\sigma>0$, (v) uses the notation that $\lambda_{1}=-\lambda_{\max}(A^{\top}C^{-1}A)$ and applies Jensen's inequality to the convex functions $\|\cdot\|$ and $\|\cdot\|^2$, (vi) uses \eqref{avg3Avar} and the inequality that $\|a_1+a_2+a_3+a_4\|^2\le 4(\|a_1\|^2+\|a_2\|^2+\|a_3\|^2+\|a_4\|^2)$ for any $a_1,a_2,a_3,a_4\in\mathbb{R}^d$, (vii) uses eqs. \eqref{bound_B} \& \eqref{bound_b}, (viii) uses eqs. \eqref{avgBvar}, \eqref{avgbvar} \& \eqref{avg3Avar}. 
	
	Next, we prove eq. \eqref{SGbias_w}. Note that
	\begin{align}
		&2\big(\overline{w}_{t}-w_t^*\big)^{\top}\mathbb{E}\big[\overline{A}_{t}\overline{\theta}_{t}+\overline{\overline{b}}_{t}+\overline{C}_{t}\overline{w}_{t}\big|\mathcal{F}_t\big] \nonumber\\
		&\stackrel{(i)}{=}2(\overline{w}_{t}-w_t^*)^{\top}\mathbb{E}\big[\overline{A}_{t}\overline{\theta}_{t}+\overline{\overline{b}}_{t}+\overline{C}_{t}(\overline{w}_{t}-w_t^*)-\overline{C}_t C^{-1}(A\overline{\theta}_t+b)\big|\mathcal{F}_t\big] \nonumber\\
		&=2(\overline{w}_{t}-w_t^*)^{\top}\mathbb{E}\big[\overline{C}_t\big|\mathcal{F}_t\big](\overline{w}_{t}-w_t^*) + 2(\overline{w}_{t}-w_t^*)^{\top}\mathbb{E}\big[(\overline{A}_{t}-\overline{C}_t C^{-1}A)\overline{\theta}_{t} +\overline{\overline{b}}_{t} -\overline{C}_t C^{-1}b\big|\mathcal{F}_t\big]\nonumber\\
		&\stackrel{(ii)}{\le}2(\overline{w}_{t}-w_t^*)^{\top}\mathbb{E}\big[\overline{C}_t-C\big|\mathcal{F}_t\big](\overline{w}_{t}-w_t^*) + 2(\overline{w}_{t}-w_t^*)^{\top}C(\overline{w}_{t}-w_t^*) \nonumber\\
		&\quad+\lambda_{2}\|\overline{w}_{t}-w_t^*\|^2 + \frac{1}{\lambda_{2}} \big\|\mathbb{E}\big[(\overline{A}_{t}-\overline{C}_t C^{-1}A)(\overline{\theta}_{t}-\theta^*) +(\overline{A}_{t}-\overline{C}_t C^{-1}A)\theta^* +\overline{\overline{b}}_{t} -\overline{C}_t C^{-1}b\big|\mathcal{F}_t\big]\big\|^2 \nonumber\\
		&\stackrel{(iii)}{\le}2\big\|\mathbb{E}\big[\overline{C}_t-C\big|\mathcal{F}_t\big]\big\|_F \big\|\overline{w}_{t}-w_t^*\big\|^2 - 2\lambda_{2}\big\|\overline{w}_{t}-w_t^*\big\|^2 \nonumber\\
		&\quad+\lambda_{2}\|\overline{w}_{t}-w_t^*\|^2 + \frac{1}{\lambda_{2}} \big\|\mathbb{E}\big[(\overline{A}_{t}-\overline{C}_t C^{-1}A)(\overline{\theta}_{t}-\theta^*) +(\overline{A}_{t}-\overline{C}_t C^{-1}A)\theta^* +\overline{\overline{b}}_{t} -\overline{C}_t C^{-1}b\big|\mathcal{F}_t\big]\big\|^2 \nonumber\\
		&\stackrel{(iv)}{\le} 2\mathbb{E}\big[\big\|\overline{C}_t-C\big\|_F \big|\mathcal{F}_t\big] \big\|\overline{w}_{t}-w_t^*\big\|^2 -\lambda_{2}\big\|\overline{w}_{t}-w_t^*\big\|^2 \nonumber\\
		&\quad + \frac{1}{\lambda_{2}} \mathbb{E}\big[\big\|(\overline{A}_{t}-\overline{C}_t C^{-1}A)(\overline{\theta}_{t}-\theta^*) +(\overline{A}_{t}-\overline{C}_t C^{-1}A)\theta^* +\overline{\overline{b}}_{t} -\overline{C}_t C^{-1}b\big\|^2\big|\mathcal{F}_t\big] \nonumber\\
		&\stackrel{(v)}{\le} \Big(\frac{2c_{\text{sd}}}{N}-\lambda_{2}\Big) \big\|\overline{w}_{t}-w_t^*\big\|^2 + \frac{3}{\lambda_{2}} \mathbb{E}\big[\big\|(\overline{A}_{t}-\overline{C}_t C^{-1}A)\big\|_F^2 \big|\mathcal{F}_t\big] \big\|\overline{\theta}_{t}-\theta^*\big\|^2 \nonumber\\
		&\quad + \frac{3}{\lambda_{2}} \mathbb{E}\big[\big\|\overline{A}_{t}-\overline{C}_t C^{-1}A\big\|_F^2\big|\mathcal{F}_t\big]\big\|\theta^*\big\|^2 + \frac{3}{\lambda_{2}} \mathbb{E}\big[\big\|\overline{\overline{b}}_{t} -\overline{C}_t C^{-1}b\big\|^2\big|\mathcal{F}_t\big] \nonumber\\
		&\stackrel{(vi)}{\le}\Big(\frac{2c_{\text{sd}}}{N}-\lambda_{2}\Big) \big\|\overline{w}_{t}-w_t^*\big\|^2 + \frac{3c_{\text{var,2}}\Omega_{A}^2}{N\lambda_{2}} \big\|\overline{\theta}_{t}-\theta^*\big\|^2 + \frac{3c_{\text{var,2}}\Omega_{A}^2}{N\lambda_{2}} \big\|\theta^*\big\|^2 + \frac{3c_{\text{var,2}}\Omega_{b}^2}{N\lambda_{2}} \nonumber\\
		&=\Big(\frac{2c_{\text{sd}}}{N}-\lambda_{2}\Big) \big\|\overline{w}_{t}-w_t^*\big\|^2 + \frac{3c_{\text{var,2}}\Omega_{A}^2}{N\lambda_{2}} \big\|\overline{\theta}_{t}-\theta^*\big\|^2 + \frac{3c_{\text{var,2}}}{N\lambda_{2}} \big(\Omega_{A}^2\big\|\theta^*\big\|^2 + \Omega_{b}^2\big) \nonumber
	\end{align}
	where (i) uses the notation that $w_t^*=-C^{-1}(A\overline{\theta}_{t}+b)$, (ii) uses the inequality that $2a_1^{\top} a_2\le \sigma^{-1}\|a_1\|^2+\sigma\|a_2\|^2$ for any $a_1,a_2\in\mathbb{R}^d$ and $\sigma>0$, (iii) uses the notation that $\lambda_{2}=-\lambda_{\max}(C)$, (iv) applies Jensen's inequality to convex functions $\|\cdot\|$ and $\|\cdot\|^2$, (v) uses eq. \eqref{avgCbias} and the inequality that $\|a_1+a_2+a_3\|^2\le 3(\|a_1\|^2+\|a_2\|^2+\|a_3\|^2)$ for any $a_1,a_2,a_3\in\mathbb{R}^d$ (vi) uses eqs. \eqref{avg2Avar} \& \eqref{avg2bvar}.
	
	Next, we prove eq. \eqref{SGvar_theta}. Note that 
	\begin{align}
	&\mathbb{E}\big[\big\|\overline{A}_{t}\overline{\theta}_{t}+\overline{\overline{b}}_{t}+\overline{B}_{t}\overline{w}_{t}\big\|^2\big|\mathcal{F}_t\big]\nonumber\\ 
	&\stackrel{(i)}{=} \mathbb{E}\big[\big\|\overline{A}_{t}\overline{\theta}_{t}+\overline{\overline{b}}_{t}+\overline{B}_{t}(\overline{w}_{t}-w_t^*)-\overline{B}_{t}C^{-1}(A\overline{\theta}_t+b)\big\|^2\big|\mathcal{F}_t\big] \nonumber\\
	&=
	\mathbb{E}\big[\big\|(\overline{A}_{t}-\overline{B}_{t}C^{-1}A-A^{\top}C^{-1}A)\overline{\theta}_{t} + A^{\top}C^{-1}A\overline{\theta}_{t} + \overline{\overline{b}}_{t}-b \nonumber\\
	&\quad+b +\overline{B}_{t}(\overline{w}_{t}-w_t^*)-(\overline{B}_{t}-B)C^{-1}b -BC^{-1}b\big\|^2\big|\mathcal{F}_t\big] \nonumber\\
	&\stackrel{(ii)}{=}
	\mathbb{E}\big[\big\|(\overline{A}_{t}-\overline{B}_{t}C^{-1}A-A^{\top}C^{-1}A)(\overline{\theta}_{t}-\theta^*) +(\overline{A}_{t}-\overline{B}_{t}C^{-1}A-A^{\top}C^{-1}A)\theta^*\nonumber\\
	&\quad + A^{\top}C^{-1}A\overline{\theta}_{t} + \overline{\overline{b}}_{t}-b  + \overline{B}_{t}(\overline{w}_{t}-w_t^*)-(\overline{B}_{t}-B)C^{-1}b - (C-B)C^{-1}A\theta^*\big\|^2\big|\mathcal{F}_t\big] \nonumber\\
	&\stackrel{(iii)}{=}
	\mathbb{E}\big[\big\|(\overline{A}_{t}-\overline{B}_{t}C^{-1}A)(\overline{\theta}_{t}-\theta^*) +(\overline{A}_{t}-\overline{B}_{t}C^{-1}A-A^{\top}C^{-1}A)\theta^*\nonumber\\
	&\quad  + \overline{\overline{b}}_{t}-b  + \overline{B}_{t}(\overline{w}_{t}-w_t^*)-(\overline{B}_{t}-B)C^{-1}b\big\|^2 \big|\mathcal{F}_t\big] \nonumber\\
	&\stackrel{(iv)}{=} 
	5\mathbb{E}\big[\big\|(\overline{A}_{t}-\overline{B}_{t}C^{-1}A)(\overline{\theta}_{t}-\theta^*)\big\|^2 +\big\|(\overline{A}_{t}-\overline{B}_{t}C^{-1}A-A^{\top}C^{-1}A)\theta^*\big\|^2\nonumber\\
	&\quad  + \big\|\overline{\overline{b}}_{t}-b\big\|^2 + \big\|\overline{B}_{t}(\overline{w}_{t}-w_t^*)\big\|^2 +\big\|(\overline{B}_{t}-B)C^{-1}b\big\|^2 \big|\mathcal{F}_t\big] \nonumber\\
	&\le
	5\mathbb{E}\big[\big\|\overline{A}_{t}-\overline{B}_{t}C^{-1}A\big\|_F^2\big|\mathcal{F}_t\big]\big\|\overline{\theta}_{t}-\theta^*\big\|^2 +5\mathbb{E}\big[\big\|\overline{A}_{t}-\overline{B}_{t}C^{-1}A-A^{\top}C^{-1}A\big\|_F^2\big|\mathcal{F}_t\big] \big\|\theta^*\big\|^2\nonumber\\
	&\quad +5\mathbb{E}\big[ \big\|\overline{\overline{b}}_{t}-b\big\|^2 \big|\mathcal{F}_t\big] + 5\mathbb{E}\big[\big\|\overline{B}_{t}\big\|_F^2\big|\mathcal{F}_t\big] \big\|\overline{w}_{t}-w_t^*\big\|^2 +5\mathbb{E}\big[\big\|\overline{B}_{t}-B\big\|_F^2 \big|\mathcal{F}_t\big] \big\|C^{-1}b\big\|^2 \nonumber\\
	&\stackrel{(v)}{\le} 10\mathbb{E}\big[\big\|\overline{A}_{t}\big\|_F^2+\big\|\overline{B}_{t}C^{-1}A\big\|_F^2\big|\mathcal{F}_t\big]\big\|\overline{\theta}_{t}-\theta^*\big\|^2 + \frac{5c_{\text{var},3}\Omega_{A}^2}{N} \big\|\theta^*\big\|^2\nonumber\\
	&\quad +\frac{5c_{\text{var}}\Omega_{b}^2}{N}+5\Omega_{B}^2\big\|\overline{w}_{t}-w_t^*\big\|^2 + \frac{5c_{\text{var}}\Omega_{B}^2}{N} \big\|C^{-1}\big\|_F^2\big\|b\big\|^2 \nonumber\\
	&\stackrel{(vi)}{\le} 10\Omega_{A}^2\big(1+\Omega_{B}^2\|C^{-1}\|_F^2\big)\big\|\overline{\theta}_{t}-\theta^*\big\|^2 + 5\Omega_{B}^2\big\|\overline{w}_{t}-w_t^*\big\|^2\nonumber\\
	&\quad + \frac{5}{N} 
	\big[c_{\text{var,3}}\Omega_{A}^2\big\|\theta^*\big\|^2+c_{\text{var}}\Omega_{b}^2\big(1+\Omega_{B}^2\|C^{-1}\|_F^2\big)\big] \nonumber\\
	\end{align}
	where (i) uses the notation that $w_t^*=-C^{-1}(A\overline{\theta}_{t}+b)$, (ii) uses the notation that $b=-A\theta^*$, (iii) uses the relation that $C-B=A^{\top}$, (iv) uses the inequality that $\|a_1+a_2+a_3+a_4+a_5\|^2\le 5(\|a_1\|^2+\|a_2\|^2+\|a_3\|^2+\|a_4\|^2+\|a_5\|^2)$ for any $a_1,a_2,a_3,a_4,a_5\in\mathbb{R}^d$, (v) uses eqs. \eqref{bound_B},\eqref{avgBvar},\eqref{avgbvar}\&\eqref{avg3Avar} as well as the inequality that $\|a_1+a_2\|^2\le 2(\|a_1\|^2+\|a_2\|^2)$ for any $a_1,a_2\in\mathbb{R}^d$, (vi) uses eqs. \eqref{bound_A},\eqref{bound_B},\eqref{bound_b}.
	
	Next, we prove eq. \eqref{SGvar_w}. Note that 
	\begin{align}
		&\mathbb{E}\big[\big\|\overline{A}_{t}\overline{\theta}_{t}+\overline{\overline{b}}_{t}+\overline{C}_{t}\overline{w}_{t}\big\|^2\big|\mathcal{F}_t\big]\nonumber\\ &\stackrel{(i)}{=} \mathbb{E}\big[\big\|\overline{A}_{t}\overline{\theta}_{t}+\overline{\overline{b}}_{t}+\overline{C}_{t}(\overline{w}_{t}-w_t^*)-\overline{C}_{t}C^{-1}(A\overline{\theta}_t+b)\big\|^2\big|\mathcal{F}_t\big] \nonumber\\
		&=  \mathbb{E}\big[\big\|(\overline{A}_{t}-\overline{C}_{t}C^{-1}A)(\overline{\theta}_{t}-\theta^*)+(\overline{A}_{t}-\overline{C}_{t}C^{-1}A)\theta^*+\overline{\overline{b}}_{t}+\overline{C}_{t}(\overline{w}_{t}-w_t^*)-\overline{C}_{t}C^{-1}b\big\|^2\big|\mathcal{F}_t\big] \nonumber\\
		&\stackrel{(ii)}{\le} 4\mathbb{E}\big[\big\|(\overline{A}_{t}-\overline{C}_{t}C^{-1}A)(\overline{\theta}_{t}-\theta^*)\big\|^2 + \|(\overline{A}_{t}-\overline{C}_{t}C^{-1}A)\theta^*\|^2 + \|\overline{C}_{t}(\overline{w}_{t}-w_t^*)\|^2 + \|\overline{\overline{b}}_{t}-\overline{C}_{t}C^{-1}b\|^2\big|\mathcal{F}_t\big] \nonumber\\
		&\stackrel{(iii)}{\le} 4\mathbb{E}\big[\|\overline{A}_{t}-\overline{C}_{t}C^{-1}A\|_F^2\big|\mathcal{F}_t\big] \|\overline{\theta}_t-\theta^*\|^2 + 4\mathbb{E}\big[\|\overline{A}_{t}-\overline{C}_{t}C^{-1}A\|_F^2\big|\mathcal{F}_t\big] \|\theta^*\|^2 \nonumber\\
		&\quad + 4\|\overline{w}_{t}-w_t^*\|^2 + 4\mathbb{E}\big[\|\overline{\overline{b}}_{t}-\overline{C}_{t}C^{-1}b\|^2\big|\mathcal{F}_t\big] \nonumber\\
		&\stackrel{(iv)}{\le} \frac{4c_{\text{var,2}}\Omega_{A}^2}{N} \|\overline{\theta}_t-\theta^*\|^2 + \frac{4c_{\text{var,2}}\Omega_{A}^2}{N} \|\theta^*\|^2 + 4\big\|\overline{w}_{t}-w_t^*\big\|^2 + \frac{4c_{\text{var,2}}\Omega_{b}^2}{N} \nonumber\\
		&= \frac{4c_{\text{var,2}}\Omega_{A}^2}{N} \|\overline{\theta}_t-\theta^*\|^2 + 4\big\|\overline{w}_{t}-w_t^*\big\|^2  + \frac{4c_{\text{var,2}}}{N} (\Omega_{A}^2\|\theta^*\|^2 + \Omega_{b}^2) \nonumber
	\end{align}
	where (i) uses the notation that $w_t^*=-C^{-1}(A\overline{\theta}_{t}+b)$, (ii) uses $\|a_1+a_2+a_3+a_4\|^2\le 4(\|a_1\|^2+\|a_2\|^2+\|a_3\|^2+\|a_4\|^2)$ for any $a_1,a_2,a_3,a_4\in\mathbb{R}^d$, (iii) uses eq. \eqref{bound_C}, (iv) uses eqs. \eqref{avg2Avar} \& \eqref{avg2bvar}. \\
\end{proof}

\begin{lemma}\label{lemma_V}
	The doubly stochastic matrix $V$ and the difference matrix $\Delta=I-\frac{1}{M}\mathbf{1}\mathbf{1}^{\top}$ have the following properties:
	\begin{enumerate}
		\item $\Delta V=V\Delta=V-\frac{1}{M}\mathbf{1}\mathbf{1}^{\top}$
		\item $\sigma_1=\|V\|_2=1$ ($\sigma_1$ is the largest singular value of $V$).
		\item For any $x\in\mathbb{R}^M$ and $n\in \mathbb{N}^+$, $\|V^n\Delta x\|\le \sigma_2^n \|\Delta x\|$ ($\sigma_2$ is the second largest singular value of $V$). Hence, for any $H\in\mathbb{R}^{M\times M}$, $\|V^n\Delta H\|_F\le \sigma_2^n \|\Delta H\|_F$
	\end{enumerate}
\end{lemma}

\begin{proof}
	The first item can be proved by the following two equalities.
	\begin{gather}
	\Delta V=\Big(I-\frac{1}{M}\mathbf{1}\mathbf{1}^{\top}\Big)V =V-\frac{1}{M}\mathbf{1}\mathbf{1}^{\top}V=V-\frac{1}{M}\mathbf{1}\mathbf{1}^{\top}\nonumber\\
	V\Delta =V\Big(I-\frac{1}{d}\mathbf{1}\mathbf{1}^{\top}\Big) =V-\frac{1}{M}V\mathbf{1}\mathbf{1}^{\top}=V-\frac{1}{M}\mathbf{1}\mathbf{1}^{\top}\nonumber
	\end{gather}
	
	Denote $V=[v_{ij}]_{M\times M}$. For any $x=[x_1,\ldots,x_M]^{\top}\in\mathbb{R}^M$, 
	\begin{align}
	\|Vx\|^2=&\sum_{i=1}^{d} \Big(\sum_{j=1}^{d} v_{ij}x_j\Big)^2 \le \sum_{i=1}^{d} \sum_{j=1}^{d} v_{ij}x_j^2= \sum_{j=1}^{d} \sum_{i=1}^{d} v_{ij}x_j^2= \sum_{j=1}^{d} x_j^2=\|x\|^2 \nonumber
	\end{align}
	where the first $\le$ uses Jensen's inequality and becomes $=$ if $x=\tau\mathbf{1}$ for some $\tau\in\mathbb{R}$. Hence, the second item $\|V\|_2=1$ holds.
	
	Next, we prove the third item via induction on $n$. 
	
	When $n=1$, we have proved that the inequality $\|V\Delta x\|\le \|\Delta x\|$ holds if $\sigma_2=1$. Otherwise, $\sigma_2<1$, and then consider the singular value decomposition $V=U^{\top}D\tilde{U}$ with unitary matrices $U$, $\tilde{U}$ and diagonal matrix $D=\text{diag}(1,\sigma_2,\sigma_3,\ldots,\sigma_M)$ where $1> \sigma_2\ge \sigma_3\ge \sigma_M \ge 0$. 
	
	Notice that $\mathbf{1}=V\mathbf{1}=U^{\top}D\tilde{U}\mathbf{1}\Rightarrow U\mathbf{1}=D\tilde{U}\mathbf{1}$ and $\mathbf{1}=V^{\top}\mathbf{1}=\tilde{U}^{\top}DU\mathbf{1}\Rightarrow \tilde{U}\mathbf{1}=DU\mathbf{1}$. Hence, $(I-D^2)U\mathbf{1}=\mathbf{0}$. Since $I-D^2$ is a diagonal matrix where the first diagonal entry is zero but the rest diagonal entries are nonzero, all the entries of $U\mathbf{1}$ are zero except the first entry. Hence, the second up to the $M$-th column of the matrix ${U}^{\top}$ form an orthogonal basis of the $(M-1)$-dim space $E\stackrel{\Delta}{=}\{x\in\mathbb{R}^M: \mathbf{1}^{\top}x=0\}$. Since $\Delta x\in E$, it can be expressed as a linear combination of this orthogonall basis, that is, there is $y\in\mathbb{R}^{M-1}$ such that 
	\begin{align}
	\Delta x=\tilde{U}^{\top}\left[ {\begin{array}{*{20}{c}}
		0 \\ 
		y 
		\end{array}} \right] \Rightarrow V\Delta x=U^{\top}D\left[ {\begin{array}{*{20}{c}}
		0 \\ 
		y 
		\end{array}} \right]=U^{\top}\left[ {\begin{array}{*{20}{c}}
		0 \\ 
		\tilde{D}y 
		\end{array}} \right],\nonumber
	\end{align}
	where $\tilde{D}=\text{diag}(\sigma_2,\sigma_3,\ldots,\sigma_M)$. Notice that $\|\tilde{D}\|_2=\sigma_2$, so
	\begin{align}
	\|V\Delta x\|=\|\tilde{D}y\|\le \sigma_2\|y\|=\sigma_2\|\Delta x\|,\nonumber
	\end{align}
	which proves the case $n=1$.
	
	Suppose that $\|V^n\Delta x\|\le \sigma_2^n \|\Delta x\|$ holds for a certain $n\in\mathbb{N}^+$. Then,
	\begin{align}
	\|V^{n+1}\Delta x\|\stackrel{(i)}{=}&\|V^n \Delta Vx\| \le \sigma_2^n\|\Delta Vx\| \stackrel{(ii)}{=} \sigma_2^n\|V \Delta x\| \le \sigma_2^{n+1}\|\Delta x\| \nonumber
	\end{align}
	where (i) and (ii) use the already proved item 1 that $\Delta V=V\Delta$. Hence, for any $x\in\mathbb{R}^M$ and $n\in \mathbb{N}^+$, $\|V^n\Delta x\|\le \sigma_2^n \|\Delta x\|$. 
	
	Furthermore, for any $H\in\mathbb{R}^{M\times M}$, by denoting $h_j$ as the $j$-th column vector of $H$, we obtain that
	\begin{align}
	\|V^n\Delta H\|_F=\sqrt{\sum_{m=1}^{M} \|V^n\Delta h_j\|^2} \le \sqrt{\sum_{m=1}^{M} \big[\sigma_2^n \|\Delta h_j\|\big]^2} = \sigma_2^n \|\Delta H\|_F. \nonumber
	\end{align}
\end{proof}

\begin{coro}\label{coro_rho}
	Under Assumptions \ref{assum_rhobound}\&\ref{assum_rhobound2} \red{and choosing $L\ge \frac{3\ln M}{2\ln(\sigma_2^{-1})}$,} the estimation error of the inexact global importance sampling ratio $\widehat{\rho}_{t}^{(m)}$ has the following upper bound.
	\begin{align}
		\sum_{m=1}^{M}\big(\widehat{\rho}_{t}^{(m)}-\rho_t\big)^2 &\le \red{M^3\sigma_2^{2L}(\rho_{\max}^2/\rho_{\min}) \ln^2 (\rho_{\max}/\rho_{\min})}.\label{rho_var}
	\end{align}
	Therefore, the following inequalities hold. 
	\begin{align}
		\sum_{m=1}^{M}\big\|A_{t}^{(m)}-A_t\big\|_F^2, \sum_{m=1}^{M}\big\|\overline{A}_{t}^{(m)}-\overline{A}_t\big\|_F^2 &\le \red{(1+\gamma)^2 M^3\sigma_2^{2L}(\rho_{\max}^2/\rho_{\min}) \ln^2 (\rho_{\max}/\rho_{\min}) \stackrel{\triangle}{=} M^3 \sigma_2^{2L} D_A} \label{Adev2}\\
		\sum_{m=1}^{M}\big\|B_{t}^{(m)}-B_t\big\|_F^2, \sum_{m=1}^{M}\big\|\overline{B}_{t}^{(m)}-\overline{B}_t\big\|_F^2 &\le \red{\gamma^2 M^3\sigma_2^{2L}(\rho_{\max}^2/\rho_{\min}) \ln^2 (\rho_{\max}/\rho_{\min}) \stackrel{\triangle}{=} M^3 \sigma_2^{2L} D_B},\label{Bdev2}\\ \sum_{m=1}^{M}\big\|\widetilde{b}_{t}^{(m)}-b_{t}^{(m)}\big\|^2, \sum_{m=1}^{M}\big\|\overline{\widetilde{b}}_{t}^{(m)}-\overline{b}_{t}^{(m)}\big\|^2 &\le  \red{R_{\max}^2 M^3\sigma_2^{2L}(\rho_{\max}^2/\rho_{\min}) \ln^2 (\rho_{\max}/\rho_{\min}) \stackrel{\triangle}{=} M^3 \sigma_2^{2L} D_b} \label{bdev2}
	\end{align}
	\red{As a result, the following upper bounds hold.}
	\red{\begin{align}
	\|A_{t}^{(m)}\|_F, \Big\|\overline{A}_{t}^{(m)}\Big\|_F&\le \Omega_A+\sqrt{D_A}\stackrel{\triangle}{=}\widetilde{\Omega}_{A} \label{bound_A2}\\
	\|B_{t}^{(m)}\|_F, \Big\|\overline{B}_{t}^{(m)}\Big\|_F &\le \Omega_B+\sqrt{D_B}\stackrel{\triangle}{=}\widetilde{\Omega}_{B} \label{bound_B2}\\
	\|\widetilde{b}_{t}^{(m)}\|, \|\overline{\widetilde{b}}_{t}^{(m)}\| &\le \Omega_b+\sqrt{D_b}\stackrel{\triangle}{=}\widetilde{\Omega}_{b} \label{bound_b2}
	\end{align}}
\end{coro}
\begin{proof}
	Eq. \eqref{rho_iter2} can be rewritten into the following matrix form. 
	\begin{align}
		\big[\widetilde{\rho}_{t,L}^{(1)};\ldots;\widetilde{\rho}_{t,L}^{(M)}\big]^{\top} = V^{L}\big[\widetilde{\rho}_{t,0}^{(1)};\ldots;\widetilde{\rho}_{t,0}^{(M)}\big]^{\top}. \nonumber
	\end{align}
	Hence, the item 1 of Lemma \ref{lemma_V} yields that
	\begin{align}
		\Delta\big[\widetilde{\rho}_{t,L}^{(1)};\ldots;\widetilde{\rho}_{t,L}^{(M)}\big]^{\top} = V^{L}\Delta\big[\widetilde{\rho}_{t,0}^{(1)};\ldots;\widetilde{\rho}_{t,0}^{(M)}\big]^{\top}. \nonumber
	\end{align}
	Then the item 3 of Lemma \ref{lemma_V} yields that
	\begin{align}
		\big\|\Delta\big[\widetilde{\rho}_{t,L}^{(1)};\ldots;\widetilde{\rho}_{t,L}^{(M)}\big]^{\top}\big\|^2 \le \sigma_2^{2L}\big\|\Delta\big[\widetilde{\rho}_{t,0}^{(1)};\ldots;\widetilde{\rho}_{t,0}^{(M)}\big]^{\top}\big\|^2. \label{rho_consensus_decay}
	\end{align}
	Assumptions \ref{assum_rhobound}\&\ref{assum_rhobound2} imply that $\widetilde{\rho}_{t,0}^{(m)}=\ln\rho_t^{(m)}\in[\ln\rho_{\min},\ln\rho_{\max}]$. Then, since
	eq. \eqref{rho_iter2} implies that $\min_{m'\in\mathcal{M}} \widetilde{\rho}_{t,\ell}^{(m')}\le \widetilde{\rho}_{t,\ell+1}^{(m)}\le \max_{m'\in\mathcal{M}} \widetilde{\rho}_{t,\ell}^{(m')}$, it can be easily proved by induction that $\widetilde{\rho}_{t,L}^{(m)}\in[\ln\rho_{\min},\ln\rho_{\max}]$. Hence,
	\begin{align}
		\frac{1}{M}\ln \rho_{t} =\frac{1}{M}\sum_{m=1}^M \ln\rho_{t}^{(m)} =\frac{1}{M}\sum_{m=1}^M \widetilde{\rho}_{t,0}^{(m)} =\frac{1}{M}\sum_{m=1}^M \widetilde{\rho}_{t,L}^{(m)} \in[\ln\rho_{\min},\ln\rho_{\max}] \label{lnrho_bound}
	\end{align} 
	Then eqs. \eqref{rho_consensus_decay}\&\eqref{lnrho_bound} imply that
	\begin{align}
		\sum_{m=1}^{M}\Big(\widetilde{\rho}_{t,L}^{(m)}-\frac{1}{M}\ln\rho_t\Big)^2 \le \sigma_2^{2L}\sum_{m=1}^{M}\Big(\widetilde{\rho}_{t,0}^{(m)}-\frac{1}{M}\ln\rho_t\Big)^2\le M\sigma_2^{2L}\ln^2(\rho_{\max}/\rho_{\min}).\label{rhoprime_var}
	\end{align}

	Hence, 
	\begin{align}
		&\Big|\widetilde{\rho}_{t,L}^{(m)}-\frac{1}{M}\ln\rho_t\Big|\le \sqrt{M} \sigma_2^{L}\ln(\rho_{\max}/\rho_{\min})\red{\stackrel{(i)}{\le} \frac{1}{M}\ln(\rho_{\max}/\rho_{\min})},
	\end{align}
	\red{where (i) uses the conditions that $L \ge \frac{3\ln M}{2\ln(\sigma_2^{-1})}$ and $\sigma_2\in[0,1)$. Hence, the triangular inequality yields that}
	\begin{align}
		\widetilde{\rho}_{t,L}^{(m)}\le \frac{1}{M}\ln\rho_t + \frac{1}{M}\ln(\rho_{\max}/\rho_{\min}) \le \frac{1}{M}\ln(\rho_{\max}^2/\rho_{\min}). \label{rho_tilde_bound}
	\end{align}
	
	Therefore, eq. \eqref{rho_var} can be proved as follows.
	\begin{align}
		\sum_{m=1}^{M}\big(\widehat{\rho}_{t}^{(m)}-\rho_t\big)^2 &\stackrel{(i)}{=}\sum_{m=1}^{M}\big(
		e^{M\widetilde{\rho}_{t,L}^{(m)}}-e^{\ln \rho_t} )^2\nonumber\\
		&\stackrel{(ii)}{\le} \sum_{m=1}^{M}\big[\max \big(e^{M\widetilde{\rho}_{t,L}^{(m)}}, e^{\ln \rho_t}\big)\big]^2 \big(M\widetilde{\rho}_{t,L}^{(m)}-\ln\rho_t\big)^2\nonumber\\
		&\stackrel{(iii)}{\le} M^2\sum_{m=1}^{M} \big[\max\big(\red{\rho_{\max}^2/\rho_{\min}}, \rho_{\max}\big)\big]^2\Big(\widetilde{\rho}_{t,L}^{(m)}-\frac{1}{M}\ln\rho_t\Big)^2\nonumber\\
		&\stackrel{(iv)}{\le} \red{M^3\sigma_2^{2L}(\rho_{\max}^2/\rho_{\min}) \ln^2 (\rho_{\max}/\rho_{\min})},\nonumber
	\end{align}
	where (i) uses eq. \eqref{rho_end2}, (ii) uses the Lagrange's Mean Value Theorem, (iii) uses eq. \eqref{rho_tilde_bound} and the inequality that $\rho_t\le\rho_{\max}$, and (iv) uses eq. \eqref{rhoprime_var}. 
	
	Then, eq. \eqref{Adev2} can be proved as follows.
	\begin{align}
		\sum_{m=1}^{M}\big\|A_{t}^{(m)}-A_t\big\|_F^2 &\le\big\| \phi(s_t) [\gamma \phi(s_{t+1})-\phi(s_t)]^{\top}\big\|_F^2 \sum_{m=1}^{M}(\widehat{\rho}_{t}^{(m)}-\rho_t)^2 \nonumber\\
		&\le (1+\gamma)^2 \red{M^3\sigma_2^{2L}(\rho_{\max}^2/\rho_{\min}) \ln^2 (\rho_{\max}/\rho_{\min})} \stackrel{\triangle}{=} M^3 \sigma_2^{2L} D_A.\label{A_dev}
	\end{align}
	Hence, 	
	\begin{align}
		\sum_{m=1}^{M}\big\|\overline{A}_{t}^{(m)}-\overline{A}_t\big\|_F^2 &= \sum_{m=1}^{M}\Big\|\frac{1}{N}\sum_{i=tN}^{(t+1)N-1}(A_{i}^{(m)}-A_i)\Big\|_F^2 \nonumber\\
		&\stackrel{(i)}{\le} \frac{1}{N} \sum_{m=1}^{M} \sum_{i=tN}^{(t+1)N-1} \Big\|A_{i}^{(m)}-A_i\Big\|_F^2 \nonumber\\
		&\stackrel{(ii)}{\le} M^3 \sigma_2^{2L} D_A \nonumber
	\end{align}
	where (i) applies Jensen's inequality to the convex function $\|\cdot\|^2$, and (ii) uses eq. \eqref{A_dev}. The eqs. \eqref{Bdev2} \& \eqref{bdev2} can be proved similarly. 
	
	\red{The eq. \eqref{A_dev} and the condition that $L \ge \frac{3\ln M}{2\ln(\sigma_2^{-1})}$ imply that $\big\|A_{t}^{(m)}-A_t\big\|_F,  \big\|\overline{A}_{t}^{(m)}-\overline{A}_t\big\|_F \le \sqrt{D_A}$. Hence, eq. \eqref{bound_A2} can be proved using triangle inequality and eq. \eqref{bound_A}. The eqs. \eqref{bound_B2} \& \eqref{bound_b2} can be proved similarly. }
\end{proof}

\begin{lemma}\label{lemma_parabound2}
	Under the update rules in \eqref{theta_iter2b}\&\eqref{w_iter2b} of Algorithm \ref{alg: 2} and \red{choosing $L\ge \frac{3\ln M}{2\ln(\sigma_2^{-1})}$, $\alpha\le \beta\min\Big(\frac{\widetilde{\Omega}_{A}}{\widetilde{\Omega}_{B}},\frac{1}{\widetilde{\Omega}_{A}},1\Big)$}, the parameters have the following upper bound.
	\red{
	\begin{align}\label{para_bound2}
		\max_{m\in\mathcal{M}}\|\theta_{T}^{(m)}\| + \max_{m\in\mathcal{M}}\|w_{T}^{(m)}\| &\le \big[1+\beta(\widetilde{\Omega}_{A}+1)\big]^{T} \Big(\max_{m\in\mathcal{M}}\|\theta_{0}^{(m)}\| + \max_{m\in\mathcal{M}}\|w_{0}^{(m)}\| +\frac{2\widetilde{\Omega}_{b}}{\widetilde{\Omega}_{A}+1} \Big) \nonumber\\
		&\stackrel{\triangle}{=} c_{\text{para}}\big[1+\beta(\widetilde{\Omega}_{A}+1)\big]^{T}.
	\end{align}}
\end{lemma}
\begin{proof}
	\red{Since $L\ge \frac{3\ln M}{2\ln(\sigma_2^{-1})}$, eqs. \eqref{bound_A2}, \eqref{bound_B2} \& \eqref{bound_b2} hold. Hence, these equations and the update rule \eqref{theta_iter2b} imply that}
	\begin{align}
		\|\theta_{t+1}^{(m)}\|&\le \sum_{m'\in \mathcal{N}_m}V_{m,m'} \|\theta_{t}^{(m')}\| + \alpha\big[\red{\widetilde{\Omega}_{A}} \|\theta_{t}^{(m)}\|+\red{\widetilde{\Omega}_{b}}+\red{\widetilde{\Omega}_{B}} \|w_{t}^{(m)}\|\big].
	\end{align}
	Taking maximum with respect to $m$ yields that
	\begin{align}
		\max_{m\in\mathcal{M}}\|\theta_{t+1}^{(m)}\|&\le \max_{m\in\mathcal{M}} \sum_{m'\in \mathcal{N}_m}V_{m,m'} \max_{m''\in\mathcal{M}} \|\theta_{t}^{(m'')}\| \nonumber\\
		&\quad + \alpha\big(\widetilde{\Omega}_{A} \max_{m\in\mathcal{M}} \|\theta_{t}^{(m)}\|+\widetilde{\Omega}_{b}+\widetilde{\Omega}_{B} \max_{m\in\mathcal{M}}\|w_{t}^{(m)}\|\big).\nonumber\\
		&\le (1+\alpha \widetilde{\Omega}_{A}) \max_{m\in\mathcal{M}}\|\theta_{t}^{(m)}\| + \alpha \widetilde{\Omega}_{B} \max_{m\in\mathcal{M}}\|w_{t}^{(m)}\| + \alpha \widetilde{\Omega}_{b}. \label{theta_norm_iter}
	\end{align}

	
	Similarly, it can be obtained from the update rule \eqref{w_iter2b} that 
	\begin{align}
	\max_{m\in\mathcal{M}}\|w_{t+1}^{(m)}\|&\le \beta \widetilde{\Omega}_{A} \max_{m\in\mathcal{M}}\|\theta_{t}^{(m)}\| + (1+\beta) \max_{m\in\mathcal{M}}\|w_{t}^{(m)}\| + \beta \widetilde{\Omega}_{b}. \label{w_norm_iter}
	\end{align}	
	
	Adding up eqs. \eqref{theta_norm_iter}\&\eqref{w_norm_iter} yields that
	\begin{align}
		&\max_{m\in\mathcal{M}}\|\theta_{t+1}^{(m)}\| + \max_{m\in\mathcal{M}}\|w_{t+1}^{(m)}\| \nonumber\\
		&\le [1+(\alpha+\beta)\widetilde{\Omega}_{A}] \max_{m\in\mathcal{M}}\|\theta_{t}^{(m)}\| + (1+\alpha \widetilde{\Omega}_{B}+\beta) \max_{m\in\mathcal{M}}\|w_{t}^{(m)}\| +(\alpha+\beta) \widetilde{\Omega}_{b}\nonumber\\
		&\stackrel{(i)}{\le} \big[1+\beta(\widetilde{\Omega}_{A}+1)\big]\Big(\max_{m\in\mathcal{M}}\|\theta_{t}^{(m)}\| + \max_{m\in\mathcal{M}}\|w_{t}^{(m)}\|\Big) + 2\beta \widetilde{\Omega}_{b},
	\end{align}
	where (i) uses the condition that \red{$\alpha\le \beta\min\Big(\frac{\widetilde{\Omega}_{A}}{\widetilde{\Omega}_{B}},\frac{1}{\widetilde{\Omega}_{A}},1\Big)$}. By iterating the inequality above, \red{we prove eq. \eqref{para_bound2}.}
\end{proof}

\end{document}